\documentclass[11pt]{article}

\usepackage[margin=1in]{geometry}
\usepackage{amsmath,amssymb,amsfonts,mathtools,bm}
\usepackage{booktabs}
\usepackage{longtable}
\usepackage{graphicx}
\usepackage{enumitem}
\usepackage{xcolor}
\usepackage{microtype}
\usepackage{array}
\usepackage{caption}
\usepackage{placeins}
\usepackage{float}
\usepackage{mathrsfs}
\usepackage{tikz,pgfplots}
\newcommand{\BEAS}{\begin{eqnarray*}}
\newcommand{\EEAS}{\end{eqnarray*}}
\newcommand{\BEA}{\begin{eqnarray}}
\newcommand{\EEA}{\end{eqnarray}}
\newcommand{\BEQ}{\begin{equation}}
\newcommand{\EEQ}{\end{equation}}
\newcommand{\BIT}{\begin{itemize}}
\newcommand{\EIT}{\end{itemize}}
\newcommand{\BNUM}{\begin{enumerate}}
\newcommand{\ENUM}{\end{enumerate}}
\newcommand{\BA}{\begin{array}}
\newcommand{\EA}{\end{array}}
\newcommand{\Diag}{\mathop{\rm Diag}}

\newcommand{\idm}{I}
\newcommand{\rb}{\mathbb{R}}
 
\newcommand{\mysec}[1]{Section~\ref{sec:#1}}
\newcommand{\eq}[1]{Eq.~(\ref{eq:#1})}
\newcommand{\myfig}[1]{Figure~\ref{fig:#1}}

\newcommand{\BlackBox}{\rule{1.5ex}{1.5ex}}  % end of proof

\newenvironment{proof}
  {\par\noindent{\bf Proof\ }}
  {\hfill\BlackBox\par\vspace{2mm}}

\newtheorem{theorem}{Theorem}
\newtheorem{proposition}{Proposition}

\newtheorem{assumption}{Assumption}

\newtheorem{remark}{Remark}

\newcommand{\R}{\mathbb R}
\newcommand{\tr}{\operatorname{tr}}
\newcommand{\diag}{\operatorname{diag}}

\newcommand{\conv}{\operatorname{conv}}
\newcommand{\Normal}{\mathcal N}
\newcommand{\dd}{  d }
\newcommand{\op}{\mathrm{op}}
\newcommand{\artanh}{\operatorname{artanh}}

\newcommand{\J}{\mathcal J}
\newcommand{\K}{\mathcal K}

\newcommand{\Zq}{Z_q}

\newcommand{\Nquad}{N_{\rm quad}}

\newcommand{\E}{\mathbb E}
\renewcommand{\P}{\mathbb P}

\newcommand{\df}{\operatorname{df}}

\providecommand{\norm}[1]{\left\lVert #1\right\rVert}
\providecommand{\ip}[2]{\left\langle #1,#2\right\rangle}
\providecommand{\hDelta}{\hat\Delta}
\providecommand{\hmu}{\hat\mu}
\providecommand{\hC}{\hat C}
\providecommand{\hM}{\hat M}
\providecommand{\hCp}{\hat C_p}
\providecommand{\hCq}{\hat C_q}
\providecommand{\hmup}{\hat\mu_p}
\providecommand{\hmuq}{\hat\mu_q}
\providecommand{\RD}{R}
\providecommand{\Sone}{\mathfrak S_1}
\providecommand{\Tr}{\operatorname{Tr}}

\usepackage[mathcal]{eucal}

\DeclareMathAlphabet\mathbfcal{OMS}{cmsy}{b}{n}

\usepackage[colorlinks=true, allcolors=blue, pagebackref]{hyperref}
\renewcommand*{\backrefalt}[4]{%
    \ifcase #1 \footnotesize{(not cited)}%
    \or        \footnotesize{(cited on page~#2)}%
    \else      \footnotesize{(cited on pages~#2)}%
    \fi}

\title{Regularized Variational and Spectral Log-Density-Ratio Estimation\\ in the Gaussian Location Model}

\author{Francis Bach \\
Inria - Ecole Normale Sup\'erieure\\ 
PSL Research University
\\
{ \url{francis.bach@inria.fr}}}

\begin{document}

\maketitle

%\linenumbers

\begin{abstract}
We study ridge-regularized log-density-ratio estimation in the Gaussian location model with a common covariance matrix. By affine invariance, the model is written as \(q\sim\Normal(0,\idm)\), \(p\sim\Normal(\Delta,\idm)\), with linear features, where $\Delta \in \mathbb{R}^m$ and the signal strength \(s=\|\Delta\|^2\) is fixed.  The variational estimator is the empirical Kullback-Leibler (KL) log-normalized fit with a squared \(\ell_2\)-penalty on its nonconstant coefficient, and the spectral estimator recently introduced in \cite{Bach2026} replaces a single variational problem by a continuum of ridge-regularized least-squares problems.   

We derive high-dimensional deterministic asymptotic equivalents when the numbers of observations and dimension tend to infinity with fixed ratios.  The regularized variational limit is characterized by a 
scalar entropy minimization problem derived from the convex-Gaussian-min-max theorem (CGMT), while the regularized spectral limit follows from deterministic equivalents for resolvents of weighted sums of two independent Gaussian sample covariance matrices.  We use these formulas to compare population risks, with experiments focused on fixed-signal aspect-ratio sweeps and optimized regularization.   Our conclusion is that with many observations, under the criteria and asymptotic regimes analyzed here, the well-specified variational estimator has the smaller risk, while with fewer observations, the spectral estimator is favored because its covariance-based construction has lower variance. We also study how a nuclear penalty can be used and partially analyzed to perform feature learning.
\end{abstract}

\section{Introduction}

\addcontentsline{toc}{section}{Introduction and objective}

Estimating a log-density ratio from two samples is a central primitive in machine learning and data science, with applications to covariate-shift correction, two-sample testing, importance weighting, likelihood-free inference, and variational estimation of divergences \cite{SugiyamaNakanoKawanabe2008,KanamoriHidoSugiyama2009,SugiyamaSuzukiKanamori2012,NguyenWainwrightJordan2010}.  A common approach is to use a variational representation of the Kullback--Leibler (KL) divergence, fit a real-valued potential \(v\) from samples of \(p\) and \(q\), and interpret \(v\) as an estimate of \(\log(dp/dq)\) \cite{DonskerVaradhan1975,NguyenWainwrightJordan2010,Qin1998}.  Without regularization this approach is exposed to the variance of empirical exponential normalization: a few sample points from $q$ can dominate the ``log-sum-exp'' term, and in high dimension, the empirical objective can even be unbounded, like in logistic regression~\cite{CandesSur2020,ChardonLerasleMourtada2026}. Ridge regularization is the standard remedy, going back to the classical ridge regression literature \cite{HoerlKennard1970}; it makes the optimization problem finite and stable, but it also introduces shrinkage bias.  The central question of this paper is to understand how the performance of regularized estimators depends on signal strength, aspect ratios (ratios between dimension and number of observations), and the ridge parameters.

Our main goal is to understand the performance of the recent spectral framework for relative density estimation \cite{Bach2026}, which represents the KL divergence as a mixture of weighted chi-square divergences.  This turns the estimation of a log-density potential into a family of least-squares problems indexed by a mixture parameter \(\rho\in[0,1]\).  The integral with respect to $\rho$ can be carried out in closed form using a generalized eigenvalue decomposition of second-order moments; this spectral construction has thus closed-form computation through empirical first and second moments, with a ridge penalty needed to stabilize estimation. The reformulation as a continuum of least-squares problems immediately brings to bear the large literature on algorithms and analyses for least-squares regression to control estimation variance, but this closed-form estimator comes at a cost of potential additional bias.   

We compare the ridge-regularized variational and spectral estimators in a simple Gaussian location model.  By affine invariance and whitening, it is enough to consider, in dimension \(m\),
\[
q\sim\Normal(0,\idm),\qquad p\sim\Normal(\Delta,\idm),\]
with linear models (with unregularized intercept).\footnote{Note that the affine invariance allows this simple reparameterization for unregularized estimation, while for regularized estimation, this is a modeling choice.}  We consider \(n_p\) i.i.d.~samples from \(p\), and \(n_q\) from \(q\), in the high-dimensional proportional regime, that is, with the following dimension-to-sample ratios, which will be referred to as aspect ratios from now on,
\[
\frac{m}{n_p}\to\alpha_p\in[0,\infty),
\qquad
\frac{m}{n_q}\to\alpha_q\in[0,\infty).\]
The signal strength \( s=\|\Delta\|^2\) is fixed in the main asymptotic statements and in our experiments.  This model is deliberately favorable to the variational estimator because the true log-density ratio is affine.  The spectral estimator is, however, not exactly well-specified, but it can reduce estimation variance by replacing empirical exponential normalization with covariance-based least squares. The goal of this paper is precisely to determine when, in this simple model, this variance reduction translates into smaller population risk.

\paragraph{Contributions.} 
The first contribution of this paper is a set of finite-sample identities for the regularized and unregularized estimators, presented in \mysec{finite}, complemented by a population analysis in \mysec{pop-bench} to understand the shrinkage bias due to regularization.  The second and main contribution is a high-dimensional fixed-signal analysis presented in \mysec{hd}:  The variational limit is a scalar entropy minimization problem obtained from CGMT. The spectral limit is a deterministic equivalent for ridge-regularized weighted sums of two Wishart matrices.  The third contribution is a risk comparison under three different natural criteria for unregularized and regularized estimation, which we perform 
in \mysec{weak} for small $\alpha_p$ and $\alpha_q$ with further asymptotic expansions, and in
\mysec{comparison} by plotting deterministic asymptotic limits in the $(\alpha_p,\alpha_q)$-plane, highlighting that the variational estimator outperforms the spectral estimator when $\alpha_p$ and $\alpha_q$ are small (many observations), but underperforms when these two parameters grow. In \mysec{experiments}, we illustrate the asymptotic limit by comparing empirical estimates with their asymptotic equivalents. Finally, in Appendix~\ref{app:nuclear}, we also study how a nuclear penalty can be used and partially analyzed to perform feature learning (for simplicity, we only consider discrete measures on $\rho$ instead of continuous measures needed for the KL divergence).

\subsection{Related work}
\addcontentsline{toc}{section}{Related work}

\paragraph{Direct density-ratio estimation.} In this paper, we compare only two estimators for log-density ratios, a natural one based on a classical variational representation of the KL divergence~\cite{NguyenWainwrightJordan2010}, and the recently introduced spectral estimators~\cite{Bach2026}. Other estimators could be considered as well, see, e.g., 
\cite{SugiyamaNakanoKawanabe2008,KanamoriHidoSugiyama2009,SugiyamaSuzukiKanamori2012,HuangSmolaGrettonBorgwardtScholkopf2007,MenonOng2016}.  

\paragraph{Asymptotic analysis of maximum-likelihood estimators.} Our analysis of the variational estimator shares many features with the analysis of maximum-likelihood estimators, such as logistic regression, with similar potential unboundedness when not regularized~\cite{SurCandes2019,CandesSur2020,ChardonLerasleMourtada2026}; we use similar tools such as the convex-Gaussian-min-max theorem (CGMT)~\cite{salehi2019impact}.

\paragraph{Random matrix theory.}  Our analysis for the spectral estimator essentially corresponds to doing a
separate analysis for each $\rho \in [0,1]$ and a single regularized least-squares problem, for which there
exists a significant literature, which we reuse here \cite{BaiSilverstein2010,HachemLoubatonNajim2007}.
Closest in goal are analyses of ridge regression and regularized
discriminant analysis, where random-matrix methods yield explicit limiting prediction-risk formulas
\cite{dobriban2018prediction}. Our setting differs in that the spectral estimator requires a continuum of
$\rho$-weighted covariance resolvents from two independent samples, with cross-resolvent quantities entering
the final risk formulas. We use degree-of-freedom quantities in the spirit of
\cite{montanari2022interpolation,Bach2024RandomProjections}.

\section{Problem setup}
\label{sec:finite}

This section describes in detail the Gaussian model, the two estimators, and the criteria that we consider in this paper, without imposing proportional-asymptotic assumptions on \(m,n_p,n_q\). The KL variational identities are the Donsker-Varadhan and Fenchel representations of KL \cite{DonskerVaradhan1975,NguyenWainwrightJordan2010}, and the spectral construction follows the spectral relative-density framework of \cite{Bach2026}.

\subsection{Gaussian model and samples}

We consider the Gaussian model
\[
q\sim \Normal(0,\idm),\qquad p\sim\Normal(\Delta,\idm),
\]
with $\Delta\in\R^m$ and $s = \| \Delta \|^2$.
The true normalized log-density ratio is
\[
 v_\ast(x)= \log \frac{dp}{dq}(x) = \frac{1}{2} \| x\|^2 - \frac{1}{2} \| x - \Delta\|^2 = \Delta^\top x-\frac{s}{2},
 \qquad  \mbox{ with } \qquad D(p\|q)=\frac{s}{2}.
\]
We observe independent samples
\[
x_i\sim p,
\quad i=1,\ldots,n_p,
\qquad
y_j\sim q,
\quad j=1,\ldots,n_q.
\]
We can then write the empirical means and their difference as
\[
\hat\mu_p=\frac1{n_p}\sum_{i=1}^{n_p}x_i,
 \qquad
 \hat\mu_q=\frac1{n_q}\sum_{j=1}^{n_q}y_j,
 \qquad
\hat{\Delta}=\hat\mu_p-\hat\mu_q.
\]
The centered sample covariances, normalized by \(1/n_p\) and \(1/n_q\), are
\begin{equation*}
 \hat C_p=\frac1{n_p}\sum_{i=1}^{n_p}(x_i-\hat\mu_p)(x_i-\hat\mu_p)^\top,
 \qquad
 \hat C_q=\frac1{n_q}\sum_{j=1}^{n_q}(y_j-\hat\mu_q)(y_j-\hat\mu_q)^\top.
\end{equation*}

\subsection{Variational estimator and its ridge-regularized version}

The variational estimator is the KL special case of the convex risk-minimization framework of \cite{NguyenWainwrightJordan2010}, which states, for general probability distributions,\footnote{In this paper, unless otherwise specified when ambiguous, $\E_p[ H(x)]$ is the expectation of $H(x)$ for $x$ distributed as $p$.}
\[
D(p\|q) = \sup_{v: \mathbb{R}^m \to \mathbb{R}} \mathbb E_p[v(x)]-\mathbb E_q[e^{v(y)}-1].
\]
This leads, for a measurable function \(v\), to the population \emph{Fenchel score} for KL:
\[
 \K(v)=\mathbb E_p[v(x)]-\mathbb E_q[e^{v(y)}-1].
\]
  If we consider a potential with an additional intercept, as \( \tilde{v} (x)=v(x)+c\), then \(
 \K(\tilde{v} )= \mathbb E_p[v(x)]+c+1-e^c\mathbb E_q [e^{v(y)}].
\)
For fixed \(v\), the intercept is optimized in closed form as
\(
 c_\ast =-\log\mathbb E_q[ e^{v(y)}].
\)
Substitution gives the \emph{log-normalized} population objective
\[
 \J(v)= \mathbb E_p[v(x)]-\log\big( \mathbb  E_q [e^{v(y)}] \big).
\]
The empirical version replaces \(\mathbb E_p\) and \(\mathbb E_q\) by sample averages, and we now consider a linear model $v(x) = \theta^\top x$.  Thus, over these affine potentials, the empirical intercept for a fixed \(\theta\) is
\[
 \hat c(\theta)=-\log\bigg(\frac1{n_q}\sum_{j=1}^{n_q}\exp(\theta^\top y_j)\bigg),
\]
and the optimized empirical criterion is the log-normalized objective in \(\theta\), on which we add a ridge penalty (only on the linear term, because the intercept has already been optimized by normalization), leading to, for $\tau >0$,
\begin{equation}
 \widehat{\J}(\theta)
 =\theta^\top\hat\mu_p
 -\log\bigg(\frac1{n_q}\sum_{j=1}^{n_q}\exp(\theta^\top y_j)\bigg)
 -\frac{\tau}{2}\|\theta\|^2 .
 \label{eq:Jhat-ridge}
\end{equation}
The ridge estimator is
\(
 \hat\theta\in\arg\max_{\theta\in\R^m}\widehat{\J}(\theta).
\) For every sample and every \(\tau>0\), the objective is strongly concave, hence the maximizer exists and is unique.  

For criterion \(\K\), which is sensitive to additive constants, we use the empirically normalized representative
\begin{equation*}
 \hat v_{{\rm var}}(x)=\hat\theta^\top x+\hat c,
 \qquad \mbox{ with } 
 \hat c =-\log\bigg(\frac1{n_q}\sum_{j=1}^{n_q}\exp(\hat\theta^\top y_j)\bigg).
\end{equation*}

The unregularized estimator is obtained by dropping the last term in \eqref{eq:Jhat-ridge}.  Unlike the ridge estimator, the unregularized objective can be unbounded; Proposition~\ref{prop:var-dual} in \mysec{zeroridge} gives the exact convex-hull condition for finite value and finite attainment.

\paragraph{Estimation algorithm.}
The objective function in \eq{Jhat-ridge} is smooth and strongly concave if $\tau>0$, so damped Newton~\cite{mccullagh2019generalized}, trust-region Newton, or L-BFGS with backtracking line search (as used in the experiments) are globally well-behaved choices \cite{BoydVandenberghe2004,NocedalWright2006}.  

\subsection{Spectral estimator}
\label{sec:est2}

The spectral estimator of \cite{Bach2026} starts from an equivalent two-potential variational formulation~\cite{broniatowski2006minimization}
\BEQ
\label{eq:DKL2}
D(p\|q) = \sup_{v,w: \mathbb{R}^m \to \mathbb{R}} 
\E_p [ v(x)]+\mathbb E_q [w(y)], \mbox{ such that } 
 \forall y \in \rb^m, \ w(y)\leqslant 1-e^{v(y)}.
\EEQ
 The optimum of the associated variational problem is attained by
\(
 v_\ast=\log(dp/dq)\) and \(w_\ast=1-dp/dq.
\)
We will estimate these two potentials through a continuum of least-squares problems.

The key contribution of~\cite{Bach2026} is to represent the KL divergence as an integral of weighted chi-square divergences as
\BEA
 \notag &\!\! \!\! & D(p\|q)  \\
 \notag &\!\!\!\!\! = \!\!\!\!\!& \int_{\rb^m} \int_0^1  \frac{ (\frac{dp}{dq}(x) - 1)^2}{\rho \frac{dp}{dq}(x) + 1-\rho} (1-\rho)d \rho dq(x)
\\
\label{eq:AAAA}& \!\!\!\!\!= \!\!\!\!\! &  \int_0^1 \bigg[ \sup_{u(\rho,\cdot):\rb^m\to \rb}\  \int_{\rb^m} \Big[
   u(\rho,x) \Big( \frac{dp}{dq}(x) - 1 \Big) - \frac{u(\rho,x)^2}{2} \Big( \rho \frac{dp}{dq}(x) + 1-\rho\Big)
   \Big] dq(x) \bigg]  2 (1-\rho) d\rho \\ 
   \notag&\!\!\!\!\! =  \!\!\!\!\!&  \int_0^1 \bigg[  \sup_{u(\rho,\cdot):\rb^m \to \rb}\  \int_{\rb^m} \Big[ u(\rho,x) - \frac{\rho}{2} u(\rho,x)^2 \Big] dp(x)
   + \int_{\rb^m} \Big[ - u(\rho,x) - \frac{1-\rho}{2} u(\rho,x)^2 \Big] dq(x) \bigg]  2 (1-\rho) d\rho,
\EEA
with the integrand having its own variational formulation through a function \(u(\rho,\cdot): \rb^m \to \rb \) for  \(\rho\in[0,1]\). The population optimal variational function for each $\rho \in [0,1]$ is
\[
 u_\ast(\rho,x)=\frac{dp/dq(x)-1}{(1-\rho)+\rho dp/dq(x)}.
\]
Given any fitted functions \(u(\rho,\cdot)\), the corresponding spectral candidate potentials are then
\begin{equation}
 v(x)=\int_0^1 2(1-\rho)\left[u(\rho,x)-\frac{\rho}{2}u(\rho,x)^2\right] d \rho,
 \label{eq:spectral-candidate-v}
\end{equation}
\[
 w(y)=\int_0^1 2(1-\rho)\Big[-u(\rho,y)-\frac{1-\rho}{2}u(\rho,y)^2\Big] d \rho.
\]
At the population optimum these formulas recover \((v_\ast,w_\ast)\).  Since the variational formulation in \eq{AAAA} only involves the function $u$ and its square, the empirical spectral estimator fits \(u(\rho,\cdot)\) by least squares with an affine function with an unregularized intercept.  This yields a continuum of covariance-based linear systems indexed by \(\rho\).

For a fixed \(\rho\), define
\BEQ
\label{eq:murho}
 \hat\mu(\rho)=\rho\hat\mu_p+(1-\rho)\hat\mu_q\in\R^m,
 \qquad
 \hat C(\rho)=\rho\hat C_p+(1-\rho)\hat C_q\in\R^{m\times m},
 \qquad
 \hat M(\rho)=\hat C(\rho)+\rho(1-\rho)\hat{\Delta}\hat{\Delta}^\top\in\R^{m\times m}.
\EEQ
The ridge-regularized spectral coefficient is, for \(\zeta>0\), using \cite[Appendix F]{Bach2026},
\begin{equation}
 \hat\beta(\rho)
 = (\hat M(\rho)+\zeta \idm)^{-1}\hat{\Delta}
 =\frac{(\hat C(\rho)+\zeta \idm)^{-1}\hat{\Delta}}
 {1+\rho(1-\rho)\hat{\Delta}^\top(\hat C(\rho)+\zeta \idm)^{-1}\hat{\Delta}},
 \label{eq:beta-ridge-spectral}
\end{equation}
where the second equality is obtained from the Sherman-Morrison formula.  Then, with 
\[
 \hat u(\rho,x)=\hat\beta(\rho)^\top(x-\hat\mu(\rho)),
\]
the ridge spectral potential is obtained from Eq.~(\ref{eq:spectral-candidate-v}) as
\begin{equation}
 \hat v_{{\rm spec}}(x)
 =\int_0^1 2(1-\rho)
 \left[\hat u(\rho,x)-\frac{\rho}{2}\hat u(\rho,x)^2\right] d \rho.
 \label{eq:vhat-spec-ridge}
\end{equation}
Since \(\hat M(\rho)+\zeta \idm\succcurlyeq \zeta \idm\), the coefficient and the integral are finite for every finite sample and every \(\rho\in[0,1]\).

Expanding \eqref{eq:vhat-spec-ridge} gives a quadratic potential
\begin{equation}
 \hat v_{{\rm spec}}(x)=x^\top\hat A x+\hat\ell^\top x+\hat c,
 \label{eq:vhat-spec-ridge-quad}
\end{equation}
where
\begin{align}
 \hat A&=-\int_0^1 \rho(1-\rho)\hat\beta(\rho)\hat\beta(\rho)^\top d \rho,
 \label{eq:Ahat-ridge}\\
 \hat\ell&=\int_0^1 2(1-\rho)\big(1+\rho \hat\beta(\rho)^\top\hat\mu(\rho)\big)\hat\beta(\rho) d \rho,
 \\
 \hat c&=\int_0^1 2(1-\rho)
 \left(-\hat \beta(\rho)^\top\hat\mu(\rho)-\frac{\rho}{2}(\hat\beta(\rho)^\top\hat\mu(\rho))^2\right) d \rho.
 \label{eq:chat-ridge}
\end{align}
These coefficients are the quantities inserted into the quadratic scoring identities of Proposition~\ref{prop:quadratic-scores}.

\paragraph{Unregularized limit.}
The unregularized full-potential estimator is the zero-ridge limit obtained by setting \(\zeta=0\), leading to 
\begin{equation}
 \hat\beta_\rho=\hat M(\rho)^{-1}\hat{\Delta},
 \qquad
 \hat v_{\rm spec}(x)=\int_0^1 2(1-\rho)
 \left(\hat u(\rho,x)-\frac{\rho}{2}\hat u(\rho,x)^2\right) d \rho.
 \label{eq:vhat-cont}
\end{equation}
This limit is not automatically finite.  The behavior of \(\hat M(\rho)^{-1}\hat{\Delta}\) at \(\rho=0\) and \(\rho=1\) controls whether the continuum integral defines a finite unregularized function: if \(\hat C_q\) is singular, an inverse component can diverge like \(\rho^{-1}\) near \(0\), and if \(\hat C_p\) is singular the analogous singularity occurs near \(1\).  Proposition~\ref{prop:chi-feasible} gives the corresponding endpoint rank conditions.

\paragraph{Exact integrals through generalized eigenvalue implementation.}
For a fixed \(\zeta>0\), all $\rho$-dependent linear systems can be reduced to a single generalized eigendecomposition of the pair \((\hat C_p+\zeta \idm,\hat C_q+\zeta \idm)\); see details in \cite{Bach2026}. Note, however, that all analyses will always be carried through the integral representation.

\paragraph{Quadrature.} In computations of asymptotic limits in \mysec{hd}, there is no simple closed form, and the continuum integral in \eqref{eq:vhat-spec-ridge} can be estimated by a \emph{deterministic} quadrature rule \(\sum_{a=1}^{\Nquad} w_a F(\rho_a) \approx \int_0^1 F(\rho)d\rho\), with $\Nquad$ nodes \(\rho_a\in(0,1)\), and $F$ a function to integrate (we use Gauss-Legendre quadrature~\cite{Gautschi2004} in experiments).  This is not needed to define the estimator, but it is needed for computing the asymptotic equivalents when no simple closed form is available.

\paragraph{Companion potential $w$.}
The spectral construction also produces the second potential for the two-potential Fenchel form introduced in \eq{DKL2} that will be used for the $\mathcal{L}$ criterion below:
\[
 \hat w_{{\rm spec}}(y)
 =\int_0^1 2(1-\rho)
 \Big[-\hat u(\rho,y)-\frac{1-\rho}{2}\hat u(\rho,y)^2\Big] d \rho.
\]
As shown by \cite{Bach2026}, this construction ensures the two learned potentials satisfy, for all $y \in \rb^m$, $\hat w_{{\rm spec}}(y) \leqslant 1 - e^{\hat v_{{\rm spec}}(y)}$, which is needed for the variational formulation in \eq{DKL2}.

\subsection{Evaluation criteria}

We use three population KL variational criteria, with corresponding nonnegative gaps. 

\paragraph{Criterion \(\mathcal J\): population log-normalized variational score.} We define
\[
\J(v)=\mathbb{E}_p[v(x)]-\log\big(\mathbb{E}_q[e^{v(y)}]\big),
 \qquad
 \mathcal R_{\J}(v)=D(p\|q)-\J(v).
\]
This is the population Donsker-Varadhan variational form of KL \cite{DonskerVaradhan1975,NguyenWainwrightJordan2010}. It is invariant under adding constants to \(v\). For unrestricted measurable \(v\), the maximizer is \(v_\ast\) modulo constants and the maximum of \(\J(v)\) is \(D(p\|q)\).

\paragraph{Criterion \(\mathcal K\): population exponential/Fenchel KL score.} We define
\[
\K(v)=\mathbb{E}_p[v(x)]-\mathbb{E}_q[e^{v(y)}-1],
 \qquad
 \mathcal R_{\K}(v)=D(p\|q)-\K(v),
\]
corresponding to  the variational divergence-estimation framework of~\cite{NguyenWainwrightJordan2010}. Unlike criterion \(\mathcal J\), it is not invariant to additive constants, and therefore checks whether the normalizing constant is well learned.

Indeed, define the population exponential normalizer
\(
\Zq(v)=\mathbb{E}_q [e^{v(y)}].
\)
Then
\begin{equation*}
\K(v)=\J(v)+\log \Zq(v)-\Zq(v)+1,
\quad \mbox{ and } \quad
 \mathcal R_{\K}(v)=\mathcal R_{\J}(v)+\Zq(v)-\log\Zq(v)-1.
\end{equation*}
Since \(z-\log z-1\geqslant0\) for \(z>0\), criterion \(\mathcal K\) is at least as stringent as criterion \(\mathcal J\), that is,
\(
\mathcal R_{\K}(v)\geqslant\mathcal R_{\J}(v),
\)
with equality if and only if \(\Zq(v)=1\) (i.e., a normalized potential).

\paragraph{Criterion $\mathcal L$: two-potential Fenchel lower-bound score.}
The spectral framework also returns a pair of potentials $(v,w)$ such that for all $y \in \rb^m$, \(
 w(y)\leqslant  1-e^{v(y)}.
\)
The lower-bound score from \eq{DKL2} defines the criterion
\[
 \mathcal L(v,w)=\mathbb E_p[v(x)]+\mathbb E_q[w(y)],
 \qquad
 \mathcal R_{\mathcal L}(v,w)=D(p\|q)-\mathcal L(v,w).
\]
Because feasibility implies $\mathcal L(v,w)\leqslant\mathcal K(v)$, the corresponding excess risk satisfies
\(
 \mathcal R_{\mathcal L}(v,w)\geqslant \mathcal R_{\mathcal K}(v).
\)
Thus $\mathcal L$ is a stricter diagnostic for a two-potential spectral fit: it penalizes not only the quality of $v$ but also whether the companion $w$ realizes a tight Fenchel lower bound.

\subsection{Population scoring identities}
\label{sec:scoring}

Given affine and quadratic scores, we can compute all criteria. We only consider the criteria $\mathcal{J}$, $\mathcal{K}$, $\mathcal{L}$, with the nonnegative gaps being obtained by taking the complement to $D(p\|q)$.
 
\begin{proposition}[Affine scoring identities]
\label{prop:affine-scores}
Let \(v(x)=\theta^\top x+c\).
Then
\begin{align}
 \J(v)
 &=\theta^\top\Delta-\frac12\|\theta\|^2,
 \label{eq:affine-J}\\
% \mathcal R_{\J}(v)
% &=\frac12\|\theta-\Delta\|^2,
% \label{eq:affine-RJ}\\
 \K(v)
 &=\theta^\top\Delta+c+1-e^{c+\|\theta\|^2/2}.
 \label{eq:affine-K}
% \mathcal R_{\K}(v)
% &=\frac12\|\theta-\Delta\|^2+e^{c+\|\theta\|^2/2}-c-\|\theta\|^2/2-1
 %\label{eq:affine-RK}
\end{align}
\end{proposition}

\begin{proof}
Under \(q\), \(\theta^\top y\sim\Normal(0,\|\theta\|^2)\). Hence
\(
 \mathbb{E}_q [e^{\theta^\top y+c}]=e^{c+\|\theta\|^2/2}.
\)
Under \(p\), \(\mathbb{E}_p[\theta^\top x+c]=\theta^\top\Delta+c\). Therefore
\(
 \J(v)=\theta^\top\Delta+c-(c+\|\theta\|^2/2)=\theta^\top\Delta-\|\theta\|^2/2,
\)
which proves \eqref{eq:affine-J}. The Fenchel score is
\(
 \K(v)=\theta^\top\Delta+c-(e^{c+\|\theta\|^2/2}-1),
\)
which gives \eqref{eq:affine-K}.
\end{proof}

\begin{proposition}[Quadratic scoring identities]
\label{prop:quadratic-scores}
Let
\(
v(x)=x^\top A x+\xi^\top x+c,
\)  with $A=A^\top$
and  \(\idm-2A\succ0\). Define
\begin{equation}
 \Lambda(A,\xi,c)=c-\frac12\log\det(\idm-2A)+\frac12\xi^\top(\idm-2A)^{-1}\xi,
 \label{eq:L-logZ}
\end{equation}
so that \(\Zq(v)=e^{\Lambda(A,\xi,c)}\). Then
\begin{align}
 \J(v)
 =&\ \tr(A)+\Delta^\top A\Delta+\xi^\top\Delta
 +\frac12\log\det(\idm-2A)
 -\frac12\xi^\top(\idm-2A)^{-1}\xi,
 \label{eq:J-quad}\\
 \K(v)
 =&\ \tr(A)+\Delta^\top A\Delta+\xi^\top\Delta+c+1-e^{\Lambda(A,\xi,c)}.
 \label{eq:K-quad}
\end{align}
\end{proposition}

\begin{proof}
For \(x\sim\Normal(\Delta,\idm)\), write \(x=\Delta+G\), \(G\sim\Normal(0,\idm)\). Then
\(
 \mathbb{E}_p[x^\top A x]=\tr(A)+\Delta^\top A\Delta\), and 
\( \mathbb{E}_p[\xi^\top x]=\xi^\top\Delta.
\)
Thus,
\(
 \mathbb{E}_p[v(x)]=\tr(A)+\Delta^\top A\Delta+\xi^\top\Delta+c.
\)
For \(y\sim\Normal(0,\idm)\), the Gaussian integral is finite exactly when \(\idm-2A\succ0\), and completing the square gives
\(
 \mathbb{E}_q [e^{y^\top A y+\xi^\top y+c}]
 =e^c\det(\idm-2A)^{-1/2}
 \exp\big(\frac12\xi^\top(\idm-2A)^{-1}\xi\big).
\)
Taking logs gives \eqref{eq:L-logZ}. Substituting \(\mathbb{E}_p [v]\) and \(\log\Zq(v)\) into \(\J\) proves \eqref{eq:J-quad}; substituting \(\Zq(v)=e^{\Lambda(A,\xi,c)}\) into \(\K\) proves \eqref{eq:K-quad}.
\end{proof}

\begin{remark}[Infinite Gaussian exponential moment]
If \(\idm-2A\) is not positive definite, then \(\mathbb{E}_q [e^{v(y)}]=+\infty\). In that case \(\J(v)=-\infty\), \(\K(v)=-\infty\), and the corresponding population gaps are \(+\infty\). The spectral estimator in \eqref{eq:vhat-spec-ridge-quad} has \(\hat A\preccurlyeq0\) whenever the integral is finite, because \(\hat A\) is an integral of negative semidefinite rank-one matrices. Hence \(\idm-2\hat A\succ0\) automatically for the fitted spectral potential.
\end{remark}

\begin{proposition}[Quadratic two-potential scoring identity]
\label{prop:quadratic-L-scores}
Let
\(
 v(x)=x^\top A_vx+\ell_v^\top x+c_v
 \)  and 
\( w(y)=y^\top A_wy+\ell_w^\top y+c_w,
\)
with $A_v=A_v^\top$ and $A_w=A_w^\top$. Then
\begin{equation}
 \mathcal L(v,w)
 =\tr(A_v)+\Delta^\top A_v\Delta+\ell_v^\top\Delta+c_v+\tr(A_w)+c_w.
 \label{eq:L-two-quad}
\end{equation}
\end{proposition}

\begin{proof}
For $x\sim\Normal(\Delta,\idm)$,
$\mathbb E_p[x^\top A_v x]=\tr(A_v)+\Delta^\top A_v\Delta$ and
$\mathbb E_p[\ell_v^\top x]=\ell_v^\top\Delta$.  For $y\sim\Normal(0,\idm)$,
$\mathbb E_q[y^\top A_wy]=\tr(A_w)$ and $\mathbb E_q[\ell_w^\top y]=0$.
Substitution in the definition of $\mathcal L$ gives \eqref{eq:L-two-quad}.  
\end{proof}

\subsection{Finite-sample unregularized feasibility}
\label{sec:zeroridge}
We consider here the limit when regularization parameters go to zero (see proofs in Appendix~\ref{app:proof-var-dual} and Appendix~\ref{app:proof-chi-feasible}).  
The following proposition introduces the finite-sample feasibility condition for the unregularized
log-normalized KL fit in the present linear-feature model. It is the analogue, in this Gaussian linear
setting, of the convex-hull existence criterion for related direct density-ratio estimators
\cite{banzato2025existence}.

\begin{proposition}[Variational dual, finite value, and attainment]
\label{prop:var-dual}
For fixed \(y_1,\ldots,y_{n_q} \in \rb^m \), let
\[
 \mathcal{C}=\conv\{y_1,\ldots,y_{n_q}\},
\]
the convex hull of the $n_q$ points $y_1,\ldots,y_{n_q}$ in $\mathbb{R}^m$.
The supremum is finite if and only if \(\hat\mu_p\in \mathcal{C}\). A finite maximizer \(\hat\theta_{\rm var}\) exists if and only if \(\hat\mu_p\in\operatorname{ri} \mathcal{C}\), where the relative interior is taken in \(\operatorname{aff} \mathcal{C}\). When \(\operatorname{aff} \mathcal{C}\ne\R^m\), the maximizer is identifiable only modulo \((\operatorname{span}(\mathcal{C}-\mathcal{C}))^\perp\). If \(\hat\mu_p\in \mathcal{C}\setminus\operatorname{ri} \mathcal{C}\), the supremum is finite but is not attained in the identifiable quotient; maximizing sequences escape to infinity in directions exposing the smallest face of \(\mathcal{C}\) that contains \(\hat\mu_p\).
\end{proposition}

\begin{proposition}[Spectral endpoint feasibility for the full potential]
\label{prop:chi-feasible}
Under the Gaussian sampling model, the spectral potential in \eqref{eq:vhat-cont} is finite almost surely if and only if $n_p, n_q \geqslant m+1$.
\end{proposition}

A different spectral estimator could be defined by replacing inverses with Moore-Penrose pseudoinverses, or equivalently by taking a ridge limit \((\hat M(\rho)+\zeta I)^{-1}\hat{\Delta}\) as \(\zeta\downarrow0\) whenever the limit exists pointwise. This is analogous to the implicit bias of unregularized gradient methods, which select particular limiting solutions in underdetermined or separable problems \cite{SoudryHofferNacsonGunasekarSrebro2018,GunasekarLeeSoudrySrebro2018}.  In this paper, for unregularized spectral full-potential risk comparisons, we only consider the case $n_p, n_q \geqslant m+1$, that is, $\alpha_p, \alpha_q < 1$, leaving the overparameterized case to future work.

\section{Population behavior (\(\alpha_p=\alpha_q=0\))}
\label{sec:pop-bench}

This section studies the population limits of the two estimators, that is, \(\alpha_p=\alpha_q=0\). We first compute the positive-ridge population criteria and then recover the unregularized criteria by sending the ridge parameters to zero.

\subsection{Variational estimator}

For the variational estimator with ridge parameter \(\tau>0\), the population objective over affine coefficients is
\(
  \theta^\top\Delta-\frac12\|\theta\|^2-\frac{\tau}{2}\|\theta\|^2.
\)
It is strictly concave, and its unique maximizer is
\(
  \theta^{(0)}=\frac{\Delta}{1+\tau}.
\)
With the population-normalized representative obtained from $\theta^{(0)}$, the intercept is
\(
  c^{(0)}=-\frac12\|\theta^{(0)}\|^2=-\frac{s}{2(1+\tau)^2}
\), leading to the potential
\(
v_{\mathrm{var}}^{(0)}(x)
=
 x^\top \theta ^{(0)}+c^{(0)}
=
\frac{\Delta^\top x}{1+\tau}
-
\frac{s}{2(1+\tau)^2}.
\)  Proposition~\ref{prop:affine-scores} then gives the risks
\[
  \mathcal R_{\J,{\rm var}}^{(0)}
  =\mathcal R_{\K,{\rm var}}^{(0)}
  =\frac{s}{2}\left(\frac{\tau}{1+\tau}\right)^2 .
\]
Letting \(\tau\downarrow0\) gives \(\theta^{(0)}\to\Delta\), \(c^{(0)}\to -s/2\), and hence the zero-ridge population fitted potential is (unsurprisingly) exactly
\(
  v_{\rm var}^{(0)}(x)=v_\ast(x)=\Delta^\top x-\frac{s}{2}.
\)
Thus, the zero-ridge population fit recovers the true log-density ratio exactly, and the population risks vanish,  that is,
\(
  \mathcal R_{\J,{\rm var}}^{(0)} = 
  \mathcal R_{\K,{\rm var}}^{(0)}=0.
\)

\subsection{Spectral estimator}
\label{sec:specpop}

At the population level, we have, from Eq.~(\ref{eq:murho}),
\(
  C(\rho)=\idm,  \mu(\rho)=\rho\Delta,\) 
  and \(
   \hat{\Delta} \to \Delta.
\)
For fixed \(\zeta\geqslant 0\) and $\rho \in [0,1]$, the affine functions in \eq{beta-ridge-spectral} are determined by
\begin{equation}
  \beta^{(0)}(\rho)=
  \frac{\Delta}{1+\zeta+s\rho(1-\rho)},
  \qquad
  u^{(0)}(\rho)(x)=
  \frac{\Delta^\top x-\rho s}{1+\zeta+s\rho(1-\rho)}.
  \label{eq:u-pop-ridge}
\end{equation}
Integrating with respect to $\rho$ will require the following integral which has a closed-form expression:
\begin{equation*}
  h(s)=\int_0^1\frac{\dd\rho}{1+\zeta+s\rho(1-\rho)} =
  \frac{4}{\sqrt{s(s+4(1+\zeta))}}
  \artanh\sqrt{\frac{s}{s+4(1+\zeta)}},
\end{equation*}
Substituting \eqref{eq:u-pop-ridge} into the continuum spectral integral gives, after some computations (see  Appendix~\ref{app:proof-pop-spec-ridge}), the scalar quadratic potentials
\BEAS
  \!v_{{\rm spec}}^{(0)}(x)
 &\!\!\!\!=\! \!\!\!& h'(s)(\Delta^\top x)^2+(h(s)-s h'(s))(\Delta^\top x)   +(1+\zeta)h(s)-1+
  \Big((1+\zeta)s+\frac{s^2}{2}\Big)h'(s) \\
 \! w_{{\rm spec}}^{(0)}(x)
  &\!\!\!\!= \!\!\!\!&  \Big( \!-\frac{h(s)+(s+2(1+\zeta))h'(s)}{2(1+\zeta)} \Big) (\Delta^\top x)^2 -  ( h(s)+s h'(s) ) (\Delta^\top x)   +   
  1-(1+\zeta)h(s)-(1+\zeta)s h'(s).
\EEAS
We can then use the scoring identities from \mysec{scoring}.  A key deviation from the variational estimator is the presence of a bias even when $\zeta=0$. As shown in \mysec{hd} below, this comes with a lower variance.

\section{Proportional high-dimensional limits with fixed signal}
\label{sec:hd}
A classical analysis of our estimator can be performed with $m$ fixed, and $n_p, n_q$ tending to infinity, that is, $\alpha_p, \alpha_q$ tending to zero, with classical tools from asymptotic statistics~\cite{van2000asymptotic}. This is only valid in small-dimensional problems (and considered in \mysec{weak}). In this section, we consider the high-dimensional limit with \emph{any} $\alpha_p,\alpha_q$. This section thus imposes the following classical assumptions.

\begin{assumption}[Proportional regime with fixed signal]
\label{ass:hd}
As \(m,n_p,n_q\to\infty\),
\[
\frac{m}{n_p}\to\alpha_p\in[0,\infty),
 \qquad
 \frac{m}{n_q}\to\alpha_q\in[0,\infty).\]
The samples from \(p\) and \(q\) are independent.
\end{assumption}

Under this assumption, we have limits in probability\footnote{These limits are immediate consequences of the law of large numbers for chi-square variables and standard Gaussian concentration \cite{Vershynin2018}. Indeed,
\(\hat\mu_p=\Delta+n_p^{-1/2}G_p\) and \(\hat\mu_q=n_q^{-1/2}G_q\), with independent standard Gaussian vectors \(G_p,G_q\in\R^m\). Hence
\(
 \|\hat\mu_p\|^2=s+2n_p^{-1/2}\Delta^\top G_p+n_p^{-1}\|G_p\|^2\to s+\alpha_p,
\)
because \(n_p^{-1/2}\Delta^\top G_p\to0\) in probability and \(n_p^{-1}\|G_p\|^2=(m/n_p)(m^{-1}\|G_p\|^2)\to\alpha_p\). The proof for \(\hat{\Delta}\) is identical, with noise covariance matrix $(n_p^{-1}+n_q^{-1})\idm  = (\alpha_p + \alpha_q) \frac{1}{m} \idm$. }
\begin{equation*}
 \|\hat\mu_p\|^2 \to s+\alpha_p,
 \qquad
 \|\hat{\Delta}\|^2\to s+\alpha_p+\alpha_q.
\end{equation*}
For this model where only $\Delta$ is unknown, a simple estimator is $\hat{\Delta} = \hat{\mu}_p - \hat{\mu}_q$ for which $\frac{1}{2} \| \Delta - \hat{\Delta} \|^2 \to \frac{1}{2}(\alpha_p+\alpha_q)$ (which is then equal to the population gap between the estimate of the $f$-divergence and the true value), and is both much simpler to compute and typically has better performance. Our goal in this paper is to study and compare estimators that apply beyond linear features and beyond the Gaussian model, keeping in mind that they can be far from optimal in this simple setup.

\subsection{Regularized variational deterministic equivalent}
\label{subsec:var-ridge-hd}
\label{sec:var-ridge-hd}

The CGMT reduction theorem below will make use of a Gaussian random variable $H \sim \mathcal{N}(0,1)$.
For a nonnegative deterministic measurable function \(A:\R\to[0,\infty)\), write \(A=A(H)\).  Thus \(A(H)\) is a random variable only through the scalar Gaussian input \(H\), while the map \(A(\cdot)\) itself is deterministic.  

\begin{theorem}[Variational deterministic equivalent]
\label{thm:var-ridge}
Under Assumption~\ref{ass:hd}, assume  \(\tau>0\)
and
\(s+\alpha_p+\alpha_q>0\).  The ridge variational estimator has a deterministic limit characterized by \begin{equation}
 \inf_{A\geqslant0,\,\E [A] =1}
    \E[A\log A - A + 1]
    +\frac{1}{2\tau}\big( ( s+\alpha_p+\alpha_q\E[A^2])^{1/2}-\E[AH]\big)_+^2.
 \label{eq:ridge-cgmt-main}
\end{equation}
 Moreover, there are (deterministic) scalars \(\xi\geqslant 0\), \(B>0\), and \(\eta\in\R\) such that the optimizer satisfies
\begin{align}
 \log A+\frac{\xi\alpha_q}{B}A&=\eta+\xi H \mbox{ almost surely} ,
  \notag\\
 \E [A] &=1,
 \label{eq:ridge-mass-main}\\
s+\alpha_p+\alpha_q\E[A^2] &= B^2,
 \label{eq:ridge-B-main}\\
 \E[AH]&=B-\tau\xi.
 \label{eq:ridge-alignment-main}
\end{align}
The estimator obeys the following projection limits (in probability):   \begin{equation}
 \|\hat\theta\|^2\to \xi^2,
 \qquad
 \hat\theta^\top\Delta\to \frac{\xi s}{B}.
 \label{eq:ridge-var-projections-main}
\end{equation}
 The empirically normalized intercept has a limit (in probability)
\begin{equation}
 \hat c\to \E[A\log A]+\tau\xi^2-\frac{\xi(s+\alpha_p)}{B}.
 \label{eq:c-tau-limit}
\end{equation}
Consequently, the population log-normalizer limit is
\(
 z
 =\E[A\log A]+\tau\xi^2-\frac{\xi (s+\alpha_p)}{B}+\frac12\xi^2,
\)
and the evaluation criteria are
$
 \mathcal J_{\rm var}
 = \frac{\xi s}{B} - \frac{\xi^2}{2}
$ and $
\mathcal K_{\rm var}
 =\mathcal J_{\rm var} - \exp(z)+z+ 1.
$
\end{theorem}

\paragraph{Proof sketch.} For a detailed proof, see Appendix~\ref{app:var-ridge}. The proof proceeds in 5 steps, which we develop here without full checks of regularity conditions. We use the notation $r^2 = s+\alpha_p$, $\mu = \hat{\mu}_p$ and $n = n_q$, $\alpha = \alpha_q$, and $Y \in \rb^{n \times m}$ the data matrix of the samples from $q$, which is a standard Gaussian matrix.
We condition on \(\hat\mu_p=\mu\). It is enough to prove the result
for deterministic sequences such that
\(\|\mu\|^2\to r^2\) and \(\Delta^\top\mu\to s\), since
\(\hat\mu_p\) is independent of the \(q\)-sample and satisfies these
limits in probability.

\emph{Step 1: ridge entropy dual.}
The optimization problem for the variational method has a traditional primal-dual representation as (with an unusual normalization for $a$ that will be useful later)
\BEAS
\max_{\theta \in \rb^m} \theta^\top   \mu - \log \Big( \frac{1}{n} \sum_{j=1}^n e^{\theta^\top y_j} \Big) - \frac{\tau}{2} \| \theta\|^2& = &\!\!\! \min_{a\in\R_+^n, \frac{1}{n}\sum_{j=1}^n a_j=1}
 \frac1n\sum_{j=1}^n a_j\log a_j
 +\frac{1}{2\tau}\Big\|\mu-\frac1nY^\top a 
 \Big\|^2\\
 & = & \!\!\!\!\min_{a\in\R_+^n, \frac{1}{n}\sum_{j=1}^n a_j=1}  \max_{\theta \in \rb^m}
 \frac1n\sum_{j=1}^n a_j\log a_j +
\theta^\top  \mu-\frac1n\theta^\top Y^\top a - \frac{\tau}{2} \| \theta\|^2
,
\EEAS
 with the optimal $\hat{\theta}$ obtained as  $\hat{\theta} = \frac{1}{\tau}( \mu - \frac1nY^\top a)$.
 Note that under the constraint \(n^{-1}\sum_{j=1}^n a_j=1\), the term
\(a_j\log a_j\) may equivalently be replaced by
\(a_j\log a_j-a_j+1\).
 
\emph{Step 2: CGMT min-max reduction.} 
Following the classical CGMT reduction~\cite{thrampoulidis2014gaussian,TOH2015}, we define a new min-max optimization problem with $g \in \rb^m,h \in \rb^n$ two independent standard Gaussian vectors, where the term 
$-\frac1n\theta^\top Y^\top a$ is replaced with $ \frac1n \| a\| g^\top \theta - \frac{1}{n} \| \theta\| h^\top a$:
\BEAS
& &  \min_{a\in\R_+^n, \frac{1}{n}\sum_{j=1}^n a_j=1}  \max_{\theta \in \rb^m}\ 
 \frac1n\sum_{j=1}^n a_j\log a_j +
 \theta^\top  \mu + \frac1n \| a\| g^\top \theta - \frac{1}{n} \| \theta\| h^\top a  - \frac{\tau}{2} \| \theta\|^2  .
 \EEAS
With proper regularity conditions (see Appendix~\ref{app:var-ridge}), this problem will enable us to approximate the limit in probability of the original problem. We went from $nm$ Gaussian entries to  $n + m$ Gaussian entries.
 The auxiliary problem can then be reduced by optimizing over the
direction of the maximization variable:
\BEAS
&   &   \min_{a\in\R_+^n, \frac{1}{n}\sum_{j=1}^n a_j=1}  \max_{\xi  \in \rb_+}
\max_{\theta \in \rb^m, \| \theta\| = \xi}
 \frac1n\sum_{j=1}^n a_j\log a_j +
 \theta^\top  \mu + \frac1n \| a\| g^\top \theta - \frac{1}{n} \| \theta\| h^\top a  - \frac{\tau}{2} \| \theta\|^2 
\\
& = &   \min_{a\in\R_+^n, \frac{1}{n}\sum_{j=1}^n a_j=1}  \max_{\xi  \in \rb_+}
 \frac1n\sum_{j=1}^n a_j\log a_j +
 \xi \big\|   \mu + \frac1n \| a\| g \big\|  - \frac{1}{n} \xi h^\top a  - \frac{\tau}{2} \xi^2 
\\
& = &   \min_{a\in\R_+^n, \frac{1}{n}\sum_{j=1}^n a_j=1}  
 \frac1n\sum_{j=1}^n a_j\log a_j +
  \frac{1}{2\tau} \Big(\big\|   \mu + \frac1n \| a\| g \big\|  - \frac{1}{n}   h^\top a \Big)_+^2     ,
 \EEAS
 where the last equality optimizes over \(\xi\geqslant 0\). When the denominator
is nonzero, the corresponding auxiliary maximizer $\theta$ is equal to
\(
\xi\,
\frac{\mu+n^{-1}\|a\|g}
{\|\mu+n^{-1}\|a\|g\|}.
\)

 \emph{Step 3: Reduction of stochastic program.} 
We have
\[
\Big\|   \mu + \frac1n \| a\| g \Big\|^2 =
\| \mu\|^2 + \frac{1}{n^2} \|a\|^2 \| g\|^2 + \frac{2}{n} \mu^\top g \| a \|,
\]
with $\mu^\top g = O_\P(1)$ (in probability) and thus the last term is negligible, and $\frac{1}{m}\| g\|^2$ tends to one in probability, while $\| \mu\|^2 \to r^2$, leading to the asymptotically equivalent problem (in min-max and min forms)
\BEAS
 & & \min_{a\in\R_+^n, \frac{1}{n}\sum_{j=1}^n a_j=1}  \max_{\xi  \in \rb_+}
 \frac1n\sum_{j=1}^n a_j\log a_j +
 \xi \Big(  r^2  +\frac{\alpha }{n} \| a\|^2\Big)^{1/2}  - \frac{1}{n} \xi h^\top a  - \frac{\tau}{2} \xi^2  \\
 & = & 
  \min_{a\in\R_+^n, \frac{1}{n}\sum_{j=1}^n a_j=1}   
 \frac1n\sum_{j=1}^n a_j\log a_j +
\frac{1}{2\tau} \Big( \big(  r^2  + \frac{\alpha }{n} \| a\|^2\big)^{1/2}  - \frac{1}{n}   h^\top a \Big)_+^2 ,
  \EEAS
with now a single Gaussian vector $h \in \rb^n$.  From optimality conditions, there exists a deterministic\linebreak
function $\varphi$ such that $a_j = \varphi(h_j)$ for all $j \in \{1,\dots,n\}$. 
When $n$ tends to infinity, empirical averages $\frac{1}{n} \sum_{j=1}^n a_j$,
$\frac{1}{n} \sum_{j=1}^n a_j^2$, $\frac{1}{n} \sum_{j=1}^n a_j \log a_j$, $\frac{1}{n} \sum_{j=1}^n a_j h_j$ become expectations $\E[A]$, $\E[A^2]$, $\E[ A \log A] $, $\E[AH]$,\linebreak
with $A$ a deterministic function of $H$, a standard Gaussian variable, solving the following problem (which is exactly 
\eq{ridge-cgmt-main}):
\[
 \min_{A \geqslant 0,\  \E[A] = 1 }   
 \E[ A \log A - A + 1] +
\frac{1}{2\tau} \Big( \big(  r^2  + \alpha\E[A^2] \big)^{1/2}  -\E[ AH] \Big)_+^2.
\]

\emph{Step 4: KKT equations.} Using the Lagrange multiplier $\xi$ from above and a multiplier $\eta$ for the constraint $\E[A] = 1$, the stationary equation becomes
\[
\log A + \alpha \frac{ \xi }{ B} A = \xi H + \eta \mbox{ almost surely},
\]
for $B = \big(  r^2  + \alpha \E[A^2] \big)^{1/2}$, with $\xi = \frac{1}{\tau} ( B - \E[AH])_+$, leading to the stated optimality conditions.

\emph{Step 5: estimator observables and scoring.}
By CGMT localization (see Appendix~\ref{app:var-ridge} for details), the primary optimizer has the same limiting
observables as the auxiliary maximizer for the norm and the projections
onto \(\mu\) and \(\Delta\). For the auxiliary maximizer,
\(
\theta_{\rm }
=
\xi\frac{\mu+n^{-1}\|a\|g}{\|\mu+n^{-1}\|a\|g\|}
+o_\P(1).
\)
Moreover, we have  the limit 
\(
\|\mu+n^{-1}\|a\|g\|\to B\) in probability, as well as  
\(
\mu^\top(\mu+n^{-1}\|a\|g)\to r^2, 
\Delta^\top(\mu+n^{-1}\|a\|g)\to s.
\)
Therefore
\(
\|\hat\theta\|\to\xi, 
\hat\theta^\top\mu\to\frac{\xi r^2}{B}, 
\hat\theta^\top\Delta\to\frac{\xi s}{B}.
\)

Value convergence also gives
\(
\frac1n\sum_{j=1}^n   a_j\log  a_j
\to
\E[A\log A].
\)
At the saddle point, we have the identities
\(
  a_j=\exp\{y_j^\top\hat\theta+\hat c\},
\frac1nY^\top  a=\mu-\tau\hat\theta .
\)
Hence, the exact identity
\(
\hat c
=
\frac1n\sum_{j=1}^n   a_j\log  a_j
-\hat\theta^\top\mu+\tau\|\hat\theta\|^2
\)
implies
\(
\hat c
\to
\E[A\log A]+\tau\xi^2-\frac{\xi r^2}{B}
=
\E[A\log A]+\tau\xi^2-\frac{\xi(s+\alpha_p)}{B}.
\)
Consequently,
\(
\log Z_q(\hat v_{\rm var})
=
\hat c+\frac12\|\hat\theta\|^2
\to
z.
\)
Using the affine scoring identities,
\(
\mathcal J(\hat v_{\rm var})
=
\hat\theta^\top\Delta-\frac12\|\hat\theta\|^2
\to
\frac{\xi s}{B}-\frac12\xi^2
=
\mathcal J_{\rm var}.
\)
Finally, by the identity
\(\mathcal K(v)=\mathcal J(v)+\log Z_q(v)-Z_q(v)+1\), we get
\(
\mathcal K(\hat v_{\rm var})
\to
\mathcal J_{\rm var}+z-e^z+1
=
\mathcal K_{\rm var}.
\)

\paragraph{Empirical evaluation.}
In our  experiments, the variational limit is computed by Gauss-Hermite quadrature~\cite{Gautschi2004} for the expectations over $H$.  For fixed $\xi$ and $B$, set $c=\xi\alpha_q/B$.  The scalar equation in Theorem~\ref{thm:var-ridge} is $\log A+cA=\eta+\xi H$.  If $c>0$, exponentiating gives $cA e^{cA}=c e^{\eta+\xi H}$, so the solution is
\(
 A=\frac{1}{c}W_0(c e^{\eta+\xi H}),
\)
where $W_0$ is the principal branch of the Lambert $W$ function \cite{CorlessGonnetHareJeffreyKnuth1996}.  When $c=0$, this is replaced by the limiting expression $A=e^{\eta+\xi H}$.  The multiplier $\eta$ is chosen by one-dimensional bisection to enforce $\mathbb E[A]=1$, and the remaining two equations
\[
 B^2=s + \alpha_p +\alpha_q\mathbb E[A^2],
 \qquad
 \mathbb E[AH]=B-\tau\xi,
\]
are solved by a safeguarded Newton method for nonlinear equations, using line-search or trust-region globalization~\cite{NocedalWright2006}.  The unregularized computation is recovered by the zero-ridge limiting convention $\tau\downarrow0$, for which the second equation becomes $\mathbb E[AH]=B$ on the strict feasible side. Note that the equation $\log A+cA=\eta+\xi H$ can be given an interpretation in terms of the proximal operator for the entropy~\cite{salehi2019impact}.

\subsection{Zero-ridge variational limit}
\label{subsec:zero-ridge-var}
\label{sec:zero-ridge-var}

We simply take the limit of \eq{ridge-cgmt-main} when $\tau$ tends to zero, where the term
$\frac{1}{2\tau}\big( ( s+\alpha_p+\alpha_q\E[A^2])^{1/2}-\E[AH]\big)_+^2$ imposes the constraint
$(( s+\alpha_p+\alpha_q\E[A^2])^{1/2}-\E[AH])_+=0$, that is,
\( ( s+\alpha_p+\alpha_q\E[A^2])^{1/2}-\E[AH] \leqslant 0. \) This leads to the optimization problem
\[
\inf_{A\geqslant0,\,\E [A] =1}
    \E[A\log A - A + 1]
\   \mbox{ such that }\  ( s+\alpha_p+\alpha_q\E[A^2])^{1/2}-\E[AH] \leqslant 0,
\]
which is not always feasible. It is if and only if there exists $A$ such that 
$ ( s+\alpha_p+\alpha_q\E[A^2])^{1/2}-\E[AH] \leqslant 0$, that is,
$ s+\alpha_p+\alpha_q \E[A^2] \leqslant ( \E[AH] )^2$ (we can always assume $\E[AH] \geqslant 0$ by symmetry). Hence, the strict zero-ridge feasible phase is characterized by
\[
s+\alpha_p
<
\sup_{A\geqslant 0, \ \mathbb E [A]=1}
    \bigl(\mathbb E[AH]\bigr)^2
    -
    \alpha_q\mathbb E[A^2] : = R_{\rm hull}^2(\alpha_q).
\]
It remains only to evaluate the scalar variational problem defining $R_{\rm hull}^2(\alpha_q)$ above.
We give an explicit formula in Appendix~\ref{app:Rhull}, where we show it is negative for $\alpha_q>1/2$.
Therefore, for the two-sample
problem, the strict zero-ridge feasible region is
\[
        0\leqslant \alpha_q< \frac12,
        \qquad
        s+\alpha_p<R_{\rm hull}^2(\alpha_q).
\]
The function $\alpha \mapsto R_{\rm hull}^2(\alpha)$ is decreasing on $(0,1/2]$, equal to zero at $1/2$ and tending to infinity as $2 \log (1/\alpha)$ around zero (see illustrative plot in \myfig{regimes}).

On this region, Proposition~\ref{prop:var-dual} gives finite attainment of the unregularized
variational estimator with probability tending to one. Moreover, all expressions for the limiting estimators and scorings from Theorem~\ref{thm:var-ridge} apply directly with $\tau=0$. Outside this region,
away from the critical boundary, the convex-hull constraint fails with
probability tending to one and the unregularized variational objective is
unbounded above.

\subsection{Regularized spectral deterministic equivalent}
\label{subsec:spec-ridge-hd}

We now state the deterministic asymptotic equivalent for the spectral estimator.
The result is finite for all fixed finite aspect ratios because the ridge parameter
keeps every regularized covariance matrix with strictly positive eigenvalues. To obtain
a high-dimensional limit, we study the $m \times m$ matrix
\[
 \hat C(\rho)=\rho\hat C_p+(1-\rho)\hat C_q,
\]
which, up to the sample-mean vectors, is a weighted sum of two independent Wishart
matrices, for which existing results from random matrix theory could be used~\cite{BaiSilverstein2010}.
The proof is instead written in a ``random-feature style'': the two centered covariance
matrices are embedded into a single standard Gaussian matrix, while the mixture
parameter \(\rho\) appears only through a deterministic diagonal matrix. This allows to treat the statistical dependence between $\hat{C}(\rho)$ and $\hat{C}(\sigma)$, for $\rho,\sigma \in [0,1]$ potentially different.

The scalar
\(\omega(\rho)\) below is the inverse limiting Stieltjes transform at the negative ridge
point \(-\zeta\); equivalently, it is the effective, or ``self-induced'', ridge parameter in
the random-matrix formulation of ridge regression in
\cite{montanari2022interpolation,Bach2024RandomProjections}.

The following theorem characterizes the limit of the $\rho$-dependent least-squares estimator. All performance criteria can then be obtained by integration in \mysec{chi-quadrature}. Note that we obtain almost sure limits instead of limits in probability obtained in \mysec{var-ridge-hd} for the variational estimator.

\begin{theorem}[Regularized spectral resolvent limits]
\label{thm:spec-ridge}
Under Assumption~\ref{ass:hd}, assume  \(\zeta>0\).
For each \(\rho\in[0,1]\), let \(\omega(\rho)\) be the unique solution on
\((\zeta,\infty)\) of
\begin{equation}
 \omega(\rho)
 =
 \zeta+
 \frac{\rho\omega(\rho)}{\omega(\rho)+\alpha_p\rho}
 +\frac{(1-\rho)\omega(\rho)}{\omega(\rho)+\alpha_q(1-\rho)}.
 \label{eq:omega-zeta-hd}
\end{equation}
Define the normalized first and cross second degrees of freedom by
\begin{align}
 \df_1(\rho)
 &=
 \frac{\rho}{\omega(\rho)+\alpha_p\rho}
 +\frac{1-\rho}{\omega(\rho)+\alpha_q(1-\rho)}
 =1-\frac{\zeta}{\omega(\rho)} \in (0,1),
  \notag\\
 \df_2(\rho,\sigma)
 &=
 \frac{\alpha_p\rho\sigma}
 {(\omega(\rho)+\alpha_p\rho)(\omega(\sigma)+\alpha_p\sigma)}
 +
 \frac{\alpha_q(1-\rho)(1-\sigma)}
 {(\omega(\rho)+\alpha_q(1-\rho))(\omega(\sigma)+\alpha_q(1-\sigma))}
 \in [0,1).
 \label{eq:df2-zeta-hd}
\end{align}
Set
$
 \chi_2(\rho,\sigma)
 =
 \frac{1}{\omega(\rho)\omega(\sigma)(1-\df_2(\rho,\sigma))},
 $
$
 R=s+\alpha_p+\alpha_q,
 $ and $
 \kappa(\rho)=1+\frac{\rho(1-\rho)R}{\omega(\rho)}.
$
Then the following uniform convergences hold \emph{almost surely} for $\hat{\mu}(\rho)$ and $\hat{\beta}(\rho)$ defined in \eq{murho} and \eq{beta-ridge-spectral}:
\begin{align}
 \sup_{\rho\in[0,1]}
 \big|
 \Delta^\top\hat\beta(\rho)
 -
 \ell_\Delta(\rho)
 \big|
 &\longrightarrow 0,
 &
 \ell_\Delta(\rho)
 &:=
 \frac{s}{\omega(\rho)\kappa(\rho)},
 \label{eq:h-zeta-hd}\\
 \sup_{\rho\in[0,1]}
 \big|
 \hat\mu(\rho)^\top\hat\beta(\rho)
 -
 m(\rho)
 \big|
 &\longrightarrow 0,
 &
 m(\rho)
 &:=
 \frac{\rho(s+\alpha_p)-(1-\rho)\alpha_q}{\omega(\rho)\kappa(\rho)},
 \label{eq:m-zeta-hd}\\
  \sup_{\rho,\sigma\in[0,1]}
 \big|
 \hat\beta(\rho)^\top\hat\beta(\sigma)
 -
 k(\rho,\sigma)
 \big|
 &\longrightarrow 0,
 &
 k(\rho,\sigma)
 &:=
 \frac{R\chi_2(\rho,\sigma)}
 {\kappa(\rho)\kappa(\sigma)}.
 \label{eq:k-zeta-hd}
\end{align}
\end{theorem}

\begin{proof} The proof proceeds in four steps.

\emph{Step 1: Reduction to standard Gaussian matrix.} 
Let \(N=n_p+n_q\). We can realize the Gaussian sampling model on one matrix
\(Z\in\R^{N\times m}\) with i.i.d.~standard Gaussian entries. After an
orthogonal change of coordinates within the \(p\)-sample and within the \(q\)-sample,
the first row and the \((n_p+1)\)-th row generate the two sample means, while the
remaining rows generate the centered sample covariances. Define
\[
 \Pi_p=\Diag\big(0,I_{n_p-1},0,0_{n_q-1}\big)\in\R^{N\times N},
 \qquad
 \Pi_q=\Diag\big(0,0_{n_p-1},0,I_{n_q-1}\big)\in\R^{N\times N},
\]
and
\[
 \Gamma(\rho)
 =
 \rho\alpha_{p}\Pi_p+(1-\rho)\alpha_{q}\Pi_q\in\R^{N\times N}.
\]
Then, in distribution, with \(e_i\in\rb^N\) the \(i\)-th canonical basis vector,
\begin{align*}
 \hat\mu_p&=\Delta+n_p^{-1/2}Z^\top e_1,
 &
 \hat\mu_q&=n_q^{-1/2}Z^\top e_{n_p+1},\\
 \hat\Delta&=\hat\mu_p-\hat\mu_q,
 &
 \hat\mu(\rho)&=\rho\hat\mu_p+(1-\rho)\hat\mu_q,
\end{align*}
and
\[
 \hat C(\rho)
 =
 \rho\hat C_p+(1-\rho)\hat C_q
 =
 \frac1m Z^\top\Gamma(\rho)Z .
\]
The finite-rank rows corresponding to the sample means do not affect the
normalized trace limits. 

Set
\[
 Q(\rho)=(\hat C(\rho)+\zeta I)^{-1},
 \qquad
 \hat\kappa(\rho)
 =
 1+\rho(1-\rho)\hat\Delta^\top Q(\rho)\hat\Delta .
\]
The Sherman-Morrison formula gives, with the notations from
\mysec{est2}, in particular \eq{beta-ridge-spectral},
\begin{equation}
 \hat\beta(\rho)
 =
 \frac{Q(\rho)\hat\Delta}{\hat\kappa(\rho)}.
 \label{eq:beta-sm-hd}
\end{equation}
To compute scores required in \mysec{scoring}, we thus need limits for the quantities $\Delta^\top\hat\beta(\rho)$, $\hat\mu(\rho)^\top\hat\beta(\rho)
$, and $ \hat\beta(\rho)^\top\hat\beta(\sigma)$, for $\rho,\sigma\in[0,1]$. 

\emph{Step 2: Asymptotic limits using existing random matrix theory results.}
For the first two, which involve a single $\rho$, we can directly  apply Proposition~3.2 of \cite{Bach2024RandomProjections} to the transposed
matrix \(Z^\top\in\R^{m\times N}\), with sample size \(m\), ambient dimension
\(N\), deterministic covariance profile \(\Gamma(\rho)\), and spectral
parameter \(-\zeta\). The corresponding Stieltjes-transform equation gives
\[
 \omega(\rho)-\zeta
 =
 \lim_{m\to\infty}
 \frac1m\tr\!\left[
 \Gamma(\rho)
 (I+\omega(\rho)^{-1}\Gamma(\rho))^{-1}
 \right].
\]
Evaluating the two diagonal blocks gives exactly \eq{omega-zeta-hd}. The
positive solution is unique by the standard uniqueness of the Stieltjes-transform
solution at a negative spectral parameter.

For any bounded deterministic diagonal matrix \(A\in\R^{N\times N}\), set
$
 C_{A}=\frac1m Z^\top A Z.
$
The deterministic equivalent gives, for each fixed \(\rho\),
\begin{equation*}
 \frac1m\tr[C_{A}Q(\rho)]
 -
 \frac1m\tr\!\left[
 A(\Gamma(\rho)+\omega(\rho)I)^{-1}
 \right]
 \longrightarrow 0.
\end{equation*}
In particular,
\(
 \frac1m\tr Q(\rho)\longrightarrow \frac1{\omega(\rho)}.
\)
Taking \(A=\Gamma(\rho)\) yields
\[
 \frac1m\tr[\hat C(\rho)Q(\rho)]
 \longrightarrow
 \frac{\rho}{\omega(\rho)+\alpha_p\rho}
 +
 \frac{1-\rho}{\omega(\rho)+\alpha_q(1-\rho)}
 =
 \df_1(\rho).
\]
Since \(\hat C(\rho)Q(\rho)=I-\zeta Q(\rho)\), this also gives the identity
\(
 \df_1(\rho)=1-\frac{\zeta}{\omega(\rho)}.
\)

\emph{Step 3: correlations between two different values of $\rho$, $\sigma$.}
We next compute terms that involve two values $\rho,\sigma \in [0,1]$. Let
\(
 t(\rho,\sigma)=\frac1m\tr(Q(\rho)Q(\sigma)).
\)
We show in Appendix~\ref{app:newprop} an extension of Proposition~3.2 of \cite{Bach2024RandomProjections} that leads to, for every bounded
deterministic diagonal~\(A\),  
\BEQ
\label{eq:AA}
\frac1{m^2}\tr[AZQ(\sigma)Q(\rho)Z^\top]
   -
 t(\rho,\sigma)
 \frac1m\tr\!\left[
 A
 (I+\omega(\rho)^{-1}\Gamma(\rho))^{-1}
 (I+\omega(\sigma)^{-1}\Gamma(\sigma))^{-1}
 \right]
 \longrightarrow 0.
 \EEQ
From
\(
 Q(\rho)(\hat C(\rho)+\zeta I)Q(\sigma)=Q(\sigma),
\)
we get
\(
 \frac1m\tr Q(\sigma)
 =
 \frac1m\tr(Q(\rho)\hat C(\rho)Q(\sigma))
 +\zeta t(\rho,\sigma).
\)
Using \eq{AA} with \(A=\Gamma(\rho)\), we get
\[
 \frac1{\omega(\sigma)}
 =
 t(\rho,\sigma)(\zeta+a(\rho,\sigma)),
\]
where
\[
 a(\rho,\sigma)
 =
 \lim_{m\to\infty}
 \frac1m\tr\!\left[
 \Gamma(\rho)
 (I+\omega(\rho)^{-1}\Gamma(\rho))^{-1}
 (I+\omega(\sigma)^{-1}\Gamma(\sigma))^{-1}
 \right].
\]
Using
\(
 (I+\omega(\sigma)^{-1}\Gamma(\sigma))^{-1}
 =
 I-
 \Gamma(\sigma)
 (\omega(\sigma)I+\Gamma(\sigma))^{-1},
\)
we obtain
\[
 a(\rho,\sigma)
 =
 \omega(\rho)\df_1(\rho)
 -\omega(\rho)\df_2(\rho,\sigma)
 =
 \omega(\rho)-\zeta-\omega(\rho)\df_2(\rho,\sigma),
\]
where evaluating the two diagonal blocks gives \eq{df2-zeta-hd}. Therefore
we get
\[
 \frac1m\tr\{Q(\rho)Q(\sigma)\}
 \longrightarrow
 \frac1{\omega(\rho)\omega(\sigma)(1-\df_2(\rho,\sigma))}
 =
 \chi_2(\rho,\sigma).
\]
It remains to transfer the trace limits to the mean-dependent bilinear forms.
Define
\[
 \varepsilon_p=n_p^{-1/2}Z^\top e_1,
 \qquad
 \varepsilon_q=n_q^{-1/2}Z^\top e_{n_p+1}.
\]
Then
\[
 \hat\Delta=\Delta+\varepsilon_p-\varepsilon_q,
 \qquad
 \hat\mu(\rho)=\rho\Delta+\rho\varepsilon_p+(1-\rho)\varepsilon_q.
\]
The vectors \(\varepsilon_p,\varepsilon_q\) are independent of the covariance rows
entering \(Q(\rho)\). The right-orthogonal invariance of the resolvent family and the trace limits above gives, for each
fixed \(\rho,\sigma\), almost surely,
\begin{align*}
 \Delta^\top Q(\rho)\Delta
 &\to \frac{s}{\omega(\rho)},  \quad
 \Delta^\top Q(\rho)\varepsilon_p,\;
 \Delta^\top Q(\rho)\varepsilon_q,\;
 \varepsilon_p^\top Q(\rho)\varepsilon_q
 \to 0,  \quad 
 \varepsilon_p^\top Q(\rho)\varepsilon_p
 \to \frac{\alpha_p}{\omega(\rho)},
 \quad
 \varepsilon_q^\top Q(\rho)\varepsilon_q
 \to \frac{\alpha_q}{\omega(\rho)}.
\end{align*}
The same argument with \(Q(\rho)Q(\sigma)\) gives
\begin{align*}
 \Delta^\top Q(\rho)Q(\sigma)\Delta
  \to s\chi_2(\rho,\sigma), \quad 
 \varepsilon_p^\top Q(\rho)Q(\sigma)\varepsilon_p
  \to \alpha_p\chi_2(\rho,\sigma), \quad 
 \varepsilon_q^\top Q(\rho)Q(\sigma)\varepsilon_q
  \to \alpha_q\chi_2(\rho,\sigma),
\end{align*}
with all corresponding cross terms converging to zero. Therefore, with
\(
 R=s+\alpha_p+\alpha_q,
\)
we have
\(
 \hat\Delta^\top Q(\rho)\hat\Delta
 \to
 \frac{R}{\omega(\rho)},
\)
and hence
\(
 \hat\kappa(\rho)
 \to
 1+\frac{\rho(1-\rho)R}{\omega(\rho)}
 =
 \kappa(\rho).
\)
Moreover,
\begin{align*}
 \Delta^\top Q(\rho)\hat\Delta
 \to \frac{s}{\omega(\rho)},\quad
 \hat\mu(\rho)^\top Q(\rho)\hat\Delta
\to
 \frac{\rho(s+\alpha_p)-(1-\rho)\alpha_q}{\omega(\rho)}, \quad
 \hat\Delta^\top Q(\rho)Q(\sigma)\hat\Delta
 \to R\chi_2(\rho,\sigma).
\end{align*}
Substituting these limits into \eq{beta-sm-hd} proves the pointwise
almost-sure versions of Eqs.~\eqref{eq:h-zeta-hd}-\eqref{eq:k-zeta-hd}.

\emph{Step 4: uniform convergence.}
Finally, we upgrade pointwise convergence to uniform convergence. The pointwise
limits above hold on a common probability-one event for all \(\rho,\sigma\) in a
fixed countable dense subset of \([0,1]\). On the same event, the Gaussian
sample-covariance spectral norms \(\|\hat C_p\|\) and \(\|\hat C_q\|\) are eventually
bounded. Since \(\zeta>0\),
\[
 \|Q(\rho)-Q(\sigma)\|
 \leq
 \zeta^{-2}|\rho-\sigma|\cdot \|\hat C_p-\hat C_q\|.
\]
The mean vectors have almost-surely bounded norms, and
\(\hat\kappa(\rho)\geqslant 1\). Hence the random maps
\[
 \rho\mapsto \Delta^\top\hat\beta(\rho),\qquad
 \rho\mapsto \hat\mu(\rho)^\top\hat\beta(\rho),\qquad
 (\rho,\sigma)\mapsto \hat\beta(\rho)^\top\hat\beta(\sigma)
\]
are almost surely eventually equicontinuous. The deterministic limits \(\ell_\Delta,m,k\)
are continuous on their compact domains. A standard grid argument then gives the
stated uniform almost-sure convergences over \([0,1]\) and \([0,1]^2\). 
\end{proof}

For fixed \(\zeta>0\), the quantities defined in Theorem~\ref{thm:spec-ridge} are computed by computing $\omega(\rho)$ at each quadrature node. Once \(\omega(\rho)\)
and \(\omega(\sigma)\) are known, \(\df_1(\rho)\), \(\df_2(\rho,\sigma)\),
\(\kappa(\rho)\), \(\ell_\Delta(\rho)\), \(m(\rho)\), \(k(\rho,\rho)\), and \(k(\rho,\sigma)\)
are explicit rational functions of \(\omega(\rho)\), \(\omega(\sigma)\), \(\zeta\),
and the parameters \(s,\alpha_p,\alpha_q\). Thus the only numerical task in the
regularized spectral deterministic equivalent is to compute these quantities
at the quadrature nodes; no high-dimensional optimization remains.

\paragraph{Recovering the population calculation of \mysec{specpop}.}
The population spectral estimator of \mysec{specpop} is the
zero-aspect-ratio specialization of Theorem~\ref{thm:spec-ridge}. Write
$
  D(\rho)=1+\zeta+s \rho(1-\rho).
$
When \(\alpha_p=\alpha_q=0\),  
\eq{omega-zeta-hd} gives
\(
  \omega(\rho)=1+\zeta,
\)
and the cross-degrees-of-freedom term in \eq{df2-zeta-hd} gives
\(
  \df_2(\rho,\sigma)=0.
\)
Moreover \(R=s\), and therefore
\(
  \kappa(\rho)
  =
  1+\frac{s \rho(1-\rho)}{1+\zeta}
  =
  \frac{D(\rho)}{1+\zeta}.
\)
Substituting these identities into the deterministic equivalents
Eqs.~\eqref{eq:h-zeta-hd}-(\ref{eq:k-zeta-hd}) yields
\[
  \ell_\Delta(\rho)
  =
  \frac{s}{D(\rho)},
  \qquad
  m(\rho)
  =
  \frac{\rho s}{D(\rho)},
  \qquad
  k(\rho,\sigma)
  =
  \frac{s}{D(\rho)D(\sigma)} .
\]
These are exactly the contractions generated by the population coefficient
in \eq{u-pop-ridge}. Indeed, since \(\mu(\rho)=\rho\Delta\) and
\(
  \beta_0(\rho)
  =
  \frac{\Delta}{D(\rho)},
\)
we have
\(
  \Delta^\top\beta^{(0)} (\rho)=\ell_\Delta(\rho),
  \mu(\rho)^\top\beta^{(0)} (\rho)=m(\rho),
  \beta^{(0)} (\rho)^\top\beta^{(0)} (\sigma)=k(\rho,\sigma).
\)
Thus Theorem~\ref{thm:spec-ridge} recovers the affine population functions
of \mysec{specpop}: 
\(
  u^{(0)} (\rho)(x)
  =
  \beta^{(0)} (\rho)^\top(x-\mu(\rho))
  =
  \frac{\Delta^\top x-\rho s}{D(\rho)} .
\)
The same specialization also recovers the population quadratic potentials.

\subsection[Integral and quadrature formulas for criteria J, K, and L]{Integral and quadrature formulas for criteria \(\mathcal J\), \(\mathcal K\), and \(\mathcal L\)}
\label{sec:chi-quadrature}

This subsection records the almost-sure deterministic limits of the population scores
obtained after fitting the empirical spectral estimator. Define
\[
 d(\rho)=2\rho(1-\rho),
 \qquad
 b(\rho)=2(1-\rho)(1+\rho m(\rho)).
\]
In the deterministic limit, write
\[
 v_{\rm spec}(x)=x^\top A_vx+\ell_v^\top x+c_v.
\]
Theorem~\ref{thm:spec-ridge} and dominated convergence give the following almost-sure
limits of the empirical trace, signal-quadratic, linear, and intercept terms:
\begin{align*}
 \tr(A_v)
 &=-\frac12\int_0^1 d(\rho)k(\rho,\rho)\,\dd\rho,
 \\
 \Delta^\top A_v\Delta
 &=-\frac12\int_0^1 d(\rho)\ell_\Delta(\rho)^2\,\dd\rho,
 \\
 \ell_v^\top\Delta
 &=\int_0^1 b(\rho)\ell_\Delta(\rho)\,\dd\rho,
 \\
 c_v
 &=\int_0^1 2(1-\rho)
 \left(-m(\rho)-\frac{\rho}{2}m(\rho)^2\right)\dd\rho.
\end{align*}
Most of the difficulty comes from the log-determinant term in \eq{J-quad}.

\paragraph{Log-determinant term.}  
Let \(\mathcal T\) be the integral operator on \(L_2([0,1])\) with kernel
\(k(\rho,\sigma)\), and let \(\mathcal D\) be multiplication by \(d(\rho)=2\rho(1-\rho)\).
The determinant and inverse-quadratic terms are
\begin{align*}
 \log\det(I-2A_v)
 &=
 \log\det(I+\mathcal D\mathcal T)
 =
 \log\det(I+\mathcal D^{1/2}\mathcal T\mathcal D^{1/2}),
 \\
 \ell_v^\top(I-2A_v)^{-1}\ell_v
 &=
 \left\langle b,(I+\mathcal T\mathcal D)^{-1}\mathcal T b
 \right\rangle_{L_2([0,1])}.
\end{align*}
Here \(\det\) is the Fredholm determinant~\cite{SimonTraceIdeals2005}, with the operator
\(\mathcal D^{1/2}\mathcal T\mathcal D^{1/2}\) that is positive trace class, with trace
\(
 \int_0^1 d(\rho)k(\rho,\rho)\,\dd\rho.
\)
The identities above are the continuum versions of the matrix determinant lemma and
Woodbury identity. Formally,
\[
 I-2A_v=I+B\mathcal D B^\ast,
 \qquad
 (Bf)=\int_0^1\beta(\rho)f(\rho)\,\dd\rho,
 \qquad
 \mathcal T=B^\ast B.
\]
The empirical kernels converge uniformly almost surely to \(k\) by
Theorem~\ref{thm:spec-ridge}; the corresponding positive trace-class operators converge
in trace norm, so the Fredholm determinants and inverse-quadratic terms above are the
almost-sure limits of their empirical analogues.

The scalar below is the population log-normalizer \(\log Z_q(v_{\rm spec})\). The continuum log-normalizer and scores are
\begin{align}
 \Lambda_{\rm spec}
 &=
 c_v-\frac12\log\det(I+\mathcal D\mathcal T)
 +\frac12
 \left\langle b,(I+\mathcal T\mathcal D)^{-1}\mathcal T b\right\rangle_{L_2([0,1])},
 \label{eq:Lambda-spec-cont}\\
 \J_{\rm spec}
 &=
 \tr(A_v)+\Delta^\top A_v\Delta+\ell_v^\top\Delta
 +\frac12\log\det(I+\mathcal D\mathcal T)
 -\frac12
 \left\langle b,(I+\mathcal T\mathcal D)^{-1}\mathcal T b\right\rangle_{L_2([0,1])},
 \label{eq:J-spec-cont}\\
 \K_{\rm spec}
 &=
 \tr(A_v)+\Delta^\top A_v\Delta+\ell_v^\top\Delta+c_v+1-e^{\Lambda_{\rm spec}}.
 \label{eq:K-spec-cont}
\end{align}
Equivalently, for the fitted empirical spectral potential
\(\hat v_{{\rm spec}}\),
\[
 \log Z_q(\hat v_{{\rm spec}})
 \xrightarrow{\mathrm{a.s.}}
 \Lambda_{\rm spec},
 \qquad
 \J(\hat v_{{\rm spec}})
 \xrightarrow{\mathrm{a.s.}}
 \J_{\rm spec},
 \qquad
 \K(\hat v_{{\rm spec}})
 \xrightarrow{\mathrm{a.s.}}
 \K_{\rm spec}.
\]
 
\paragraph{Criterion \(\mathcal L\).}
For the spectral pair \((v,w)\), write the deterministic companion potential as
\[
 w(y)=y^\top A_wy+\ell_w^\top y+c_w.
\]
Using the same functions \(m(\rho)\), \(k(\rho,\rho)\), and \(k(\rho,\sigma)\), its
continuum coefficients satisfy
\begin{align*}
 A_w
 &=-\int_0^1(1-\rho)^2\beta(\rho)\beta(\rho)^\top\,\dd\rho,\\
 \ell_w
 &=\int_0^1 2(1-\rho)\{-1+(1-\rho)m(\rho)\}\beta(\rho)\,\dd\rho,\\
 c_w
 &=\int_0^1 2(1-\rho)
 \Big(m(\rho)-\frac{1-\rho}{2}m(\rho)^2\Big) \dd\rho.
\end{align*}
Therefore
\(
 \tr(A_w)=-\int_0^1(1-\rho)^2k(\rho,\rho)\,\dd\rho, 
\) and
Proposition~\ref{prop:quadratic-L-scores} gives the continuum two-potential score
\begin{align}
 \mathcal L_{\rm spec}
 &=
 \tr(A_v)+\Delta^\top A_v\Delta+\ell_v^\top\Delta+c_v+\tr(A_w)+c_w.
 \label{eq:L-spec-cont}
 \end{align}
Moreover, for the fitted empirical spectral pair
\((\hat v_{{\rm spec}},\hat w_{{\rm spec}})\),
\(
 \mathcal L(\hat v_{{\rm spec}},\hat w_{{\rm spec}})
 \xrightarrow{\mathrm{a.s.}}
 \mathcal L_{\rm spec}.
\)

\paragraph{Quadrature.}
For any fixed deterministic quadrature rule, the corresponding finite-quadrature
versions of all displayed quantities converge almost surely by the same uniform
limits in Theorem~\ref{thm:spec-ridge}. For the explicit quadrature formulas used in
simulations, see Appendix~\ref{sec:quadas}.

\section{Weak signal with proportionally small aspect ratios}
\label{sec:weak-signal-proportional-small-alpha}
\label{sec:weak}

This section gives two local comparisons between affine variational fitting and continuum spectral fitting.  The first keeps the signal strength
\(s=\|\Delta\|^2\) fixed and sends \((\alpha_p,\alpha_q)\) to zero.  This is a fixed-signal expansion and keeps only terms that are linear in the two aspect ratios.  The second is the weak-signal scaling
\(s\downarrow0\), \(\alpha_p=s\beta_p\), \(\alpha_q=s\beta_q\), where quadratic terms in \((\alpha_p,\alpha_q)\) also contribute at order \(s^2\).  Note that these two limits are not interchangeable.

\subsection{Fixed signal and small aspect ratios}
\label{subsec:fixed-signal-small-alpha}

In this subsection \(s>0\) is fixed and
\[
    \alpha_p\downarrow0,\qquad \alpha_q\downarrow0.
\]
For a criterion \(\mathcal A \in \{ \mathcal{J}, \mathcal{K}, \mathcal{L}\} \), write the excess risk expansion as
\begin{equation}
\label{eq:fixed-signal-alpha-expansion-template-all-criteria}
    \mathcal R_{\mathcal A,\mathcal{T}}(s,\alpha_p,\alpha_q)
    =\mathcal R_{\mathcal A,\mathcal{T}}^{(0)}(s)
    +C_{\mathcal A,p,\mathcal{T}}(s)\alpha_p
    +C_{\mathcal A,q,\mathcal{T}}(s)\alpha_q
    +o(\alpha_p+\alpha_q),
\end{equation}
where \(\mathcal{T}\in\{\mathrm{var},\mathrm{spec}\}\) and
$\mathcal R_{\mathcal A,\mathcal{T}}^{(0)}$ is the population criterion from \mysec{pop-bench}.  Criterion \(\mathcal L\) is the two-potential spectral criterion, and is therefore only used for the spectral estimator.

We start with the variational estimator, for which we use the classical M-estimation proof~\cite{van2000asymptotic}  (we could also have used expansions of formulas from Theorem~\ref{thm:var-ridge}), with a key dependence in $e^s$.
\begin{proposition}[Fixed-signal small-aspect expansion, variational estimator]
\label{prop:fixed-signal-var-JK-first-order}
For the unregularized affine variational estimator,
\begin{align*}
    \mathcal R_{\mathcal J,\mathrm{var}}(s,\alpha_p,\alpha_q)
    &=\frac12\alpha_p+\frac12 e^s\alpha_q+o(\alpha_p+\alpha_q),\\
    \mathcal R_{\mathcal K,\mathrm{var}}(s,\alpha_p,\alpha_q)
    &=\frac12\alpha_p+\frac12 e^s\alpha_q+o(\alpha_p+\alpha_q). \notag
\end{align*}
Thus \(\mathcal R_{\mathcal J,\mathrm{var}}^{(0)}=\mathcal R_{\mathcal K,\mathrm{var}}^{(0)}=0\),
and \(
    C_{\mathcal J,p,\mathrm{var}}=C_{\mathcal K,p,\mathrm{var}}=\frac12,
    C_{\mathcal J,q,\mathrm{var}}=C_{\mathcal K,q,\mathrm{var}}=\frac12 e^s .
\)
\end{proposition}

\begin{proof}
Let \(\hat\theta\) be the unregularized variational affine coefficient.  At \(\alpha_p=\alpha_q=0\), \(\hat\theta=\Delta\).  By Proposition~\ref{prop:affine-scores},
we have \(
    \mathcal R_{\mathcal J,\mathrm{var}}
    =\frac12\|\hat\theta-\Delta\|^2 .
\)
The score equation gives the local M-estimation expansion, with $o_\P$ and $O_\P$ notations in probability~\cite{van2000asymptotic} 
\[
    \hat\theta-\Delta
    = (\hat\mu_p-\Delta)
      -\frac1{n_q}\sum_{j=1}^{n_q}
      \exp(\Delta^\top y_j-s/2)(y_j-\Delta)
      +o_\P(n_p^{-1/2}+n_q^{-1/2}).
\]
The two samples are independent, \(\E[\|\hat\mu_p-\Delta\|^2]=m/n_p=\alpha_p\), and
\(
    \E_q\left[\exp(2\Delta^\top y-s)\|y-\Delta\|^2\right]
    = e^s(m+s).
\)
After division by \(n_q\), the term \(s/n_q\) is lower order while \(m/n_q\to\alpha_q\).  This proves the \(\mathcal J\) expansion.  For \(\mathcal K\), use the normalizer identity
\(
    \mathcal R_{\mathcal K}(v)=\mathcal R_{\mathcal J}(v)+Z_q(v)-\log Z_q(v)-1.
\)
For the empirically normalized variational representative, \(\log Z_q(\hat v)=O_\P(\alpha_p+\alpha_q)\), hence
\(Z_q(\hat v)-\log Z_q(\hat v)-1=O_\P((\alpha_p+\alpha_q)^2)\).  Therefore \(\mathcal J\) and \(\mathcal K\) have the same fixed-signal first-order expansion.
\end{proof}

For the spectral estimator, we can expand the results of Theorem~\ref{thm:spec-ridge} for $\alpha_p, \alpha_q$ close to zero, and derive closed form formulas for each $\rho$, that can then be integrated over $[0,1]$. Since the formulas are overly complicated, we only plot them in Figure~\ref{fig:fixed-signal-small-alpha-JKL-terms}.

Figure~\ref{fig:fixed-signal-small-alpha-JKL-terms} plots the three terms in \eq{fixed-signal-alpha-expansion-template-all-criteria}.  The variational estimator is correctly specified, so its population terms \(\mathcal R^{(0)}_{\mathcal J,\mathrm{var}}\) and \(\mathcal R^{(0)}_{\mathcal K,\mathrm{var}}\) are zero for \(\mathcal J\) and \(\mathcal K\) (while it is not for the spectral estimator).  Its \(\alpha_q\)-coefficient is \(e^s/2\), which is shown in the right panel and reflects the variance of empirical exponential normalization.  The spectral estimator has a nonzero population approximation term (left panel) but much weaker dependence on \(\alpha_q\) at larger fixed signal (middle and right panels).

\begin{figure}[t]
    \centering
    \includegraphics[width=.95\textwidth]{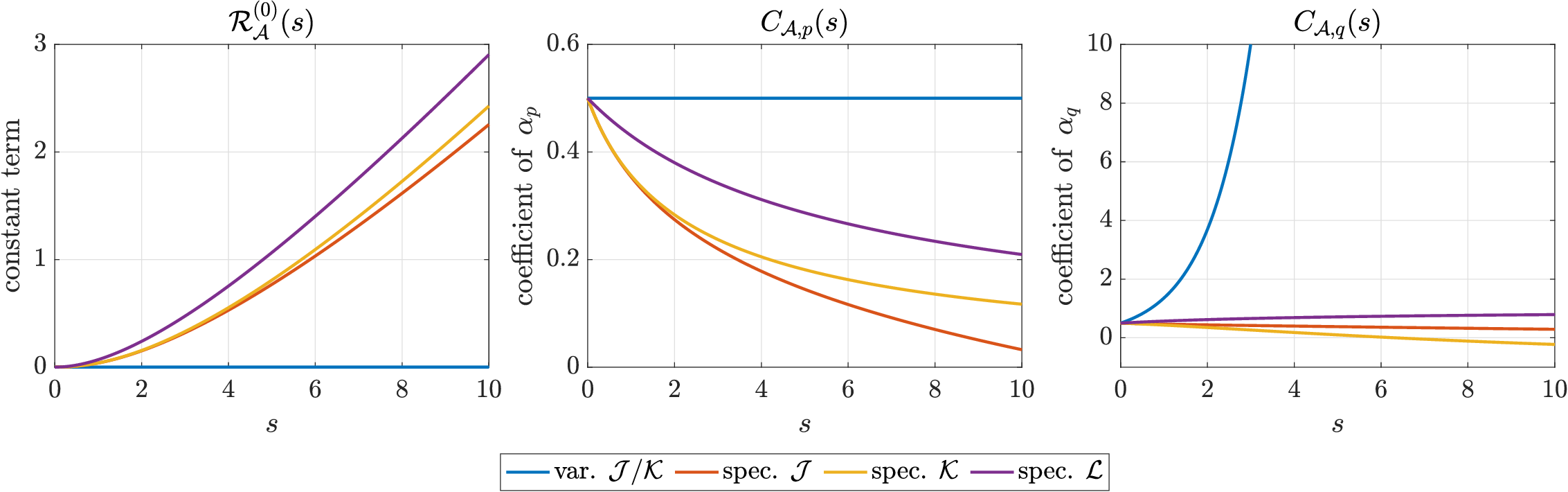} 
    % from make_section5_figures_three_criteria.m
    \caption{Fixed-signal first-order expansion of the unregularized risks under criteria \(\mathcal J\), \(\mathcal K\), and \(\mathcal L\) (denoted~\(\mathcal A\)).  Left: constant term \(\mathcal R_{\mathcal{A}}^{(0)}(s)\), middle: coefficient \(C_{\mathcal{A},p}(s)\) of \(\alpha_p\), right:  coefficient \(C_{\mathcal{A},q}(s)\) of \(\alpha_q\).  Criterion \(\mathcal L\) is shown only for the spectral estimator. }
    \label{fig:fixed-signal-small-alpha-JKL-terms}
\end{figure}

\subsection{Weak signal with proportionally small aspect ratios}
\label{subsec:weak-signal-proportional-small-alpha}

We now consider the joint weak-signal/small-aspect regime
\[
    s\downarrow0,
    \qquad
    \alpha_p=s\beta_p,
    \qquad
    \alpha_q=s\beta_q,
\]
where \(\beta_p,\beta_q\geqslant 0\) are fixed.  Write \(\beta=\beta_p+\beta_q\).  For every fixed finite pair \((\beta_p,\beta_q)\), the zero-ridge feasibility constraints are inactive for all sufficiently small \(s\): the spectral endpoint conditions \(\alpha_p<1\), \(\alpha_q<1\) hold, and the variational convex-hull radius diverges as \(\alpha_q\downarrow0\).  We first start with the unregularized estimators.

\begin{proposition}[Weak-signal expansion of the unregularized risks]
\label{prop:weak-signal-unregularized}
Under \(\alpha_p=s\beta_p\), \(\alpha_q=s\beta_q\), the unregularized affine variational estimator satisfies
\begin{align}
\label{eq:weak-var-unreg-JK}
    \mathcal R_{\mathcal J,\mathrm{var}}
    &=\frac{\beta}{2}s+\frac{\beta_q(1+3\beta)}{2}s^2+o(s^2),\\
    \mathcal R_{\mathcal K,\mathrm{var}}
    &=\frac{\beta}{2}s+
      \Big(\frac{\beta_q(1+3\beta)}{2}+\frac{\beta_q^2}{2}\Big)s^2+o(s^2). \notag
\end{align}
The unregularized spectral estimator satisfies
\begin{align}
\label{eq:weak-spec-unreg-JKL}
    \mathcal R_{\mathcal J,\mathrm{spec}}
    &=\frac{\beta}{2}s+
     \Big(\frac1{36}-\frac{2\beta_p}{9}-\frac{\beta_q}{18}
      +\frac{\beta_p^2}{12}+\frac{\beta_p\beta_q}{3}+\frac{\beta_q^2}{4}\Big) s^2+o(s^2),\\
    \mathcal R_{\mathcal K,\mathrm{spec}}
    &=\frac{\beta}{2}s+
      \Big(\frac1{36}-\frac{2\beta_p}{9}-\frac{\beta_q}{18}
      +\frac{\beta_p^2}{12}+\frac{\beta_p\beta_q}{3}+\frac{3\beta_q^2}{4}\Big) s^2+o(s^2), \notag\\
    \mathcal R_{\mathcal L,\mathrm{spec}}
    &=\frac{\beta}{2}s+
     \Big(\frac1{12}-\frac{\beta_p}{12}+\frac{\beta_q}{12}
      +\frac{\beta_p^2}{6}+\frac{\beta_p\beta_q}{2}+\frac{5\beta_q^2}{6}\Big) s^2+o(s^2). \notag
\end{align}
Consequently,
\begin{equation}
\label{eq:Qunreg-weak}
    \mathcal R_{\mathcal J,\mathrm{spec}}-\mathcal R_{\mathcal J,\mathrm{var}}
    =\mathcal R_{\mathcal K,\mathrm{spec}}-\mathcal R_{\mathcal K,\mathrm{var}}
    =\frac{s^2}{36}Q_{\rm unreg}(\beta_p,\beta_q)+o(s^2),
\end{equation}
where $
    Q_{\rm unreg}(\beta_p,\beta_q)
    =1-8\beta_p-20\beta_q+3\beta_p^2-42\beta_p\beta_q-45\beta_q^2.
$
Thus the unregularized spectral estimator is favored to second order for criteria \(\mathcal J\) and \(\mathcal K\) when \(Q_{\rm unreg}<0\), and the affine variational estimator is favored when \(Q_{\rm unreg}>0\).
\end{proposition}

\begin{proof}
This is a Taylor expansion of the zero-ridge deterministic equivalents.  For the variational estimator, the hard-limit equations of Theorem~\ref{thm:var-ridge} yield, after tedious computations,
\[
    \mathcal R_{\mathcal J,\mathrm{var}}
    =\frac{\alpha_p}{2}+\frac{e^s\alpha_q}{2}
      +\frac32\alpha_p\alpha_q+\frac32\alpha_q^2+o(s^2),
\]
\[
    \mathcal R_{\mathcal K,\mathrm{var}}
    =\frac{\alpha_p}{2}+\frac{e^s\alpha_q}{2}
      +\frac32\alpha_p\alpha_q+2\alpha_q^2+o(s^2),
\]
valid for \(\alpha_p,\alpha_q=O(s)\).  The difference is the normalizer penalty: if \(z=\log Z_q(\hat v_{\rm var})\), then \(z=\alpha_q+o(s)\) and \(e^z-z-1=\alpha_q^2/2+o(s^2)\).  Substituting \(\alpha_p=s\beta_p\), \(\alpha_q=s\beta_q\), and \(e^s=1+s+O(s^2)\) gives \eq{weak-var-unreg-JK}.

For the spectral estimator, we can expand the zero-ridge specialization of Theorem~\ref{thm:spec-ridge} and the score formulas in Eqs.~\eqref{eq:J-spec-cont}-\eqref{eq:L-spec-cont}.   Integrating the resulting polynomials in \(\rho\) gives \eq{weak-spec-unreg-JKL}.  Subtracting the variational expansion gives \eq{Qunreg-weak}.
\end{proof}

We now consider the regularized estimators, with their optimized regularization parameter.
\begin{proposition}[Weak-signal expansion with optimized ridge]
\label{prop:weak-signal-optimized-ridge}
Let the variational ridge be parameterized by \(t=(1+\tau)^{-1}\in[0,1]\), and the spectral ridge by \(\gamma=(1+\zeta)^{-1}\in[0,1]\).  For fixed \(t\), we have
\begin{align*}
    \mathcal R_{\mathcal J,\mathrm{var}}
    &=\frac{s}{2}((1-t)^2+\beta t^2)
      +\beta_qt^3\Big(\frac32(1+\beta)t-1\Big) s^2+o(s^2),\\
    \mathcal R_{\mathcal K,\mathrm{var}}
    &=\mathcal R_{\mathcal J,\mathrm{var}}+\frac{\beta_q^2t^2}{2}s^2+o(s^2). \notag
\end{align*}
The leading shrinkage problem gives
\(
    t^\ast_{\mathcal J,\mathrm{var}}=t^\ast_{\mathcal K,\mathrm{var}}
    =\frac1{1+\beta}+O(s),  
    \tau^\ast_{\mathcal J,\mathrm{var}}=\tau^\ast_{\mathcal K,\mathrm{var}}
    =\beta+O(s),
\)
and therefore
\begin{align}
\label{eq:weak-var-opt-ridge}
    \mathcal R^{\rm opt}_{\mathcal J,\mathrm{var}}
    &=\frac{\beta}{2(1+\beta)}s+\frac{\beta_q}{2(1+\beta)^3}s^2+o(s^2),\\
    \mathcal R^{\rm opt}_{\mathcal K,\mathrm{var}}
    &=\frac{\beta}{2(1+\beta)}s+
     \Big( \frac{\beta_q}{2(1+\beta)^3}
      +\frac{\beta_q^2}{2(1+\beta)^2}\Big) s^2+o(s^2). \notag
\end{align}
For the spectral estimator, for each criterion \(\mathcal A\in\{\mathcal J,\mathcal K,\mathcal L\}\),
\[
    \gamma^\ast_{\mathcal A,\mathrm{spec}}=\frac1{1+\beta}+O(s),
    \qquad
    \zeta^\ast_{\mathcal A,\mathrm{spec}}=\beta+O(s).
\]
Substituting \(\gamma=(1+\beta)^{-1}\) gives
\begin{align}
\label{eq:weak-spec-opt-ridge}
    \mathcal R^{\rm opt}_{\mathcal J,\mathrm{spec}}
    &=\frac{\beta}{2(1+\beta)}s+\frac{3\beta_p+9\beta_q+1}{36(1+\beta)^3}s^2+o(s^2),\\
    \mathcal R^{\rm opt}_{\mathcal K,\mathrm{spec}}
    &=\frac{\beta}{2(1+\beta)}s+
      \left\{\frac{3\beta_p+9\beta_q+1}{36(1+\beta)^3}
      +\frac{(2\beta_q-\beta_p)^2}{18(1+\beta)^2}\right\}s^2+o(s^2), \notag\\
    \mathcal R^{\rm opt}_{\mathcal L,\mathrm{spec}}
    &=\frac{\beta}{2(1+\beta)}s +
      \frac{\beta_p^3-\beta_p^2\beta_q+\beta_p^2+
      \beta_p\beta_q^2-2\beta_p\beta_q+2\beta_p+
      3\beta_q^3+3\beta_q^2+4\beta_q+1}
      {12(1+\beta)^3}s^2+o(s^2). \notag
\end{align}
For criterion \(\mathcal J\),
\begin{equation*}
    \mathcal R^{\rm opt}_{\mathcal J,\mathrm{spec}}
    -\mathcal R^{\rm opt}_{\mathcal J,\mathrm{var}}
    =\frac{s^2}{36(1+\beta)^3}Q_{\mathcal J,{\rm opt}}(\beta_p,\beta_q)+o(s^2),
\end{equation*}
where
$
    Q_{\mathcal J,{\rm opt}}(\beta_p,\beta_q)=1+3\beta_p-9\beta_q.
$
For criterion \(\mathcal K\),
\begin{equation*}
    \mathcal R^{\rm opt}_{\mathcal K,\mathrm{spec}}
    -\mathcal R^{\rm opt}_{\mathcal K,\mathrm{var}}
    =\frac{s^2}{36(1+\beta)^3}Q_{\mathcal K,{\rm opt}}(\beta_p,\beta_q)+o(s^2),
\end{equation*}
where
$
    Q_{\mathcal K,{\rm opt}}(\beta_p,\beta_q)
    =1+3\beta_p-9\beta_q
      +2(1+\beta_p+\beta_q)(\beta_p^2-4\beta_p\beta_q-5\beta_q^2).
$
Optimized spectral fitting is favored under criterion \(\mathcal A\in\{\mathcal J,\mathcal K\}\) when the corresponding \(Q_{\mathcal A,{\rm opt}}\) is negative.
\end{proposition}

\begin{proof}
For the variational estimator, substitute \(t=(1+\tau)^{-1}\) in the regularized fixed-signal expansion and set \(\alpha_p=s\beta_p\), \(\alpha_q=s\beta_q\).  The order-\(s\) term is
\(
    \frac12\{(1-t)^2+\beta t^2\}s,
\)
whose unique minimizer is \(t=(1+\beta)^{-1}\).  Since this minimizer is interior for finite \(\beta\), the order-\(s\) optimizer is enough to evaluate the risk up to order \(s^2\), giving \eq{weak-var-opt-ridge}.  For the spectral estimator, expand Theorem~\ref{thm:spec-ridge} with \(\gamma=(1+\zeta)^{-1}\), \(\alpha_p=s\beta_p\), and \(\alpha_q=s\beta_q\).  The leading term is the same scalar shrinkage risk with \(\gamma\) in place of \(t\), so \(\gamma^\ast=(1+\beta)^{-1}+O(s)\).  Substitution into the order-\(s^2\) coefficients from the spectral score formulas gives \eq{weak-spec-opt-ridge}.  The formulas for \(Q_{\mathcal J,{\rm opt}}\) and \(Q_{\mathcal K,{\rm opt}}\) follow by subtraction.
\end{proof}

Figure~\ref{fig:weak-signal-separating-curves} shows the second-order separating curves defined by $Q_{\mathcal K,{\rm opt}}(\beta_p,\beta_q)=0$, $Q_{\mathcal J,{\rm opt}}(\beta_p,\beta_q)$ and $Q_{\rm unreg}(\beta_p,\beta_q)=0$.  The unregularized curve is common to criteria \(\mathcal J\) and \(\mathcal K\), while the optimized-ridge curves differ because \(\mathcal K\) includes the population normalizer penalty.
We see the advantage of the spectral estimator when $\beta_q$ gets larger (fewer observations), in particular for the criterion $\mathcal{K}$ (which characterizes the proper model normalization).  While we consider here small $s$, $\alpha_p$, $\alpha_q$, with explicit formulas, we consider in \mysec{comparison} plots of the deterministic equivalents for larger values and similar conclusions.

\begin{figure}[t]
    \centering
    \includegraphics[width=.4\textwidth]{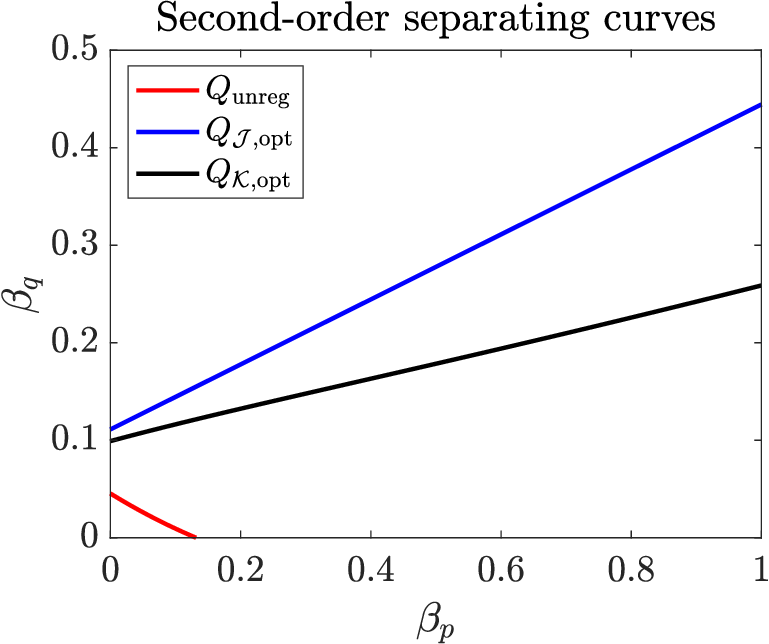}
    % from make_section5_figures_three_criteria.m
    \caption{Second-order separating curves in the weak-signal proportional regime \(\alpha_p=s\beta_p\), \(\alpha_q=s\beta_q\).  The curves are the zero-level sets of \(Q_{\rm unreg}\), \(Q_{\mathcal J,{\rm opt}}\), and \(Q_{\mathcal K,{\rm opt}}\).  Spectral fitting is favored on the side where the corresponding \(Q\)-function is negative (above each curve); variational fitting is favored where it is positive (below each curve). }
    \label{fig:weak-signal-separating-curves}
\end{figure}

\section{Comparisons between estimators }
\label{sec:comparison}

The comparison depends on whether one asks about feasibility or population risk.  The regimes below are direct consequences of Proposition~\ref{prop:var-dual}, Proposition~\ref{prop:chi-feasible}, Theorems~\ref{thm:var-ridge} and~\ref{thm:spec-ridge}, and \mysec{zero-ridge-var}.

\paragraph{Feasibility-first unregularized regimes.}
We consider here the two unregularized (variational and spectral) estimators. Away from
the boundary cases
\(
    \alpha_q=\frac12,  \alpha_p=1,  \alpha_q=1,
\)
the zero-ridge variational estimator is finite with high probability if
and only if
\[
 \qquad
    \alpha_q<\frac12
    \quad\text{and}\quad
    s+\alpha_p<R^2_{\rm hull}(\alpha_q),
\]
whereas the unregularized continuum spectral full potential is finite
with high probability if and only if
\[
   \qquad
    \alpha_p<1
    \quad\text{and}\quad
    \alpha_q<1 .
\]
This leads to the phase diagram in \myfig{regimes}.  

\begin{figure}
\centering
\definecolor{bothfail}{HTML}{D9D9D9}
\definecolor{speconly}{HTML}{8FB9E8}
\definecolor{bothfeas}{HTML}{88C999}
\definecolor{varonly}{HTML}{F3C178}
\definecolor{boundaryred}{HTML}{9D1B1E}

\begin{tikzpicture}
\begin{axis}[
    width=12cm,
    height=12cm,
    xmin=0, xmax=2,
    ymin=0, ymax=2,
    axis lines=left,
    xlabel={$\alpha_p=m/n_p$},
    ylabel={$\alpha_q=m/n_q$},
    xtick={0,0.5,1,1.5,2},
    ytick={0,0.5,1,1.5,2},
    grid=major,
    grid style={line width=.2pt, draw=black!18},
    axis on top,
    clip=false,
    legend style={
        at={(1.03,0.98)}, anchor=north west,
        draw=black!35, fill=white, font=\small,
        cells={anchor=west}
    },
    legend image code/.code={\draw[#1] (0cm,-0.08cm) rectangle (0.35cm,0.08cm);},
]

% Base region: both unregularized full-potential constructions fail.
\addplot[draw=none, fill=bothfail] coordinates {(0,0) (2,0) (2,2) (0,2)} \closedcycle;
\addlegendentry{both fail}

% Spectral endpoint-feasible rectangle: alpha_p<1 and alpha_q<1.
% The variational-feasible subset is overwritten below.
\addplot[draw=none, fill=speconly] coordinates {(0,0) (1,0) (1,1) (0,1)} \closedcycle;
\addlegendentry{spectral only}

% Variational-only: alpha_p>=1, alpha_q<1/2, and s+alpha_p<R_hull^2(alpha_q).
\addplot[draw=none, fill=varonly] coordinates {
      (1,0)
      (2,0)
      (2,0.08466)
      (2.00000,0.08466)
      (1.97239,0.08599)
      (1.94523,0.08733)
      (1.91849,0.08866)
      (1.89216,0.09000)
      (1.86623,0.09133)
      (1.84069,0.09267)
      (1.81553,0.09401)
      (1.79073,0.09534)
      (1.76630,0.09668)
      (1.74221,0.09801)
      (1.71845,0.09935)
      (1.69503,0.10068)
      (1.67192,0.10202)
      (1.64912,0.10335)
      (1.62663,0.10469)
      (1.60443,0.10602)
      (1.58253,0.10736)
      (1.56090,0.10869)
      (1.53954,0.11003)
      (1.51846,0.11136)
      (1.49763,0.11270)
      (1.47706,0.11403)
      (1.45673,0.11537)
      (1.43665,0.11670)
      (1.41681,0.11804)
      (1.39720,0.11938)
      (1.37781,0.12071)
      (1.35865,0.12205)
      (1.33970,0.12338)
      (1.32097,0.12472)
      (1.30244,0.12605)
      (1.28412,0.12739)
      (1.26599,0.12872)
      (1.24806,0.13006)
      (1.23032,0.13139)
      (1.21277,0.13273)
      (1.19540,0.13406)
      (1.17821,0.13540)
      (1.16119,0.13673)
      (1.14435,0.13807)
      (1.12768,0.13940)
      (1.11117,0.14074)
      (1.09482,0.14208)
      (1.07864,0.14341)
      (1.06261,0.14475)
      (1.04673,0.14608)
      (1.03101,0.14742)
      (1.01543,0.14875)
      (1.00000,0.15009)
      (1,0.15009)
} \closedcycle;
\addlegendentry{variational only}

% Both feasible: spectral endpoint feasible and variational hull feasible.
\addplot[draw=none, fill=bothfeas] coordinates {
      (0,0)
      (1,0)
      (1,0.15009)
      (1.00000,0.15009)
      (0.98259,0.15161)
      (0.96537,0.15313)
      (0.94832,0.15465)
      (0.93145,0.15617)
      (0.91475,0.15769)
      (0.89821,0.15921)
      (0.88185,0.16074)
      (0.86564,0.16226)
      (0.84959,0.16378)
      (0.83370,0.16530)
      (0.81795,0.16682)
      (0.80236,0.16834)
      (0.78692,0.16986)
      (0.77162,0.17138)
      (0.75646,0.17291)
      (0.74144,0.17443)
      (0.72656,0.17595)
      (0.71182,0.17747)
      (0.69720,0.17899)
      (0.68272,0.18051)
      (0.66836,0.18203)
      (0.65413,0.18355)
      (0.64002,0.18508)
      (0.62604,0.18660)
      (0.61217,0.18812)
      (0.59842,0.18964)
      (0.58479,0.19116)
      (0.57127,0.19268)
      (0.55786,0.19420)
      (0.54456,0.19573)
      (0.53137,0.19725)
      (0.51829,0.19877)
      (0.50532,0.20029)
      (0.49244,0.20181)
      (0.47967,0.20333)
      (0.46700,0.20485)
      (0.45443,0.20637)
      (0.44196,0.20790)
      (0.42958,0.20942)
      (0.41730,0.21094)
      (0.40511,0.21246)
      (0.39301,0.21398)
      (0.38100,0.21550)
      (0.36909,0.21702)
      (0.35726,0.21854)
      (0.34551,0.22007)
      (0.33386,0.22159)
      (0.32229,0.22311)
      (0.31080,0.22463)
      (0.29939,0.22615)
      (0.28807,0.22767)
      (0.27682,0.22919)
      (0.26565,0.23071)
      (0.25457,0.23224)
      (0.24355,0.23376)
      (0.23262,0.23528)
      (0.22176,0.23680)
      (0.21097,0.23832)
      (0.20026,0.23984)
      (0.18962,0.24136)
      (0.17905,0.24288)
      (0.16855,0.24441)
      (0.15811,0.24593)
      (0.14775,0.24745)
      (0.13746,0.24897)
      (0.12723,0.25049)
      (0.11707,0.25201)
      (0.10697,0.25353)
      (0.09694,0.25506)
      (0.08697,0.25658)
      (0.07707,0.25810)
      (0.06722,0.25962)
      (0.05744,0.26114)
      (0.04772,0.26266)
      (0.03806,0.26418)
      (0.02846,0.26570)
      (0.01891,0.26723)
      (0.00943,0.26875)
      (-0.00000,0.27027)
      (0,0.27027)
} \closedcycle;
\addlegendentry{both feasible}

% Endpoint feasibility boundaries for the spectral continuum full potential.
\addplot[black, very thick] coordinates {(1,0) (1,1)};
\node[anchor=west, font=\small] at (axis cs:1.02,0.92) {$\alpha_p=1$};
\addplot[black, very thick] coordinates {(0,1) (1,1)};
\node[anchor=south west, font=\small] at (axis cs:0.04,1.02) {$\alpha_q=1$};

% Variational hard cutoff alpha_q=1/2.
\addplot[black, dashed, thick] coordinates {(0,0.5) (2,0.5)};
\node[anchor=south east, font=\small] at (axis cs:1.96,0.515) {$\alpha_q=1/2$};

% Variational convex-hull boundary: alpha_p=R_hull^2(alpha_q)-s.
\addplot[boundaryred, very thick] coordinates {
      (2.00000,0.08466)
      (1.96780,0.08622)
      (1.93619,0.08778)
      (1.90516,0.08934)
      (1.87468,0.09090)
      (1.84474,0.09246)
      (1.81531,0.09402)
      (1.78639,0.09558)
      (1.75795,0.09714)
      (1.72998,0.09870)
      (1.70246,0.10026)
      (1.67538,0.10182)
      (1.64873,0.10338)
      (1.62249,0.10493)
      (1.59666,0.10649)
      (1.57121,0.10805)
      (1.54614,0.10961)
      (1.52144,0.11117)
      (1.49709,0.11273)
      (1.47309,0.11429)
      (1.44943,0.11585)
      (1.42610,0.11741)
      (1.40309,0.11897)
      (1.38039,0.12053)
      (1.35799,0.12209)
      (1.33589,0.12365)
      (1.31407,0.12521)
      (1.29254,0.12677)
      (1.27128,0.12833)
      (1.25029,0.12989)
      (1.22956,0.13145)
      (1.20908,0.13301)
      (1.18885,0.13457)
      (1.16887,0.13613)
      (1.14912,0.13769)
      (1.12960,0.13925)
      (1.11032,0.14081)
      (1.09125,0.14237)
      (1.07240,0.14393)
      (1.05376,0.14549)
      (1.03533,0.14705)
      (1.01710,0.14861)
      (0.99907,0.15017)
      (0.98124,0.15173)
      (0.96359,0.15329)
      (0.94614,0.15485)
      (0.92886,0.15641)
      (0.91177,0.15797)
      (0.89485,0.15953)
      (0.87810,0.16109)
      (0.86152,0.16265)
      (0.84511,0.16421)
      (0.82886,0.16576)
      (0.81277,0.16732)
      (0.79684,0.16888)
      (0.78106,0.17044)
      (0.76543,0.17200)
      (0.74995,0.17356)
      (0.73462,0.17512)
      (0.71942,0.17668)
      (0.70437,0.17824)
      (0.68946,0.17980)
      (0.67468,0.18136)
      (0.66003,0.18292)
      (0.64552,0.18448)
      (0.63113,0.18604)
      (0.61687,0.18760)
      (0.60274,0.18916)
      (0.58872,0.19072)
      (0.57483,0.19228)
      (0.56105,0.19384)
      (0.54739,0.19540)
      (0.53385,0.19696)
      (0.52042,0.19852)
      (0.50710,0.20008)
      (0.49389,0.20164)
      (0.48078,0.20320)
      (0.46778,0.20476)
      (0.45489,0.20632)
      (0.44210,0.20788)
      (0.42941,0.20944)
      (0.41682,0.21100)
      (0.40432,0.21256)
      (0.39193,0.21412)
      (0.37963,0.21568)
      (0.36742,0.21724)
      (0.35531,0.21880)
      (0.34328,0.22036)
      (0.33135,0.22192)
      (0.31950,0.22348)
      (0.30775,0.22504)
      (0.29608,0.22659)
      (0.28449,0.22815)
      (0.27299,0.22971)
      (0.26157,0.23127)
      (0.25023,0.23283)
      (0.23897,0.23439)
      (0.22779,0.23595)
      (0.21669,0.23751)
      (0.20567,0.23907)
      (0.19472,0.24063)
      (0.18385,0.24219)
      (0.17305,0.24375)
      (0.16233,0.24531)
      (0.15167,0.24687)
      (0.14109,0.24843)
      (0.13058,0.24999)
      (0.12014,0.25155)
      (0.10977,0.25311)
      (0.09947,0.25467)
      (0.08923,0.25623)
      (0.07906,0.25779)
      (0.06896,0.25935)
      (0.05892,0.26091)
      (0.04894,0.26247)
      (0.03903,0.26403)
      (0.02918,0.26559)
      (0.01939,0.26715)
      (0.00967,0.26871)
      (-0.00000,0.27027)
};
\node[anchor=south west, rotate=-7, font=\normalsize, text=boundaryred]
  at (axis cs:0.17,0.245) {$s+\alpha_p=R_{\rm hull}^2(\alpha_q)$};

% Region labels.
\node[font=\normalsize, align=center] at (axis cs:0.48,0.095) {both\\feasible};
\node[font=\normalsize, align=center] at (axis cs:0.45,0.67) {spectral\\only};
\node[font=\normalsize, align=center] at (axis cs:1.52,0.055) {variational only};
\node[font=\normalsize, align=center] at (axis cs:1.50,1.45) {both\\fail};

\end{axis}
\end{tikzpicture}

\vspace*{-.25cm}

\caption{Regimes of feasibility for unregularized estimation ($s=1$). \label{fig:regimes}}
\end{figure}
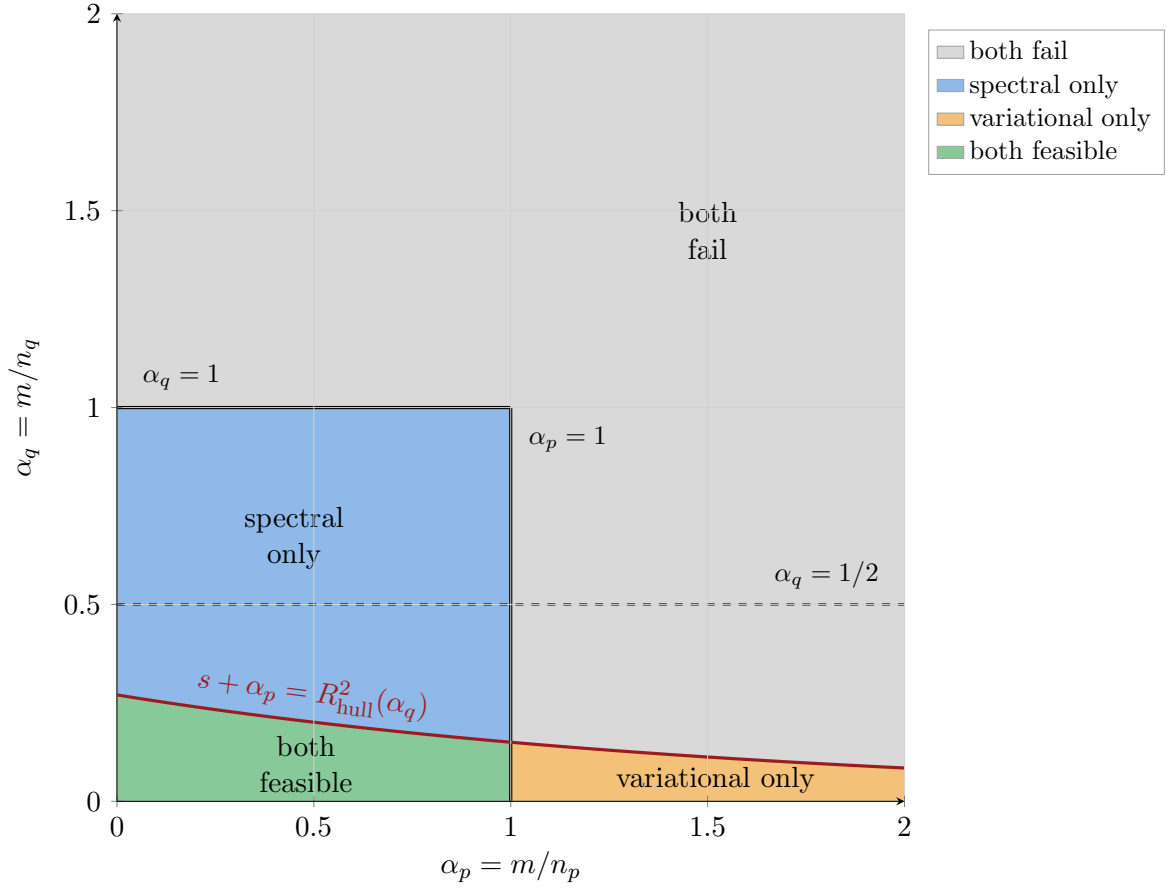

\begin{figure}
\centering
\includegraphics[width=0.4\linewidth]{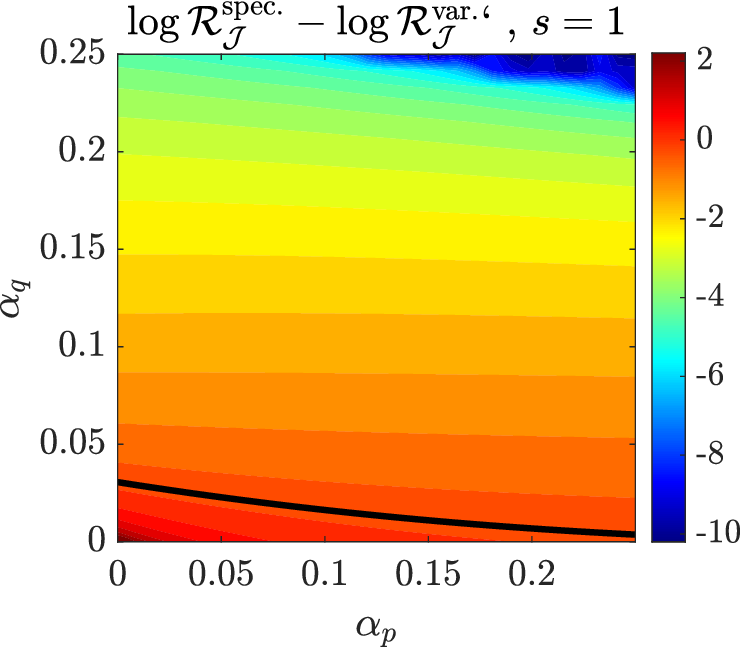}
\includegraphics[width=0.4\linewidth]{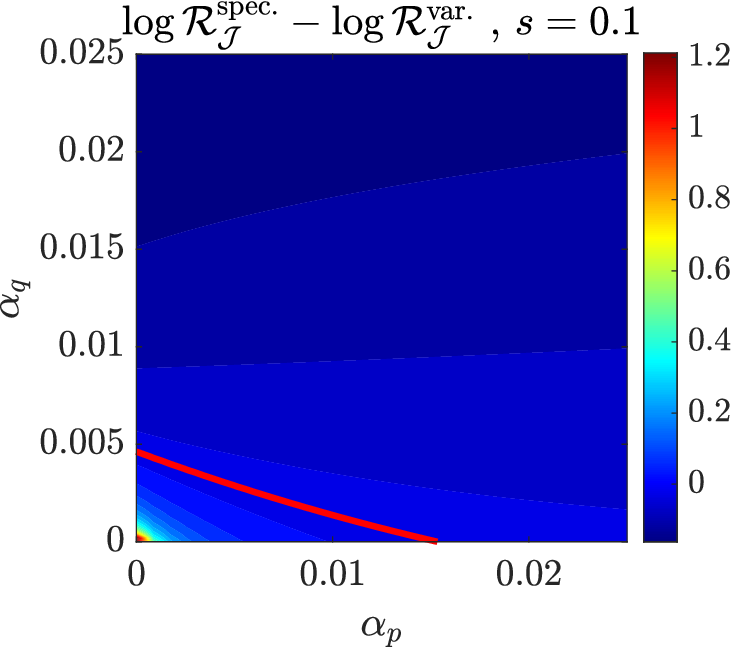}

\vspace*{-.25cm}

\caption{Comparison of two estimators for a given $s=1$ (left) and $s=0.1$ (right) for a fixed small regularization parameter (equal to $10^{-8}$). Only the difference in criterion $\mathcal{J}$ is reported as the difference in criterion $\mathcal{K}$ explodes. The black line (left plot) and red line (right plot) are zero-level lines.
\label{fig:criterion-k-contours-unreg}}
\end{figure}

\begin{figure}
\centering
\includegraphics[width=0.4\linewidth]{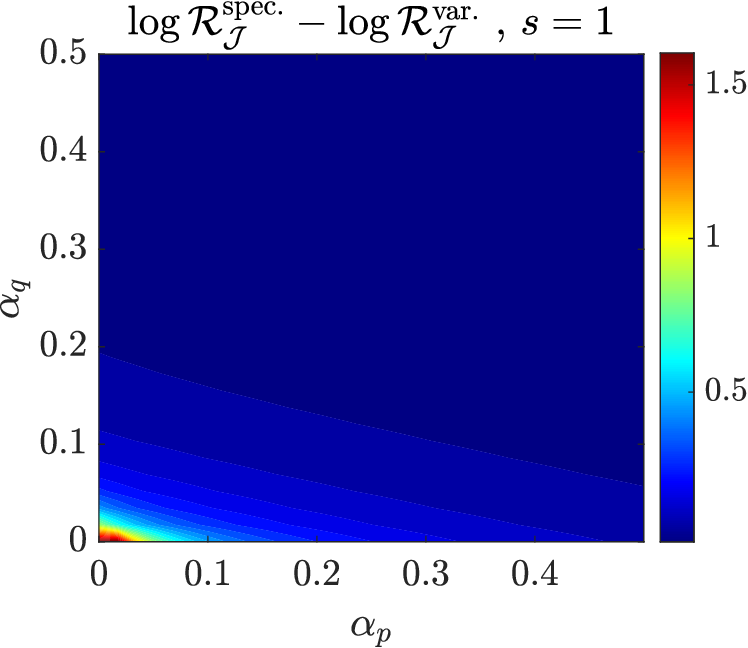}
\includegraphics[width=0.4\linewidth]{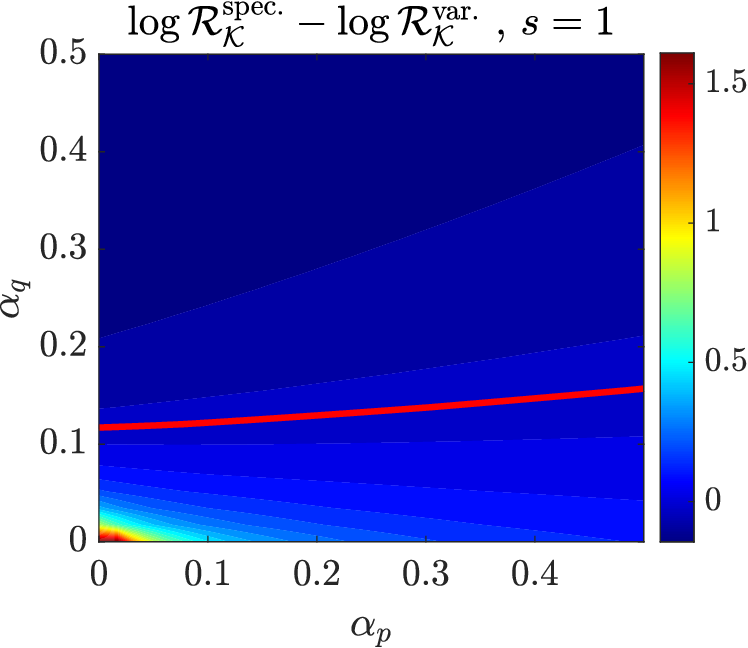}

\includegraphics[width=0.4\linewidth]{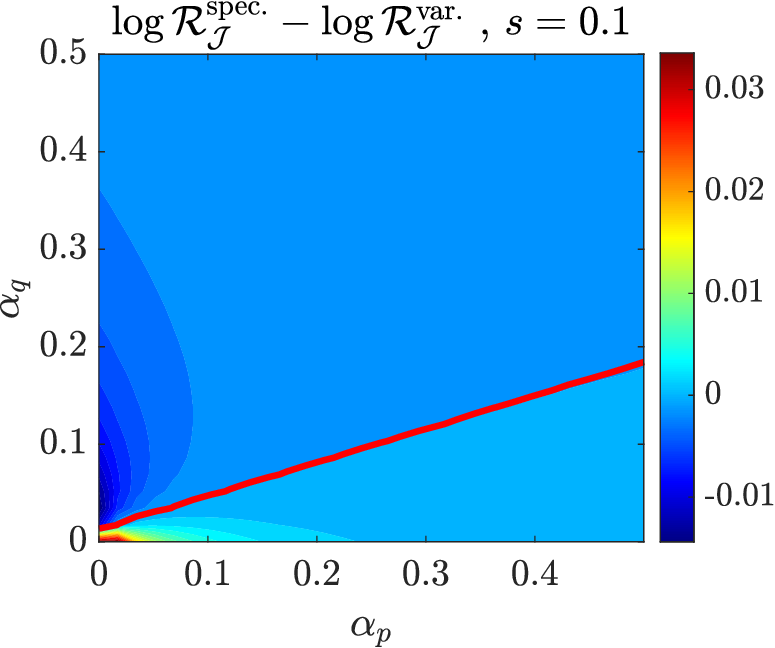}
\includegraphics[width=0.4\linewidth]{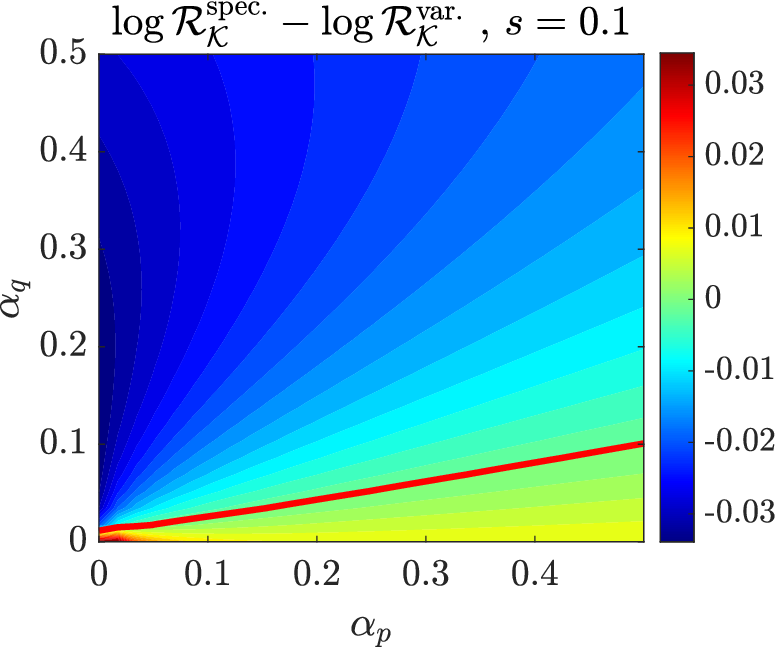}

\vspace*{-.25cm}

\caption{Comparison of two estimators for a given $s=1$ (top) and $s=0.1$ (bottom) when regularization parameters are optimized. Difference in criterion $\mathcal{J}$ (left) and criterion $\mathcal{K}$ (right). The red lines denote the zero line, above which the spectral method is better than the variational method. Note that  on the top left plot ($s=1$, criterion $\mathcal{J}$ that is insensitive to constants), the variational method is always better.\label{fig:criterion-k-contours-reg}}
\end{figure}

\paragraph{Fixed-signal aspect-ratio unregularized regimes.}
The fixed-signal contour plots in Figure~\ref{fig:criterion-k-contours-unreg} vary \((\alpha_p,\alpha_q)\)  in $[0,1/4] \times [0,1/4]$ at a fixed value of \(s=1\) (left plot), and \((\alpha_p,\alpha_q)\)  in $[0,1/40] \times [0,1/40]$ at a fixed value of \(s=0.1\) (right plot), with both regularization parameters equal to $10^{-8}$ to mimic unregularized estimation while preserving numerical stability. 
In \myfig{criterion-k-contours-unreg}, we show the difference  $\log(\mathcal{R}_{\mathcal{J},\rm spec})
- \log(\mathcal{R}_{\mathcal{J},\rm var})$, with the zero-level line, showing that for small $\alpha_p, \alpha_q$, the variational method has better performance, while for larger values, the spectral estimator is better. This generalizes the plot in \myfig{weak-signal-separating-curves} to larger values of $s, \alpha_p,\alpha_q$.

\paragraph{Regularization: comparison regimes.}

Figure~\ref{fig:criterion-k-contours-reg} displays the criteria
\[
\mathcal{R}^{\rm opt}_{\mathcal{J}, \rm var} (\alpha_p,\alpha_q,s)=\inf_{\tau>0}\mathcal{R}^{\tau}_{\mathcal{J}, \rm var} (\alpha_p,\alpha_q,s),
\qquad
\mathcal{R}^{\rm opt}_{\mathcal{J}, \rm spec} (\alpha_p,\alpha_q,s)=\inf_{\zeta>0}\mathcal{R}^{\rm \zeta}_{\mathcal{J}, \rm spec} (\alpha_p,\alpha_q,s),
\]
where we now make explicit the dependence of the risks on the regularization parameter used for training. This is done under the same setup as \myfig{criterion-k-contours-unreg}, but on the set $[0,1/2] \times [0,1/2]$, and for two values of $s$ ($s=1$ and $s=0.1$), and we only consider the $\mathcal{J}$ and $\mathcal{K}$ criteria, and plot $\log \mathcal{R}^{\rm opt}_{\mathcal{J}, \rm spec} (\alpha_p,\alpha_q,s) - \log \mathcal{R}^{\rm opt}_{\mathcal{J}, \rm var} (\alpha_p,\alpha_q,s)$ to assess which estimator is preferable. For each value of $(\alpha_p,\alpha_q,s)$, the optimal regularization parameters are found by golden-section search plus parabolic interpolation~\cite{brent2013algorithms}. We can draw conclusions that are similar to the ones for \myfig{criterion-k-contours-unreg}.

\section{Asymptotic comparisons}
\label{sec:experiments}

We consider comparisons between empirical estimation and the asymptotic results presented in \mysec{hd},  illustrating the convergence behaviors.\footnote{The MATLAB code for all figures in this paper can be found at \url{https://www.di.ens.fr/~fbach/fdiv_gaussian.zip}.} We empirically illustrate a match between the asymptotic determinstic equivalents and the empirical behavior for reasonable $m$ (as low as $100$ for the spectral estimator, bigger for the variational one).

\paragraph{Varying $m$.} We study in \myfig{sweepm} the convergence behavior towards our asymptotic limit, with fixed $\alpha_p=0.125$, $\alpha_q=0.125$, and $s=1$, with regularization parameters $\tau = 10^{-3}$ and $\zeta = 10^{-4}$. We vary $m$ and plot interquartile range and median obtained from 32 replications. We see that the empirical curves converge to their asymptotic limit, with a faster convergence for the spectral estimator (right panel). We also see the differences in criteria, where $\mathcal{K}$ measures the estimation of the normalization constants that the criterion $\mathcal{J}$ ignores. For the spectral estimator (right panel), we see that the $\mathcal{L}$ criterion, which uses the learned potential $w$ (instead of replacing it with $w=1-e^{v}$) leads to a significantly worse result.

\begin{figure}
    \centering
    \includegraphics[width=0.75\linewidth]{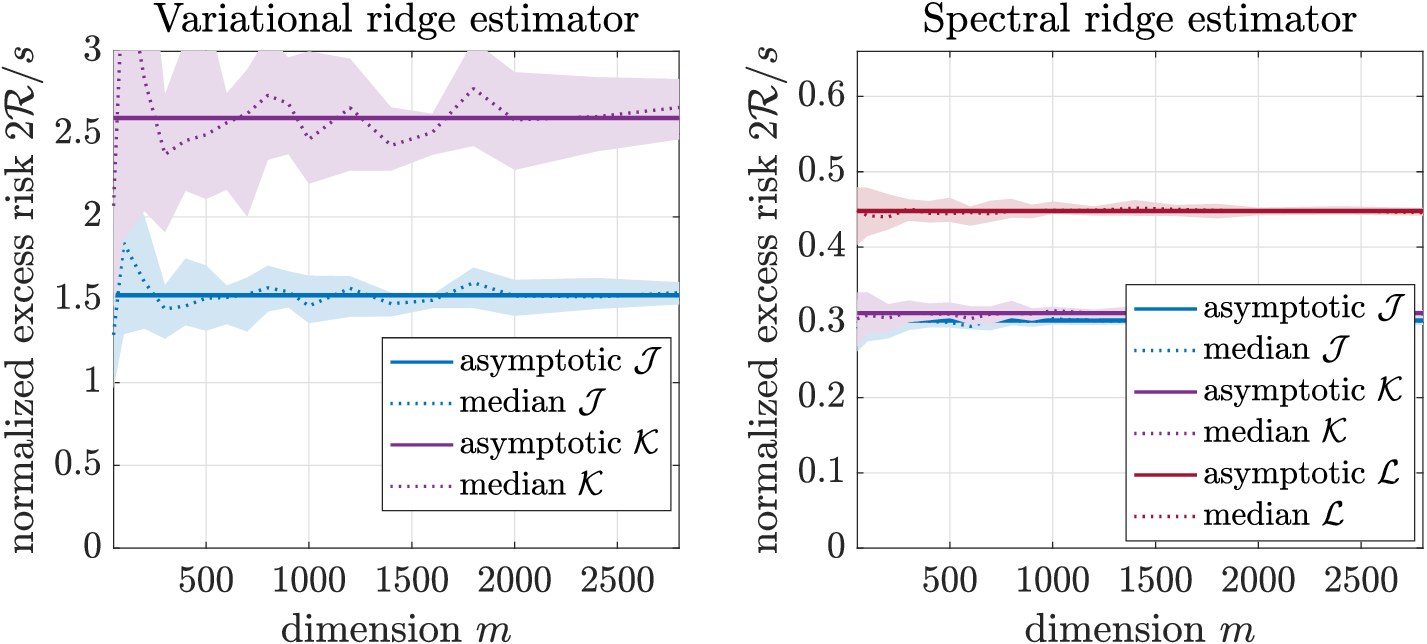}
    % sweep_m.m
    \caption{Comparison of estimators, limits, and empirical performance for varying $m$. Left: variational estimator, right: spectral estimator. \label{fig:sweepm} }
\end{figure}

\paragraph{Varying $\alpha_p$.}
 We set $s=\|\Delta\|^2=1$ and
$m=1000$ for the variational estimator, and $m=100$ for the spectral estimator, and for every displayed aspect ratio we draw independent samples from
$q=\mathcal{N}(0,\idm)$ and $p=\mathcal{N}(\Delta,\idm)$.  The empirical curves are averages over
16 independent replications; the shaded areas show the interquartile range across
these replications.  In \myfig{alphap} and \myfig{alphaq}, the ridge parameters are fixed at
$\tau=10^{-2}$ and $\zeta=10^{-4}$.  The left panel reports the ridge-regularized variational
estimator under the log-normalized criterion $\mathcal{J}$ and the Fenchel criterion $\mathcal{K}$,
whereas the right panel reports the ridge-regularized spectral estimator under
$\mathcal{J}$, $\mathcal{K}$, and the two-potential lower-bound criterion $\mathcal{L}$.  \myfig{alphap}
fixes $\alpha_q=0.1$ and sweeps $\alpha_p$ uniformly over $[0,1]$ using 20 grid
points; \myfig{alphaq} fixes $\alpha_p=0.05$
and sweeps $\alpha_q$ over the same interval.  Solid curves are the proportional
deterministic equivalents: the variational limits are obtained from the scalar
CGMT system, while the spectral limits are evaluated by the ridge-resolvent
formula using 256-point Gauss--Legendre quadrature.  The empirical spectral
curves are computed using the generalized-eigenvalue reduction from~\cite{Bach2026} presented in \mysec{est2}. 

\begin{figure}
     \centering
    \includegraphics[width=0.7\linewidth]{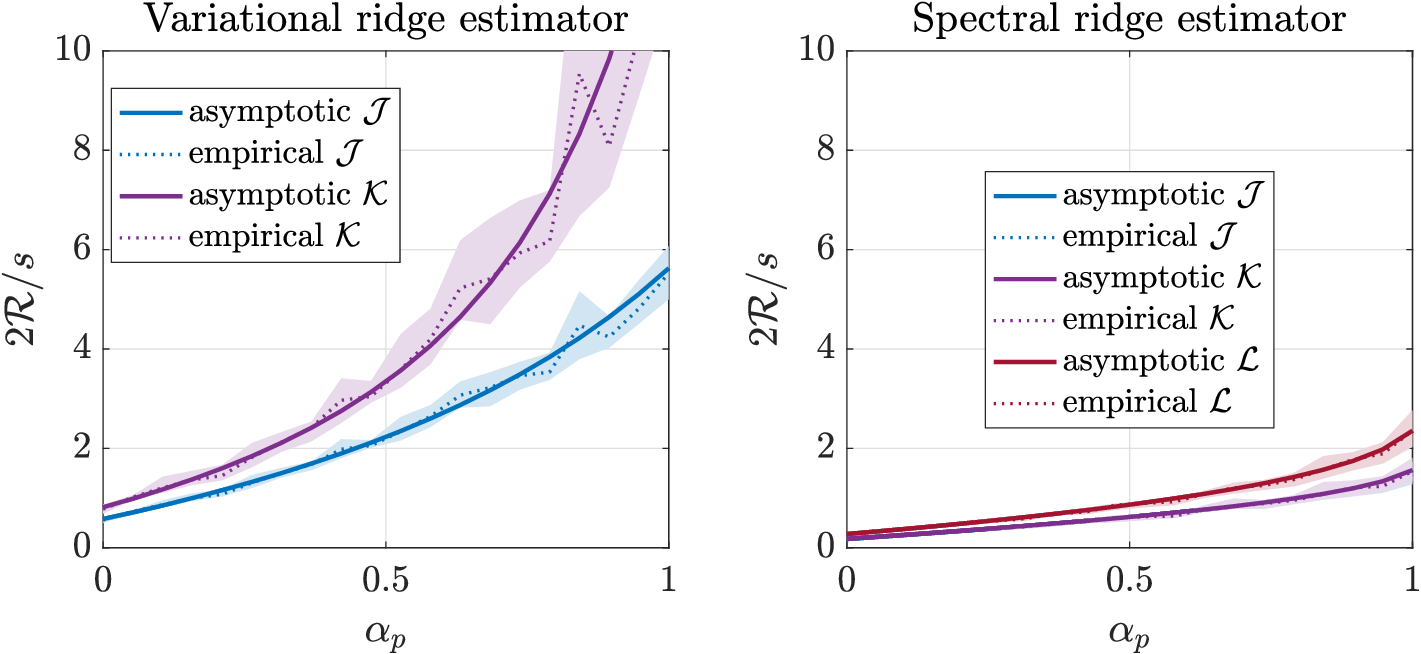}
    % compare_asymptotics_all_reps_quantiles.m
    \caption{Comparison of estimators when varying $\alpha_p$ ($\alpha_q=0.1$ fixed, with $\alpha_p$ swept over $[0,1]$). Left: variational, right: spectral. \label{fig:alphap} }
\end{figure}

\begin{figure}
    \centering
    \includegraphics[width=0.8\linewidth]{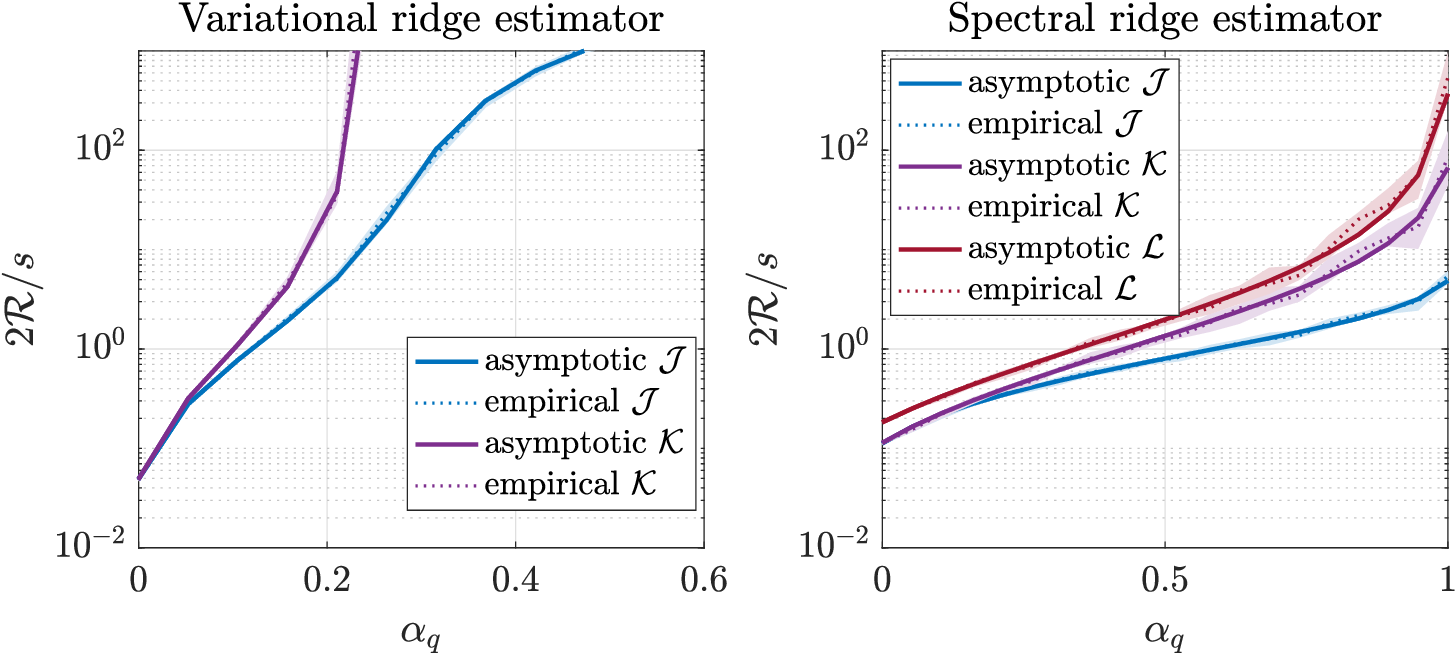}
    % compare_asymptotics_all_reps_quantiles.m
    \caption{Comparison of estimators when varying $\alpha_q$ ($\alpha_p=0.05$ fixed, with $\alpha_q$ swept over $[0,1]$, only shown partially for the variational estimator since the performance is out of range). Left: variational, right: spectral. \label{fig:alphaq} }
\end{figure}

\section{Conclusion}
\label{sec:conclusion}
 
\paragraph{Summary.}
This paper compared ridge-regularized variational and spectral
density-ratio estimation in a simple Gaussian location model
under proportional high-dimensional asymptotics
\[
        \frac{m}{n_p}\to\alpha_p,
        \qquad
        \frac{m}{n_q}\to\alpha_q .
\]
The variational estimator was analyzed through the entropy dual of the
ridge-regularized log-normalized KL objective and a CGMT
reduction~\cite{salehi2019impact}.  The spectral estimator of \cite{Bach2026} was
analyzed through random-matrix deterministic equivalents for ridge
resolvents of weighted sums of two independent Gaussian sample covariance
matrices, in the spirit of standard resolvent methods
\cite{BaiSilverstein2010,HachemLoubatonNajim2007}, but using a random feature formulation~\cite{montanari2022interpolation,Bach2024RandomProjections}.  These two asymptotic
calculations give computable population risks for our three criteria.

At small aspect
ratios (many observations), the affine variational estimator benefits from correct
specification of the Gaussian log-density ratio and can have the smaller
population risk.  As \(\alpha_p\) and\linebreak
especially \(\alpha_q\) increase,
empirical exponential normalization becomes more variable: in these larger-aspect regimes (with fewer observations), the spectral estimator is preferable in the regimes identified by the asymptotic comparison (for this simple model):
it replaces empirical log-sum-exp normalization by covariance-based
least-squares problems, trading some approximation bias for lower
normalization variance.  Positive ridge keeps both estimators finite for
all finite aspect ratios, and once optimized for regularization parameters, the spectral method is preferable for small numbers of observations (but with a reduced range of aspect values).

\paragraph{Extension to nuclear penalty.} In this paper, we primarily focused on an independent ridge penalty on all $\rho$-dependent linear predictors. In order to perform feature learning, other penalties can be considered, such as the nuclear penalty, which is classical in this context of multi-task learning~\cite{ArgyriouEvgeniouPontil2008}. Iteratively reweighted least-square formulations~\cite{daubechies2010iteratively,BachJenattonMairalObozinski2012} can naturally be used, both for algorithms and for the asymptotic analysis, which can now be done through matrix Dyson equations~\cite{AjankiErdosKruger2019MDE}. Overall, given that, for our Gaussian location model, the optimal predictors are all aligned, we see an improvement over ridge penalties.
See Appendix~\ref{app:nuclear} for details and derivation of deterministic equivalents a few illustrative experiments.

\paragraph{Extension to mutual information.}
As noted in \cite{Bach2026}, a first natural extension is mutual information estimation, since it is equal to the KL divergence between a joint distribution $p(x_1,x_2)$ and the product of its marginals.  Empirically, this corresponds to
comparing paired samples from the joint law with shuffled or independently
paired samples from the product of marginals.  This connects directly to
variational mutual-information estimation and variational KL bounds
\cite{DonskerVaradhan1975,NguyenWainwrightJordan2010,
BelghaziEtAl2018,PooleEtAl2019}.     The corresponding
high-dimensional question is to determine how the CGMT and random-matrix
limits change when the two empirical samples come from joint and product
feature distributions rather than from a Gaussian location pair.  Such an
analysis would give asymptotic risk predictions for variational and
spectral mutual-information estimators as functions of dependence
strength, feature dimension, sample size, and regularization, with in particular application to softmax regression when one of the two variables takes finitely many values.

\paragraph{Extension to nonlinear features.}
A second extension is to replace the identity feature map by richer
nonlinear features, for instance
\(
        \varphi(x)=\sigma(Wx),
\)
with fixed random weights $W$, kernel features, or features learned from data.
This preserves a linear-in-parameters
structure after feature construction
\cite{RahimiRecht2007,WilliamsSeeger2001}.  However, after the nonlinear
map, the empirical design is no longer simply a Gaussian matrix or a pair
of Wishart matrices.  Recent high-dimensional analyses of nonlinear
random-feature Gram matrices, random-feature ridge regression,
Gaussian-equivalent feature models, and generic feature maps show that
such problems often still admit deterministic equivalents or
low-dimensional state-evolution descriptions
\cite{LouartLiaoCouillet2018,MeiMontanari2022,HuLu2023,
GeraceEtAl2020,LoureiroEtAl2021GenericFeatures,
LoureiroEtAl2022ConvexLosses}.  This suggests an extension of the present
CGMT and resolvent calculations in which the ambient dimension \(m\) is
replaced by the effective spectral distribution of the nonlinear feature
covariance or kernel matrix.  The main technical question would then be
to rederive the asymptotic equivalents and compare their performance.

\paragraph{Extension to spiked and covariance-shift models.}
A third extension is to study settings where $p$ and $q$ differ not only in their means but also, or primarily, in their covariance matrices \cite{Johnstone2001,BaikBenArousPeche2005,BaiSilverstein2010}.
In such models, part of the signal is second-order rather than purely
mean-shifted.  The spectral estimator may then be structurally better
aligned with the data-generating mechanism, because it is built from
covariance-based least-squares problems.  By contrast, an affine
variational estimator would be misspecified unless augmented with
quadratic or nonlinear features.  These extensions would clarify whether
the spectral advantage observed at larger aspect ratios is specific to the
Gaussian location model, or whether it reflects a broader high-dimensional
phenomenon in density-ratio and mutual-information estimation.

\vspace*{1.25cm}

\subsection*{Tool and computational resource disclosure}
During the exploratory phase of this work, the author used a large language model (GPT-pro 5.5) to explore possible applications of the convex Gaussian min-max theorem (CGMT) and random matrix theory to the problem studied in this paper, including candidate reductions, intermediate identities, and proof strategies.

All mathematical claims, statements of assumptions, and final proofs in the main paper were subsequently checked, significantly rewritten, and completed by the author. Proofs in the appendix were reviewed for correctness. The final arguments rely on the author's own mathematical judgement and background knowledge, and no result is included solely on the basis of an LLM-generated derivation. 

The same LLM was also used to assist with MATLAB code generation for numerical experiments and figure production. The resulting code was reviewed, edited, and validated by the author, including checks for consistency with the stated mathematical model and reproduction of the reported plots. All numerical experiments and figures reported in the paper were run on a single CPU. 

\subsection*{Acknowledgements}
   This work has received support from the French government, managed by the National Research Agency,
under the France 2030 program with the reference ``PR[AI]RIE-PSAI'' (ANR-23-IACL-0008).

\newpage

\appendix
 
\section{Proof of Proposition~\ref{prop:var-dual}}
\label{app:proof-var-dual}

\begin{proof}
Write \(\mu=\hat\mu_p\), \(n=n_q\), and
\(
\Phi(\theta)
=
\theta^\top\mu
-
\log\big(\frac1n\sum_{j=1}^n \exp(\theta^\top y_j)\big).
\)
Let \(C=\operatorname{conv}\{y_1,\ldots,y_n\}\) and
\(S=\operatorname{span}(C-C)\).
If \(\mu\notin C\), strict separation gives a vector \(a\) such that
\(
a^\top\mu>\max_{1\leqslant  j\leqslant  n} a^\top y_j .
\)
Hence, as \(t\to\infty\), we get 
\(
\Phi(ta)
=
t (a^\top\mu-\max_{j \in \{1,\dots,n\}} a^\top y_j)+O(1)\to+\infty .
\)
Thus the supremum is infinite.

Assume now that \(\mu\in C\). Then, for every \(\theta\), we have
\(
\theta^\top\mu\leqslant  \max_{j \in \{1,\dots,n\}}\theta^\top y_j
\)
and therefore
\[
\Phi(\theta)
\leqslant
\max_{j \in \{1,\dots,n\}}\ \theta^\top y_j
-
\log\bigg(\frac1n\sum_{j=1}^n \exp(\theta^\top y_j)\bigg)
\leqslant  \log n .
\]
Thus the supremum is finite. Moreover, if \(h\in S^\perp\), then
\(h^\top y_j\) is constant in \(j\), and since \(\mu\in C\), this constant is
\(h^\top\mu\). Hence \(\Phi(\theta+h)=\Phi(\theta)\). This proves the
non-identifiability modulo \(S^\perp\), and it remains to work on~\(S\).

Suppose first that \(\mu\in \operatorname{ri} C\). If \(S=\{0\}\), there is
nothing to prove. Otherwise,
\[
\delta
:=
\inf_{\substack{u\in S\\ \|u\|=1}}
\max_{j \in \{1,\dots,n\}} u^\top(y_j-\mu)
>0 .
\]
Indeed, a zero value would give a nonzero supporting direction through
\(\mu\), contradicting \(\mu\in\operatorname{ri} C\). For
\(\theta=tu\in S\), with \(t=\|\theta\|\) and \(\|u\|=1\),
\[
\Phi(\theta)
\leqslant
\log n
-
t\max_{j \in \{1,\dots,n\}} u^\top(y_j-\mu)
\leqslant
\log n-\delta \|\theta\|.
\]
Thus \(\Phi\) is coercive from above on \(S\). Since \(\Phi\) is continuous,
it attains its maximum on \(S\), and hence in \(\mathbb R^m\) modulo
\(S^\perp\).

It remains to consider \(\mu\in C\setminus \operatorname{ri} C\). Let \(F\)
be the smallest face of \(C\) containing \(\mu\), and set
\(
I=\{j:\ y_j\in F\}.
\)
Since \(C\) is a polytope, \(F\) is exposed: there exists \(a\) such that
\[
a^\top y_j=a^\top\mu \quad (j\in I),
\qquad
a^\top y_j<a^\top\mu \quad (j\notin I).
\]
Define the face-restricted objective
\(
\Phi_F(\theta)
=
\theta^\top\mu
-
\log\big(\frac1n\sum_{j\in I}\exp(\theta^\top y_j)\big).
\)
Since \(\mu\in\operatorname{ri}F\), the previous paragraph applied inside
\(\operatorname{aff}F\) gives a maximizer \(\theta_F\) of \(\Phi_F\). Also,
\(\Phi(\theta)\leqslant  \Phi_F(\theta)\) for every \(\theta\). On the other hand,
\(
\Phi(\theta_F+ta)
\longrightarrow
\Phi_F(\theta_F)
\) when  \( t\to\infty 
\),
because the terms with \(j\notin I\) are exponentially negligible relative
to the terms with \(j\in I\). Hence
\(
\sup_{\theta\in\mathbb R^m}\Phi(\theta)=\max_{\theta}\Phi_F(\theta)<\infty .
\)
For every finite \(\theta\), the terms with \(j\notin I\) are strictly
positive, so
\(
\Phi(\theta)<\Phi_F(\theta)\leqslant  \max_\eta \Phi_F(\eta)=\sup_\eta \Phi(\eta).
\)
Thus the supremum is not attained. The sequence
\(\theta_F+ta\) is a maximizing sequence escaping to infinity in a direction
exposing the smallest face containing \(\mu\). Conversely, any maximizing
sequence must be unbounded; otherwise a convergent subsequence would yield a
finite maximizer. This completes the proof.
\end{proof}

\section{Proof of Proposition~\ref{prop:chi-feasible}}
\label{app:proof-chi-feasible}

\begin{proof}
Assume first that \(\hat C_p\succ0\) and \(\hat C_q\succ0\). Then, for every
\(\rho\in[0,1]\),
\[
\hat C(\rho)=\rho\hat C_p+(1-\rho)\hat C_q
\succcurlyeq
\min\{\lambda_{\min}(\hat C_p),\lambda_{\min}(\hat C_q)\}I
\succ0 .
\]
Hence
\(
\hat M(\rho)
=
\hat C(\rho)+\rho(1-\rho)\hat\Delta\hat\Delta^\top
\succ0
\)
uniformly on \([0,1]\). Therefore
\(
\rho\mapsto \hat M(\rho)^{-1}\hat\Delta
\)
is continuous and bounded on \([0,1]\). Since
\(
\hat\mu(\rho)=\rho\hat\mu_p+(1-\rho)\hat\mu_q
\)
is also continuous, all integrands defining the coefficients in
Eqs.~\eqref{eq:Ahat-ridge}-\eqref{eq:chat-ridge} are continuous and bounded. Thus the
unregularized full spectral potential in \eq{vhat-cont} is finite.

Conversely, suppose that \(\hat C_q\) is singular. Let \(\hat P_q\) denote the
orthogonal projector onto \(\ker\hat C_q\). We first show the deterministic implication
\[
\hat P_q\hat\Delta\ne0
\quad\Longrightarrow\quad
\text{the full potential is not finite.}
\]
If \(\hat M(\rho)\) is singular for some \(\rho\in(0,1)\), the ordinary-inverse
construction is already undefined. Otherwise, set
\(
\hat\beta_\rho=\hat M(\rho)^{-1}\hat\Delta .
\)
From \(\hat M(\rho)\hat\beta_\rho=\hat\Delta\), projection onto \(\ker\hat C_q\) gives
\(
\hat P_q\hat\Delta
=
\rho\hat P_q\hat C_p\hat\beta_\rho
+
\rho(1-\rho)\hat P_q\hat\Delta\,\hat\Delta^\top\hat\beta_\rho ,
\)
because \(\hat P_q\hat C_q=0\). Hence
\[
\|\hat P_q\hat\Delta\|
\leqslant
\rho\bigl(\|\hat C_p\|+\|\hat\Delta\|^2\bigr)\|\hat\beta_\rho\|.
\]
If \(\hat P_q\hat\Delta\ne0\), then, for some constant \(c_q>0\),
\(
\|\hat\beta_\rho\|\geqslant {c_q\over \rho}
,
\) with $0<\rho<1$.
Therefore
\[
\int_0^\varepsilon
\rho(1-\rho)\|\hat\beta_\rho\|^2\,d\rho
\geqslant
c_q^2\int_0^\varepsilon {1-\rho\over \rho}\,d\rho
=
+\infty .
\]
But the quadratic coefficient satisfies
\[
\hat A
=
-\int_0^1 \rho(1-\rho)\hat\beta_\rho\hat\beta_\rho^\top\,d\rho ,
\]
so a finite full quadratic potential would imply
\[
\int_0^1\rho(1-\rho)\|\hat\beta_\rho\|^2\,d\rho<\infty .
\]
This contradiction proves non-finiteness at the \(q\)-endpoint.

The \(p\)-endpoint is identical. If \(\hat C_p\) is singular and \(\hat P_p\) is the
orthogonal projector onto \(\ker\hat C_p\), then
\(
\hat P_p\hat\Delta\ne0
\)
implies, with \(\varepsilon=1-\rho\),
\(
\|\hat\beta_{1-\varepsilon}\|\geqslant {c_p\over \varepsilon}
\)
for some \(c_p>0\). Consequently,
\[
\int_{1-\varepsilon_0}^1
\rho(1-\rho)\|\hat\beta_\rho\|^2\,d\rho
=
+\infty ,
\]
and the full potential is again not finite.

It remains to verify that these projection conditions hold almost surely under the Gaussian sampling
model. For Gaussian samples, each centered sample covariance is independent of the corresponding sample
mean, and the \(p\)- and \(q\)-samples are independent. Thus \(\hat\Delta=\hat\mu_p-\hat\mu_q\) is independent
of \(\hat C_p\) and \(\hat C_q\). Conditional on either covariance matrix,
\(
\hat\Delta\sim
N\!\big(\Delta,\big({1\over n_p}+{1\over n_q}\big)I\big).
\)
If \(\hat C_q\) is singular, then \(\ker\hat C_q\ne\{0\}\), and conditional on \(\hat C_q\),
\(
\hat P_q\hat\Delta
\)
is a nondegenerate Gaussian vector on \(\ker\hat C_q\), with covariance
\(
\big({1\over n_p}+{1\over n_q}\big)\hat P_q .
\)
Therefore
\(
\mathbb P(\hat P_q\hat\Delta=0\mid \hat C_q)=0
\)
on the event that \(\hat C_q\) is singular. The same argument gives
\(
\mathbb P(\hat P_p\hat\Delta=0\mid \hat C_p)=0
\)
on the event that \(\hat C_p\) is singular.

Finally, centered Gaussian sample covariances have ranks
\(
\operatorname{rank}(\hat C_p)=\min(m,n_p-1), \)
\( 
\operatorname{rank}(\hat C_q)=\min(m,n_q-1)
\)
almost surely. Hence
\(
\hat C_p\succ0 \Leftrightarrow n_p\geqslant m+1,
\hat C_q\succ0 \Leftrightarrow n_q\geqslant m+1
\)
almost surely. Combining the positive-definite case with the endpoint-divergence argument proves that
the unregularized continuum spectral full potential is finite almost surely if and only if
\(
n_p,n_q\geqslant m+1 .
\)
\end{proof}

\section{Detailed computations for \mysec{specpop}}
\label{app:proof-pop-spec-ridge}

Set
\[
a=1+\zeta,\qquad 
g(\rho)=\rho(1-\rho),\qquad
D(\rho)=a+s g(\rho),\qquad
t=\Delta^\top x .
\]
At the population level, we have
\[
M_0(\rho)+\zeta I
=
(1+\zeta)I+\rho(1-\rho)\Delta\Delta^\top
=
aI+g(\rho)\Delta\Delta^\top .
\]
Since \(\|\Delta\|^2=s\),
\[
\{aI+g(\rho)\Delta\Delta^\top\}\Delta
=
(a+s g(\rho))\Delta
=
D(\rho)\Delta .
\]
Therefore
\[
\beta^{(0)}(\rho)=\frac{\Delta}{D(\rho)},
\qquad
u^{(0)}(\rho)(x)
=
\beta^{(0)}(\rho)^\top(x-\rho\Delta)
=
\frac{t-\rho s}{D(\rho)}.
\]

We use throughout
\[
h(s)=\int_0^1\frac{d\rho}{D(\rho)}
=
\int_0^1\frac{d\rho}{a+s\rho(1-\rho)} ,
\]
where the derivative \(h'(s)\) is taken with respect to \(s\), keeping \(a=1+\zeta\) fixed. Thus
\[
h'(s)
=
-\int_0^1\frac{g(\rho)}{D(\rho)^2}\,d\rho .
\]
For \(s>0\), the closed form is obtained by writing \(\rho=1/2+z\):
\[
h(s)
=
\int_{-1/2}^{1/2}
\frac{dz}{a+s/4-sz^2}
=
\frac{4}{\sqrt{s\{s+4a\}}}
\operatorname{arctanh}
\sqrt{\frac{s}{s+4a}},
\]
with continuous extension
\(
h(0)=\frac1a.
\)

We first record the elementary integral identities used below. By symmetry of
\(D(\rho)\) under \(\rho\mapsto 1-\rho\),
\[
\int_0^1\frac{\rho}{D(\rho)}\,d\rho
=
\int_0^1\frac{1-\rho}{D(\rho)}\,d\rho
=
\frac{h(s)}{2},
\quad \mbox{ and } \quad
\int_0^1\frac{\rho g(\rho)}{D(\rho)^2}\,d\rho
=
\int_0^1\frac{(1-\rho)g(\rho)}{D(\rho)^2}\,d\rho
=
-\frac{h'(s)}{2}.
\]
Moreover, since \(D(\rho)=a+s g(\rho)\),
\[
\int_0^1\frac{g(\rho)}{D(\rho)}\,d\rho
=
\frac{1-a h(s)}{s},
\quad
\mbox{ and }
\quad
\int_0^1\frac{g(\rho)^2}{D(\rho)^2}\,d\rho
=
\frac{1-a h(s)}{s^2}
+
\frac{a}{s}h'(s).
\]
Also,
\[
\int_0^1\frac{d\rho}{D(\rho)^2}
=
\frac{h(s)+s h'(s)}{a},
\]
because
\[
h(s)
=
\int_0^1\frac{D(\rho)}{D(\rho)^2}\,d\rho
=
a\int_0^1\frac{d\rho}{D(\rho)^2}
+
s\int_0^1\frac{g(\rho)}{D(\rho)^2}\,d\rho .
\]
Consequently,
\[
\int_0^1\frac{(1-\rho)^2}{D(\rho)^2}\,d\rho
=
\frac12\int_0^1\frac{d\rho}{D(\rho)^2}
-
\int_0^1\frac{g(\rho)}{D(\rho)^2}\,d\rho
=
\frac{h(s)+(s+2a) h'(s)}{2a}.
\]
The identities containing \(1/s\) are intermediate identities for \(s>0\); the final formulas below extend continuously to \(s=0\).

We now compute \(v_{\mathrm{spec}}^{(0)}\). By definition,
\[
v_{\mathrm{spec}}^{(0)}(x)
=
\int_0^1
2(1-\rho)
\left(
u_0(\rho)(x)-\frac{\rho}{2}u_0(\rho)(x)^2
\right)\,d\rho .
\]
Substituting \(u^{(0)}(\rho)(x)=(t-\rho s)/D(\rho)\) gives
\[
v_{\mathrm{spec}}^{(0)}(x)
=
\int_0^1
\left[
\frac{2(1-\rho)(t-\rho s)}{D(\rho)}
-
\frac{\rho(1-\rho)(t-\rho s)^2}{D(\rho)^2}
\right]\,d\rho .
\]
Hence \(v_{\mathrm{spec}}^{(0)}\) is a scalar quadratic polynomial in \(t=\Delta^\top x\):
\[
v_{\mathrm{spec}}^{(0)}(x)
=
A_v(s)t^2+B_v(s)t+C_v(s).
\]
The quadratic coefficient is
\[
A_v(s)
=
-\int_0^1\frac{g(\rho)}{D(\rho)^2}\,d\rho
=
h'(s).
\]
The linear coefficient is
\[
\begin{aligned}
B_v(s)
&=
2\int_0^1\frac{1-\rho}{D(\rho)}\,d\rho
+
2s\int_0^1\frac{\rho g(\rho)}{D(\rho)^2}\,d\rho =
h(s)-s h'(s).
\end{aligned}
\]
For the constant coefficient,
\[
\begin{aligned}
C_v(s)
&=
-2s\int_0^1\frac{g(\rho)}{D(\rho)}\,d\rho
-
s^2\int_0^1\frac{\rho^2 g(\rho)}{D(\rho)^2}\,d\rho .
\end{aligned}
\]
Since \(\rho^2=\rho-g(\rho)\),
\[
\int_0^1\frac{\rho^2 g(\rho)}{D(\rho)^2}\,d\rho
=
\int_0^1\frac{\rho g(\rho)}{D(\rho)^2}\,d\rho
-
\int_0^1\frac{g(\rho)^2}{D(\rho)^2}\,d\rho
=
-\frac{h'(s)}{2}
-
\left(
\frac{1-a h(s)}{s^2}
+
\frac{a}{s}h'(s)
\right).
\]
Therefore
\[
\begin{aligned}
C_v(s)
&=
-2\{1-a h(s)\}
+
\frac{s^2}{2}h'(s)
+
\{1-a h(s)\}
+
a s h'(s)  =
a h(s)-1+\left(a s+\frac{s^2}{2}\right)h'(s).
\end{aligned}
\]
Thus
\[
v_{\mathrm{spec}}^{(0)}(x)
=
h'(s)t^2
+
\{h(s)-s h'(s)\}t
+
a h(s)-1
+
\left(a s+\frac{s^2}{2}\right)h'(s),
\qquad
t=\Delta^\top x,\quad a=1+\zeta .
\]

We next compute the companion potential \(w_{\mathrm{spec}}^{(0)}\). By definition,
\[
w_{\mathrm{spec}}^{(0)}(x)
=
\int_0^1
2(1-\rho)
\left(
-u_0(\rho)(x)-\frac{1-\rho}{2}u_0(\rho)(x)^2
\right)\,d\rho .
\]
Substituting \(u^{(0)}(\rho)(x)=(t-\rho s)/D(\rho)\) gives
\[
w_{\mathrm{spec}}^{(0)}(x)
=
\int_0^1
\left[
-\frac{2(1-\rho)(t-\rho s)}{D(\rho)}
-
\frac{(1-\rho)^2(t-\rho s)^2}{D(\rho)^2}
\right]\,d\rho .
\]
Thus
\[
w_{\mathrm{spec}}^{(0)}(x)
=
A_w(s)t^2+B_w(s)t+C_w(s).
\]
The quadratic coefficient is
\[
A_w(s)
=
-\int_0^1\frac{(1-\rho)^2}{D(\rho)^2}\,d\rho
=
-\frac{h(s)+(s+2a)h'(s)}{2a}.
\]
The linear coefficient is
\[
\begin{aligned}
B_w(s)
&=
-2\int_0^1\frac{1-\rho}{D(\rho)}\,d\rho
+
2s\int_0^1\frac{\rho(1-\rho)^2}{D(\rho)^2}\,d\rho  =
-h(s)-s h'(s),
\end{aligned}
\]
because
\[
\int_0^1\frac{\rho(1-\rho)^2}{D(\rho)^2}\,d\rho
=
\int_0^1\frac{(1-\rho)g(\rho)}{D(\rho)^2}\,d\rho
=
-\frac{h'(s)}{2}.
\]
Finally,
\[
\begin{aligned}
C_w(s)
&=
2s\int_0^1\frac{g(\rho)}{D(\rho)}\,d\rho
-
s^2\int_0^1\frac{g(\rho)^2}{D(\rho)^2}\,d\rho \\
&=
2(1-a h(s))
-
\left(
1-a h(s)+a s h'(s)
\right)  =
1-a h(s)-a s h'(s).
\end{aligned}
\]
Therefore
\[
w_{\mathrm{spec}}^{(0)}(x)
=
-\frac{h(s)+(s+2a)h'(s)}{2a}t^2
-
(h(s)+s h'(s))t
+
1-a h(s)-a s h'(s),
\qquad
t=\Delta^\top x,\quad a=1+\zeta .
\]

\section{Proof of Theorem~\ref{thm:var-ridge}}
\label{app:var-ridge}

\begin{proof}
We write
\[
  n=n_q,\qquad \frac{m}{n}\to \alpha=\alpha_q,
  \qquad r^2=s+\alpha_p,
\]
and denote the \(q\)-sample matrix by
 $Y = (y_1,\dots,y_n)^\top \in\R^{n\times m}.
$
Thus \(Y u\in\R^n\) has entries \(y_j^\top u\), and
\(Y^\top a=\sum_{j=1}^n a_jy_j\).  We first condition on
\(\hat\mu_p=\mu\).  It is enough to prove the conditional result for any
sequence of deterministic vectors \(\mu=\mu_n\) such that
\[
  \|\mu\|^2\to r^2,
  \qquad
  \Delta^\top\mu\to s.
\]
Indeed, \(\hat\mu_p\) is independent of the \(q\)-sample, and these two
limits hold in probability under Assumption~\ref{ass:hd}; the unconditional
result follows by conditioning.  We prove the nondegenerate case
\(s+\alpha_p+\alpha_q>0\), which is the case in which the displayed formulas
with \(B^{-1}\) are used.  If \(s=\alpha_p=\alpha_q=0\), the limiting entropy
problem has the unique optimizer \(A\equiv1\), the limiting coefficient is
\(\xi=0\), and the fitted affine potential converges to zero, giving
\(J_{\rm var}=K_{\rm var}=0\).

Let
\[
  h(a)=a\log a-a+1,
  \qquad a\geqslant 0,
\]
with the convention \(h(0)=1\).

\emph{Step 1: ridge entropy dual.}
For every \(z\in\R^n\), the entropy conjugacy of log-sum-exp gives
\[
  \log\bigg(\frac1n\sum_{j=1}^n e^{z_j}\bigg)
  =
  \sup_{\substack{a\in\R_+^n\\ n^{-1}\sum_{j=1}^n a_j=1}}
    \frac1n a^\top z-\frac1n\sum_{j=1}^n a_j\log a_j
.
\]
Since \(n^{-1}\sum_j a_j=1\), the entropy term may equivalently be written
as \(n^{-1}\sum_j h(a_j)\).  Applying this identity to \(z=Y\theta\), and
using strong concavity in the ridge coefficient, gives
\[
\begin{aligned}
&\max_{\theta\in\R^m}
    \theta^\top\mu
    -\log\bigg(\frac1n\sum_{j=1}^n e^{y_j^\top\theta}\bigg)
    -\frac{\tau}{2}\|\theta\|^2
   \\
&\quad =
  \min_{\substack{a\in\R_+^n\\ n^{-1}\sum_{j=1}^n a_j=1}}
  \max_{u\in\R^m}
    \frac1n\sum_{j=1}^n h(a_j)
    +u^\top\mu
    -\frac1n u^\top Y^\top a
    -\frac{\tau}{2}\|u\|^2
  \\
&\quad =
  \min_{\substack{a\in\R_+^n\\ n^{-1}\sum_{j=1}^n a_j=1}}
    \frac1n\sum_{j=1}^n h(a_j)
    +\frac{1}{2\tau}\left\|\mu-\frac1nY^\top a\right\|^2
.
\end{aligned}
\]
The maximizer over \(u\) is
\(
  u=\frac1\tau\left(\mu-\frac1nY^\top a\right).
\)
At the saddle point this vector is the fitted ridge coefficient, hence
\(
  \hat\theta
  =\frac1\tau\left(\mu-\frac1nY^\top a\right).
\)
Thus the ridge penalty replaces the zero-ridge moment constraint
\(n^{-1}Y^\top a=\mu\) by a squared residual penalty.

\emph{Step 2: compact CGMT reduction.}
For fixed \(a\), dualize the squared norm as
\[
  \frac{1}{2\tau}\left\|\mu-\frac1nY^\top a\right\|^2
  =
  \sup_{u\in\R^m}
    u^\top\mu-\frac1n u^\top Y^\top a-\frac{\tau}{2}\|u\|^2
.
\]
Fix \(0<\varepsilon<1<M\) and \(R<\infty\), and consider the compact primary
optimization
\[
\begin{aligned}
\varphi_{n,\tau}^{\varepsilon,M,R}(\mu)
=
\min_{\substack{a\in[\varepsilon,M]^n\\ n^{-1}\sum_{j=1}^n a_j=1}}
\max_{\|u\|\leqslant R}
  \frac1n\sum_{j=1}^n h(a_j)
  +u^\top\mu
  -\frac1n u^\top Y^\top a
  -\frac{\tau}{2}\|u\|^2.
\end{aligned}
\]
This is convex in \(a\) and concave in \(u\) on compact convex sets.  The CGMT
auxiliary optimization has independent standard Gaussian vectors
\(g\in\R^m\) and \(h \in\R^n\), and replaces the bilinear
Gaussian term by
\[
  -\frac1n u^\top Y^\top a
  \quad\leadsto\quad
  \frac{\|a\|}{n}g^\top u-\frac{\|u\|}{n}h^\top a.
\]
Hence
\[
\begin{aligned}
\widetilde\varphi_{n,\tau}^{\varepsilon,M,R}(\mu)
=
\min_{\substack{a\in[\varepsilon,M]^n\\ n^{-1}\sum_{j=1}^n a_j=1}}
\max_{\|u\|\leqslant R}
  \frac1n\sum_{j=1}^n h(a_j)
  +u^\top\mu
  +\frac{\|a\|}{n}g^\top u
  -\frac{\|u\|}{n}h^\top a
  -\frac{\tau}{2}\|u\|^2
.
\end{aligned}
\]
For fixed \(a\), define
\[
  Z_a=\mu+\frac{\|a\|}{n}g,
  \qquad
  \ell_a=\frac1nh^\top a.
\]
Optimizing over the direction of \(u\), and writing \(\xi=\|u\|\), gives
\[
\begin{aligned}
\widetilde\varphi_{n,\tau}^{\varepsilon,M,R}(\mu)
=
\min_{\substack{a\in[\varepsilon,M]^n\\ n^{-1}\sum_{j=1}^n a_j=1}}
\max_{0\leqslant \xi\leqslant R}
  \frac1n\sum_{j=1}^n h(a_j)
  +\xi\{\|Z_a\|-\ell_a\}
  -\frac{\tau}{2}\xi^2
 .
\end{aligned}
\]
If the radius constraint \(\xi\leqslant R\) is removed, the maximization over
\(\xi\geqslant0\) equals
\(
  \frac{1}{2\tau}(\|Z_a\|-\ell_a)_+^2,
\)
which is the source of the positive-part square in
\eq{ridge-cgmt-main}.

\emph{Step 3: deterministic limit of the auxiliary problem.}
Uniformly over \(a\in[\varepsilon,M]^n\),
\[
\begin{aligned}
  \|Z_a\|^2
  &=
  \left\|\mu+\frac{\|a\|}{n}g\right\|^2   =
  \|\mu\|^2
  +2\frac{\|a\|}{n}\mu^\top g
  +\frac{\|a\|^2}{n^2}\|g\|^2.
\end{aligned}
\]
Conditionally on \(\mu\), \(\mu^\top g=O_{\mathbb P}(1)\), while
\(\|a\|/n\leqslant M n^{-1/2}\).  Hence the cross term is
\(o_{\mathbb P}(1)\), uniformly over the compact set.  Also,
\[
  \frac{\|a\|^2}{n^2}\|g\|^2
  =
  \bigg(\frac1n\sum_{j=1}^n a_j^2\bigg)
  \bigg(\frac{m}{n}\frac{\|g\|^2}{m}\bigg)
  =
  \alpha_q\bigg(\frac1n\sum_{j=1}^n a_j^2\bigg)+o_{\mathbb P}(1),
\]
uniformly over \(a\in[\varepsilon,M]^n\).  Thus the vector \(g\) self-averages
into the scalar radius
\[
  B_n(a)=
  \bigg(
    \|\mu\|^2+
    \alpha_q\frac1n\sum_{j=1}^n a_j^2
  \bigg)^{1/2}.
\]
The vector \(h\) remains in the coordinatewise average
\(n^{-1}\sum_j a_jH_j\).  Passing to empirical occupation measures
\(n^{-1}\sum_j\delta_{(H_j,a_j)}\), every subsequential limit has first
marginal \(\Normal(0,1)\).  Conditional randomization of the second
coordinate given \(H\) cannot improve the infimum: replacing the second
coordinate by its conditional mean preserves \(\E[A]\) and \(\E[AH]\), and
Jensen's inequality decreases both \(\E[h(A)]\) and \(\E[A^2]\).  Therefore
it is enough in the limit to optimize over deterministic measurable selectors
\(A=A(H)\), where \(H\sim\Normal(0,1)\).

For such a selector define
\[
  B(A)=\left(r^2+\alpha\E[A^2]\right)^{1/2}.
\]
The compact deterministic auxiliary limit is
\[
\begin{aligned}
\psi_{\varepsilon,M,R}(r)
=
\inf_{\substack{A(H)\in[\varepsilon,M]\\ \E[A]=1}}
\sup_{0\leqslant \xi\leqslant R}
  \E[h(A)]
  +\xi\{B(A)-\E[AH]\}
  -\frac{\tau}{2}\xi^2.
\end{aligned}
\]
The preceding uniform approximation and the standard sample-average
/ epi-convergence argument give
\[
  \widetilde\varphi_{n,\tau}^{\varepsilon,M,R}(\mu)
  \xrightarrow{\mathbb P}
  \psi_{\varepsilon,M,R}(r).
\]
The compact CGMT~\cite{thrampoulidis2014gaussian} transfers the same limit to the primary problem
\(
  \varphi_{n,\tau}^{\varepsilon,M,R}(\mu)
  \xrightarrow{\mathbb P}
  \psi_{\varepsilon,M,R}(r).
\)

It remains to remove the compact restrictions.  The feasible point
\(a_j\equiv1\) gives an \(O_{\mathbb P}(1)\) upper bound on the untruncated
objective.  Hence every near minimizer has bounded empirical entropy and
bounded residual
\(\|\mu-n^{-1}Y^\top a\|\).  The entropy bound gives uniform integrability
of the weights, and the residual bound gives tightness of the associated
dual vector \(u=\tau^{-1}(\mu-n^{-1}Y^\top a)\).  Truncating the weights to
\([\varepsilon,M]\), renormalizing them to have empirical mean one, and then
letting \(\varepsilon\downarrow0\), \(M\to\infty\) gives the same epigraphical
limit; finally let \(R\to\infty\).  Denote the corresponding untruncated
value by
\[
\begin{aligned}
\varphi_{n,\tau}(\mu)
=
\min_{\substack{a\in\R_+^n\\ n^{-1}\sum_{j=1}^n a_j=1}}
  \frac1n\sum_{j=1}^n h(a_j)
  +\frac{1}{2\tau}\left\|\mu-\frac1nY^\top a\right\|^2
 .
\end{aligned}
\]
Consequently, \(\varphi_{n,\tau}(\mu)\) converges in probability to
\[
\begin{aligned}
\psi_\tau(r)
&=
\inf_{\substack{A\geqslant0\\ \E[A]=1}}
  \E[h(A)]
  +\frac{1}{2\tau}(B(A)-\E[AH])_+^2                                                 \\
&=
\inf_{\substack{A\geqslant0\\ \E[A]=1}}
  \E[A\log A-A+1]
  +\frac{1}{2\tau}
    \left(
      \left(r^2+\alpha\E[A^2]\right)^{1/2}-\E[AH]
    \right)_+^2
,
\end{aligned}
\]
which is \eq{ridge-cgmt-main} after substituting
\(r^2=s+\alpha_p\) and \(\alpha=\alpha_q\).

\emph{Step 4: uniqueness and KKT equations.}
The feasible set
\(\{A\geqslant0:\E[A]=1\}\) is convex.  The entropy term \(\E[h(A)]\) is strictly
convex, while
\[
  A\longmapsto
  \left(
    \left(r^2+\alpha\E[A^2]\right)^{1/2}-\E[AH]
  \right)_+^2
\]
is convex: \(A\mapsto (r^2+\alpha\E[A^2])^{1/2}\) is a norm-type convex
functional, subtracting the linear term \(\E[AH]\) preserves convexity,
positive-part preserves convexity, and squaring preserves convexity on
\([0,\infty)\).  Entropy sublevel sets are uniformly integrable, and the
remaining terms are lower semicontinuous under the induced weak convergence.
Thus the limiting problem has a unique optimizer, denoted \(A\).
Moreover \(A>0\) almost surely: if \(A=0\) on a set of positive Gaussian
measure, then increasing \(A\) slightly on a bounded subset of that set and
compensating the mass elsewhere gives a first-order entropy decrease with
slope \(-\infty\), whereas the penalty term has finite one-sided variation.

Since \(s+\alpha_p+\alpha_q>0\), the optimizer is nondegenerate and
\(B=B(A)>0\).  The positive-part term is strictly active.  Otherwise the
penalty has zero derivative, and the stationarity of the entropy under the
mass constraint would force \(A\) to be constant; the constraint \(\E[A]=1\)
then gives \(A\equiv1\), for which
\(B(A)-\E[AH]=(r^2+\alpha)^{1/2}>0\), a contradiction.  Therefore
\[
  \xi=\frac{B-\E[AH]}{\tau}>0,
  \qquad
  \E[AH]=B-\tau\xi.
\]
For variations \(\delta A\) preserving integrability,
\(\delta B(A)=\frac{\alpha}{B}\E[A\,\delta A]\). Introducing a multiplier
\(\lambda\) for the constraint \(\E[A]=1\), the first variation of
\(\E[h(A)]+\frac{1}{2\tau}(B(A)-\E[AH])^2+\lambda(\E[A]-1)\) is
\(\E[(\log A+\xi(\frac{\alpha}{B}A-H)+\lambda)\delta A]\).
Absorbing \(-\lambda\) into \(\eta\), stationarity gives
\[
  \log A+\frac{\xi\alpha}{B}A=
  \eta+\xi H
  \qquad\text{almost surely}.
\]
Together with the mass constraint, the definition of \(B\), and the active
relation above, this yields
\[
  \E[A]=1,
  \qquad
  r^2+\alpha\E[A^2]=B^2,
  \qquad
  \E[AH]=B-\tau\xi.
\]
After substituting \(r^2=s+\alpha_p\) and \(\alpha=\alpha_q\), these are
Eqs.~\eqref{eq:ridge-mass-main}-\eqref{eq:ridge-alignment-main}.  For fixed
\((\xi,B,\eta)\), the map
\(a\mapsto \log a+(\xi\alpha/B)a\) is strictly increasing on \((0,\infty)\),
so the scalar KKT equation determines \(A(H)\) uniquely.

\emph{Step 5: optimizer observables.}
To identify projections of the primary optimizer, we use the optimizer-localization
consequence of the CGMT~\cite{ThrampoulidisAbbasiHassibi2018} through a perturbation argument.
Fix a deterministic sequence \(v\in\R^m\) with bounded norm and add the linear
perturbation \(\lambda v^\top u\) to the compact primary problem, equivalently replacing
\(\mu\) by \(\mu+\lambda v\).  The compact CGMT and the same truncation-removal
argument apply locally uniformly for \(\lambda\) in a neighborhood of zero.  Hence the
perturbed primary values converge locally uniformly, in probability, to the corresponding
deterministic perturbed value.  Since the finite-\(n\) value is differentiable at
\(\lambda=0\), with derivative \(\hat\theta^\top v\), and since the limiting
variational problem has a unique optimizer, Danskin's envelope theorem gives the derivative
of the limiting value.  The standard convergence theorem for derivatives of locally uniformly
convergent convex functions then yields convergence of the optimizer observable.

The limiting derivative is obtained by differentiating
\[
  B_\lambda(A)=
  \left(\|\mu+\lambda v\|^2+
        \alpha\E[A^2]\right)^{1/2}
\]
at the optimizer, hence equals \(\xi\frac{\mu^\top v}{B}\). Taking
\(v=\mu\) and \(v=\Delta\), and using \(\|\mu\|^2\to r^2\) and
\(\Delta^\top\mu\to s\), gives
\[
  \hat\theta^\top\mu\to \frac{\xi r^2}{B},
  \qquad
  \hat\theta^\top\Delta\to \frac{\xi s}{B}.
\]
Similarly, differentiating the finite-\(n\) value with respect to \(\tau\)
gives \(-\|\hat\theta\|^2/2\), while differentiating the deterministic limit
with respect to \(\tau\) gives \(-\xi^2/2\). Hence
\(\|\hat\theta\|^2\to\xi^2,\hat\theta^\top\Delta\to\frac{\xi s}{B}\),
which proves \eqref{eq:ridge-var-projections-main}.

The entropy of the fitted dual weights is also determined by the value
convergence.  For the optimizer \(a\),
\[
  \varphi_{n,\tau}(\mu)
  =
  \frac1n\sum_{j=1}^n h(a_j)
  +\frac{\tau}{2}\|\hat\theta\|^2,
\]
because
\(\mu-n^{-1}Y^\top a=\tau\hat\theta\).  The deterministic value is
\[
  \psi_\tau(r)=\E[h(A)]+\frac{\tau}{2}\xi^2,
\]
since \(B-\E[AH]=\tau\xi\).  Therefore
\(
  \frac1n\sum_{j=1}^n h(a_j)
  \to \E[h(A)].
\)
The finite-\(n\) and limiting mass constraints give
\(
  \frac1n\sum_{j=1}^n a_j\log a_j
  \to
  \E[A\log A].
\)

It remains to identify the empirically normalized intercept.  At the saddle
point, the entropy-dual weights are the empirical exponential weights,
\[
  a_j=
  \exp\{y_j^\top\hat\theta+\hat c\},
  \qquad
  \hat c=-\log\bigg(\frac1n\sum_{j=1}^n
      e^{y_j^\top\hat\theta}\bigg),
\]
and \(n^{-1}\sum_ja_j=1\).  Therefore
\[
\begin{aligned}
  \frac1n\sum_{j=1}^n a_j\log a_j
  &=
  \hat\theta^\top\Big(\frac1nY^\top a\Big)
  +\hat c\bigg(\frac1n\sum_{j=1}^na_j\bigg)  =
  \hat\theta^\top(\mu-\tau\hat\theta)+\hat c.
\end{aligned}
\]
Thus
\(
  \hat c=
  \frac1n\sum_{j=1}^n a_j\log a_j
  -\hat\theta^\top\mu
  +\tau\|\hat\theta\|^2.
\)
Using the limits just proved,
\[
  \hat c
  \to
  \E[A\log A]+\tau\xi^2-\frac{\xi r^2}{B}
  =
  \E[A\log A]+\tau\xi^2-\frac{\xi(s+\alpha_p)}{B},
\]
which proves \eq{c-tau-limit}.

For the fitted affine potential
\(\hat v_{\rm var}(x)=\hat\theta^\top x+\hat c\), the population
\(q\)-log-normalizer satisfies
\[
  \log Z_q(\hat v_{\rm var})
  =\hat c+\frac12\|\hat\theta\|^2
  \to
  z:=
  \E[A\log A]+\tau\xi^2-\frac{\xi(s+\alpha_p)}{B}
  +\frac12\xi^2.
\]
By Proposition~\ref{prop:affine-scores},
\[
  J(\hat v_{\rm var})
  =\hat\theta^\top\Delta-\frac12\|\hat\theta\|^2
  \to
  J_{\rm var}:=\frac{\xi s}{B}-\frac12\xi^2.
\]
Using the identity
\(K(v)=J(v)+\log Z_q(v)-Z_q(v)+1\), we also get
\(
  K(\hat v_{\rm var})
  \to
  K_{\rm var}:=J_{\rm var}+z-e^z+1.
\)
Equivalently, the excess risks satisfy
\(
  \mathcal R^{\rm var}_{\J}
  \to
  \frac{s}{2}-J_{\rm var}
  =
  \frac12\left(s+\xi^2-2\frac{\xi s}{B}\right),
\)
and
\(
  \mathcal R^{\rm var}_{\K}
  \to
  \mathcal R^{\rm var}_{\J}+e^z-z-1.
\)
This proves the stated score limits, and separately records the corresponding
risk limits.
\end{proof}

 \section{Feasibility of the zero-ridge variational limit (\mysec{zero-ridge-var})}
 \label{app:Rhull}

Let $H \sim \mathcal{N}(0,1)$.
Define
\[
\varphi(z)=\frac{1}{\sqrt{2\pi}}e^{-z^2/2},
\qquad
\bar\Phi(z)=\mathbb P(H\geqslant z)=1-\Phi(z).
\]
For \(\alpha\geqslant 0\), set
\begin{equation}\label{eq:appE-rhull-def}
R_{\rm hull}^2(\alpha)
=
\sup_{A\geqslant 0,\ \mathbb E[A]=1}
\bigl(\mathbb E[AH]\bigr)^2-\alpha \mathbb E[A^2]
.
\end{equation}

We compute this quantity explicitly. First, if \(\alpha=0\), then
\(
R_{\rm hull}^2(0)=+\infty.
\)
Indeed, with
\(
A_t(H)=\frac{  1_{\{H\geqslant t\}}}{\bar\Phi(t)},
\)
we have \(\mathbb E[A_t]=1\) and
\(
\mathbb E[A_tH]= {\varphi(t)}/{\bar\Phi(t)}\sim t\) tends to infinity when 
\(t\to+\infty,
\)
so the objective in \eq{appE-rhull-def} is unbounded.

Assume now that \(\alpha>0\). Using
\(
x^2=\sup_{u\in\mathbb R}\{2ux-u^2\},
\)
we may write
\begin{equation}\label{eq:appE-rhull-dual}
R_{\rm hull}^2(\alpha)
=
\sup_{u\in\mathbb R}
\Big\{
-u^2+
\sup_{A\geqslant0,\ \mathbb E[A]=1}
\mathbb E\bigl[2uHA-\alpha A^2\bigr]
\Big\}.
\end{equation}
By symmetry of \(H\), it is enough to consider \(u\geqslant 0\). For fixed \(u>0\), the inner problem is concave in \(A\). Its KKT conditions give
\[
A(H)=\frac{u}{\alpha}(H-z)_+
\]
for some threshold \(z\in\mathbb R\), where \((x)_+=\max\{x,0\}\). The normalization constraint gives
\(
1=\frac{u}{\alpha}\mathbb E[(H-z)_+].
\)
Define
\(
d(z)=\mathbb E[(H-z)_+].
\)
A direct Gaussian integration gives
\[
d(z)
=
\int_z^\infty (h-z)\varphi(h)\,dh
=
\varphi(z)-z\bar\Phi(z).
\]
Thus
\begin{equation}\label{eq:appE-Az}
u=\frac{\alpha}{d(z)},
\qquad
A_z(H)=\frac{(H-z)_+}{d(z)}.
\end{equation}
Substituting \eq{appE-Az} into \eq{appE-rhull-dual}, the problem reduces to the one-dimensional maximization
\[
R_{\rm hull}^2(\alpha)
=
\sup_{z\in\mathbb R} G_\alpha(z),
\]
where
\begin{equation}\label{eq:appE-G-alpha}
G_\alpha(z)
=
\frac{\alpha\{\bar\Phi(z)-\alpha+z d(z)\}}{d(z)^2}.
\end{equation}
The endpoint \(u=0\) corresponds to \(z=-\infty\), and is included by taking the limit.

We now maximize \(G_\alpha\). Since
\(
\bar\Phi'(z)=-\varphi(z),
d'(z)=-\bar\Phi(z),
\)
differentiating \eq{appE-G-alpha} gives
\begin{equation}\label{eq:appE-G-alpha-derivative}
G_\alpha'(z)
=
\frac{2\alpha \bar\Phi(z)\{\bar\Phi(z)-\alpha\}}{d(z)^3}.
\end{equation}
Because \(d(z)>0\), the sign of \(G_\alpha'(z)\) is the sign of
\(\bar\Phi(z)-\alpha\).

If \(0<\alpha<1\), there is a unique \(z_\alpha\in\mathbb R\) such that
\[
\bar\Phi(z_\alpha)=\alpha,
\qquad
z_\alpha=\Phi^{-1}(1-\alpha).
\]
\eq{appE-G-alpha-derivative} shows that this point is the global maximizer. At this maximizer,
\(
d(z_\alpha)=\varphi(z_\alpha)-\alpha z_\alpha,
\)
and therefore
\[
R_{\rm hull}^2(\alpha)
=
\frac{\alpha z_\alpha}{\varphi(z_\alpha)-\alpha z_\alpha},
\qquad
0<\alpha<1.
\]

If \(\alpha\geqslant 1\), then \(R_{\rm hull}^2(\alpha)=-\alpha\). Indeed, for any admissible \(A\), since \(\mathbb E[A]=1\), \(\mathbb E[AH]=\mathbb E[(A-1)H]\), and by Cauchy--Schwarz,
\[
\bigl(\mathbb E[AH]\bigr)^2
\leqslant
\mathbb E[(A-1)^2]\mathbb E[H^2]
=
\mathbb E[A^2]-1.
\]
Hence, for \(\alpha\geqslant 1\),
\[
\bigl(\mathbb E[AH]\bigr)^2-\alpha\mathbb E[A^2]
\leqslant(\mathbb E[A^2]-1)-\alpha\mathbb E[A^2]
=
(1-\alpha)\mathbb E[A^2]-1
\leqslant
-\alpha,
\]
because \(\mathbb E[A^2]\geqslant(\mathbb E[A])^2=1\). Equality is attained by \(A\equiv1\).

Combining the cases, for all \(\alpha\geqslant 0\),
\[
R_{\rm hull}^2(\alpha)
=
\begin{cases}
+\infty,
& \alpha=0,\\[0.4em]
\dfrac{\alpha z_\alpha}{\varphi(z_\alpha)-\alpha z_\alpha},
\quad z_\alpha=\Phi^{-1}(1-\alpha),
& 0<\alpha<1,\\[1.0em]
-\alpha,
& \alpha\geqslant 1.
\end{cases}
\]

In particular, we get
\[
R_{\rm hull}^2(\alpha)>0
\quad\Longleftrightarrow\quad
0\leqslant \alpha<\frac12,
\]
with the convention \(R_{\rm hull}^2(0)=+\infty\). Also,
\(
R_{\rm hull}^2\!\left(\frac12\right)=0,
\)
because \(z_{1/2}=0\). For \(\alpha>1/2\), the quantity
\(R_{\rm hull}^2(\alpha)\) is negative, so there is no positive-radius strict feasible phase.

Finally, as \(\alpha\downarrow0\), \(z_\alpha\to+\infty\). By Mills' expansion,
\(
\bar\Phi(z_\alpha)
=
\varphi(z_\alpha)
\left(
\frac{1}{z_\alpha}+O(z_\alpha^{-3})
\right)
=
\alpha,
\)
and hence
\(
\varphi(z_\alpha)-\alpha z_\alpha
=
\frac{\alpha}{z_\alpha}\{1+o(1)\}.
\)
Therefore
\(
R_{\rm hull}^2(\alpha)
=
z_\alpha^2\{1+o(1)\}
\sim
2\log(1/\alpha),\) when \( \alpha\downarrow0.
\)

\section{Two-resolvent random matrix theory result}

\label{app:newprop}

\paragraph{Two-resolvent input.}
Proposition~3.2 of \cite{Bach2024RandomProjections} gives deterministic
 equivalents for one resolvent associated with a single covariance matrix.
The spectral-risk calculation in Theorem~\ref{thm:spec-ridge} also requires the
same leave-one-out calculation for two covariance profiles evaluated on the same
Gaussian matrix.  The proposition below states only the two consequences needed
there: the trace of a product of two kernel resolvents and the corresponding
kernel expectation with one deterministic diagonal matrix.  The notation follows
\cite{Bach2024RandomProjections}: for a covariance matrix \(\Sigma\),
\(\df_1(\kappa)=\tr(\Sigma(\Sigma+\kappa I)^{-1})\), and
\(\kappa(\lambda)=1/\varphi(-\lambda)\) is the self-induced regularization
parameter.

\begin{proposition}[Two-covariance version of Proposition~3.2 of \cite{Bach2024RandomProjections}]
\label{prop:two-covariance-prop32}
Let \(Z\in\R^{n\times d}\) have i.i.d. standard Gaussian entries, with
\(d/n\to\gamma\in(0,\infty)\). Let \(\Sigma_1,\Sigma_2\in\R^{d\times d}\) be
deterministic nonnegative symmetric matrices, uniformly bounded in operator
norm, and diagonal in the same orthonormal basis:
\[
        \Sigma_1=\diag(\sigma_{1,1},\ldots,\sigma_{1,d}),
        \qquad
        \Sigma_2=\diag(\sigma_{2,1},\ldots,\sigma_{2,d}).
\]
Assume that their joint empirical spectral distribution has a compactly
supported limit. Let \(A=\diag(a_1,\ldots,a_d)\) be deterministic with
\(\|A\|_{\op}=O(1)\), and assume that the normalized weighted traces appearing
below have limits.

Fix \(\lambda>0\). Let \(\kappa_1=\kappa_1(\lambda)\) and
\(\kappa_2=\kappa_2(\lambda)\) be the self-induced regularization parameters
associated with \(\Sigma_1\) and \(\Sigma_2\), respectively, defined by
\begin{align}
        \lambda
        &=
        \kappa_1
        \Big(1-\frac1n
        \tr\bigl[\Sigma_1(\Sigma_1+\kappa_1 I)^{-1}\bigr]\Big)
        \label{eq:newrmt-kappa1}   =
        \kappa_2
        \Big(1-\frac1n
        \tr\bigl[\Sigma_2(\Sigma_2+\kappa_2 I)^{-1}\bigr]\Big).
\end{align}
Define the cross second degrees of freedom by
\begin{equation*}
        \df_{12}(\kappa_1,\kappa_2)
        =
        \tr\!\left[
        \Sigma_1(\Sigma_1+\kappa_1 I)^{-1}
        \Sigma_2(\Sigma_2+\kappa_2 I)^{-1}
        \right].
\end{equation*}
Set
\begin{equation*}
        R_1=(Z\Sigma_1Z^\top+n\lambda I_n)^{-1},
        \qquad
        R_2=(Z\Sigma_2Z^\top+n\lambda I_n)^{-1}.
\end{equation*}
Then, almost surely,
\begin{equation}
        n\tr(R_1R_2)
        \sim
        \frac1{
        \kappa_1\kappa_2
        \left(1-\frac1n\df_{12}(\kappa_1,\kappa_2)\right)}
        \label{eq:newrmt-two-resolvent-trace}
\end{equation}
and
\begin{equation}
        \tr\!\left[A Z^\top R_1R_2 Z\right]
        \sim
        \frac{
        \frac1n\tr\!\left[
        A(\Sigma_1+\kappa_1 I)^{-1}
        (\Sigma_2+\kappa_2 I)^{-1}
        \right]
        }{
        1-\frac1n\df_{12}(\kappa_1,\kappa_2)
        }.
        \label{eq:newrmt-two-resolvent-kernel}
\end{equation}
The same equivalent holds with \(R_1R_2\) replaced by \(R_2R_1\).
\end{proposition}

\begin{proof}
Work in the common eigenbasis of \(\Sigma_1\) and \(\Sigma_2\), and write
\(z_i\in\R^n\) for the \(i\)-th column of \(Z\). Define the leave-one-column
resolvents
\begin{equation*}
        R_1^{(i)}=
        \bigg(\sum_{j\ne i}\sigma_{1,j}z_jz_j^\top+n\lambda I_n\bigg)^{-1},
        \qquad
        R_2^{(i)}=
        \bigg(\sum_{j\ne i}\sigma_{2,j}z_jz_j^\top+n\lambda I_n\bigg)^{-1}.
\end{equation*}
The Sherman-Morrison formula gives
\begin{equation}
        R_1z_i=
        \frac{R_1^{(i)}z_i}{1+\sigma_{1,i}z_i^\top R_1^{(i)}z_i},
        \qquad
        R_2z_i=
        \frac{R_2^{(i)}z_i}{1+\sigma_{2,i}z_i^\top R_2^{(i)}z_i}.
        \label{eq:newrmt-sherman-morrison}
\end{equation}
The one-resolvent argument of Proposition~3.2 of
\cite{Bach2024RandomProjections}, applied separately to \(\Sigma_1\) and
\(\Sigma_2\), gives
\begin{equation}
        \tr R_1\sim \frac1{\kappa_1},
        \qquad
        \tr R_2\sim \frac1{\kappa_2}.
        \label{eq:newrmt-one-resolvent-trace}
\end{equation}
Consequently, Gaussian quadratic-form concentration and the usual rank-one
resolvent comparison imply
\begin{equation}
        z_i^\top R_1^{(i)}z_i=\frac1{\kappa_1}+o(1),
        \qquad
        z_i^\top R_2^{(i)}z_i=\frac1{\kappa_2}+o(1),
        \label{eq:newrmt-one-resolvent-qf}
\end{equation}
and
\begin{equation}
        z_i^\top R_1^{(i)}R_2^{(i)}z_i
        =
        \tr(R_1R_2)+o(n^{-1}).
        \label{eq:newrmt-two-resolvent-qf}
\end{equation}
Combining Eqs.~\eqref{eq:newrmt-sherman-morrison}-\eqref{eq:newrmt-two-resolvent-qf}
 yields
\begin{equation}
        z_i^\top R_1R_2z_i
        =
        \frac{\tr(R_1R_2)}{
        (1+\sigma_{1,i}/\kappa_1)(1+\sigma_{2,i}/\kappa_2)}
        +o(n^{-1}).
        \label{eq:newrmt-column-qf}
\end{equation}
The same scalar identity holds with \(R_1R_2\) replaced by \(R_2R_1\), because
\(z_i^\top R_1R_2z_i=z_i^\top R_2R_1z_i\).

Let
\begin{equation*}
        D_1=\left(I+\kappa_1^{-1}\Sigma_1\right)^{-1},
        \qquad
        D_2=\left(I+\kappa_2^{-1}\Sigma_2\right)^{-1}.
\end{equation*}
Multiplying \eq{newrmt-column-qf} by \(a_i\) and summing over
\(i=1,\ldots,d\) gives
\begin{equation}
        \tr\!\left[A Z^\top R_1R_2Z\right]
        \sim
        n\tr(R_1R_2)\,\frac1n\tr[AD_1D_2].
        \label{eq:newrmt-kernel-from-trace}
\end{equation}
It remains to identify the scalar \(n\tr(R_1R_2)\).
Use the identity
\[
        R_1(Z\Sigma_1Z^\top+n\lambda I_n)R_2=R_2.
\]
Taking traces and using cyclicity,
\begin{equation}
        \tr R_2
        =
        \tr\!\left[\Sigma_1Z^\top R_2R_1Z\right]
        +n\lambda\tr(R_1R_2).
        \label{eq:newrmt-resolvent-identity-trace}
\end{equation}
Apply \eq{newrmt-kernel-from-trace}, with the reversed order and
\(A=\Sigma_1\), to the first term in
\eqref{eq:newrmt-resolvent-identity-trace}. If \(\tau\) is a subsequential limit
of \(n\tr(R_1R_2)\), then \eq{newrmt-one-resolvent-trace} gives
\begin{equation}
        \frac1{\kappa_2}
        =
        \tau
        \left(
        \lambda+\frac1n\tr[\Sigma_1D_1D_2]
        \right).
        \label{eq:newrmt-tau-equation}
\end{equation}
Since \(D_2=I-\Sigma_2(\Sigma_2+\kappa_2I)^{-1}\),
\begin{align*}
        \frac1n\tr[\Sigma_1D_1D_2]
        & =
        \kappa_1\Big(
        \frac1n\tr[\Sigma_1(\Sigma_1+\kappa_1 I)^{-1}]
        -\frac1n\df_{12}(\kappa_1,\kappa_2)
        \Big).
\end{align*}
Together with \eq{newrmt-kappa1}, this gives
\begin{equation*}
        \lambda+\frac1n\tr[\Sigma_1D_1D_2]
        =
        \kappa_1
        \Big(1-\frac1n\df_{12}(\kappa_1,\kappa_2)\Big).
\end{equation*}
Substitution in \eq{newrmt-tau-equation} identifies the unique subsequential
limit,
$
        \tau=
        \frac1{
        \kappa_1\kappa_2
        \left\{1-\frac1n\df_{12}(\kappa_1,\kappa_2)\right\}},
$
which proves \eq{newrmt-two-resolvent-trace}.

Finally, substituting \eq{newrmt-two-resolvent-trace} into
\eq{newrmt-kernel-from-trace}, and using
$
        D_1D_2
        =
        \kappa_1\kappa_2
        (\Sigma_1+\kappa_1 I)^{-1}
        (\Sigma_2+\kappa_2 I)^{-1},
$
gives \eq{newrmt-two-resolvent-kernel}. The reversed-order statement
follows from the same argument and from \(\tr(R_2R_1)=\tr(R_1R_2)\).
\end{proof}

 \section{Quadrature for the spectral deterministic equivalent}
\label{sec:quadas}
This appendix gives the finite-dimensional quadrature version of the continuum
spectral formulas in \mysec{chi-quadrature}.  Let
\((\rho_a,w_a)_{a=1}^{\Nquad}\) be a deterministic quadrature rule on \((0,1)\),
with \(w_a>0\).  At each node \(\rho_a\), solve the scalar effective-ridge
equation in \eq{omega-zeta-hd} and evaluate the quantities in
Eqs.~\eqref{eq:df2-zeta-hd}-(\ref{eq:k-zeta-hd}).  We write
\[
 \omega_a=\omega(\rho_a),\qquad
 \kappa_a=\kappa(\rho_a),\qquad
 \ell_a=\ell_\Delta(\rho_a),\qquad
 m_a=m(\rho_a),\qquad
 k_{ab}=k(\rho_a,\rho_b).
\]
Equivalently,
\[
 \ell_a=\frac{s}{\omega_a\kappa_a},
 \qquad
 m_a=
 \frac{\rho_a(s+\alpha_p)-(1-\rho_a)\alpha_q}{\omega_a\kappa_a},
 \qquad
 k_{ab}=
 \frac{(s+\alpha_p+\alpha_q)\chi_{ab}}{\kappa_a\kappa_b},
\]
where
\[
 \chi_{ab}
 =
 \frac{1}{\omega_a\omega_b(1-\df_2(\rho_a,\rho_b))}.
\]
Set
\[
 d_a=2\rho_a(1-\rho_a),
 \qquad
 b_a=2(1-\rho_a)(1+\rho_a m_a).
\]

The scalar integral terms entering Eqs.~\eqref{eq:Lambda-spec-cont}-(\ref{eq:K-spec-cont})
are replaced by
\begin{align*}
 \tr(A_v)^{(\Nquad)}
 &=
 -\frac12\sum_{a=1}^{\Nquad} w_a d_a k_{aa},
 \\
 (\Delta^\top A_v\Delta)^{(\Nquad)}
 &=
 -\frac12\sum_{a=1}^{\Nquad} w_a d_a \ell_a^2,
 \\
 (\ell_v^\top\Delta)^{(\Nquad)}
 &=
 \sum_{a=1}^{\Nquad} w_a b_a\ell_a,
 \\
 c_v^{(\Nquad)}
 &=
 \sum_{a=1}^{\Nquad}
 2w_a(1-\rho_a)
 \left\{
   -m_a-\frac{\rho_a}{2}m_a^2
 \right\}.
\end{align*}

For the Fredholm determinant and inverse-quadratic term in
\eq{Lambda-spec-cont} and \eq{J-spec-cont}, define
\[
 W=\Diag(w_1,\ldots,w_{\Nquad}),
 \qquad
 D=\Diag(d_1,\ldots,d_{\Nquad}),
 \qquad
 K=(k_{ab})_{a,b=1}^{\Nquad},
\]
and
\[
 K_W=W^{1/2}KW^{1/2},
 \qquad
 b_W=W^{1/2}(b_1,\ldots,b_{\Nquad})^\top .
\]
The Nystr\"om quadrature approximation of the Fredholm determinant
\cite{Bornemann2010Fredholm} is
\[
 \{\log\det(I+\mathcal D\mathcal T)\}^{(\Nquad)}
 =
 \log\det\!\left(I_{\Nquad}+D^{1/2}K_WD^{1/2}\right),
\]
equivalently \(\log\det(I_{\Nquad}+DK_W)\).  The inverse-quadratic term is
replaced by
\[
 \left\{
 \left\langle b,(I+\mathcal T\mathcal D)^{-1}\mathcal T b
 \right\rangle_{L_2([0,1])}
 \right\}^{(\Nquad)}
 =
 b_W^\top (I_{\Nquad}+K_WD)^{-1}K_W b_W .
\]
Substituting these replacements into
Eqs.~\eqref{eq:Lambda-spec-cont}-(\ref{eq:K-spec-cont}) gives
\(\Lambda_{\rm spec}^{(\Nquad)}\),
\(\mathcal J_{\rm spec}^{(\Nquad)}\), and
\(\mathcal K_{\rm spec}^{(\Nquad)}\).

For the two-potential criterion \(\mathcal L\), the additional terms in
\eq{L-spec-cont} are replaced by
\begin{align*}
 \tr(A_w)^{(\Nquad)}
 &=
 -\sum_{a=1}^{\Nquad} w_a(1-\rho_a)^2 k_{aa},
 \\
 c_w^{(\Nquad)}
 &=
 \sum_{a=1}^{\Nquad}
 2w_a(1-\rho_a)
 \left(
   m_a-\frac{1-\rho_a}{2}m_a^2
 \right).
\end{align*}
Then \(\mathcal L_{\rm spec}^{(\Nquad)}\) is obtained from
\eq{L-spec-cont} by replacing
\(\tr(A_v)\), \(\Delta^\top A_v\Delta\), \(\ell_v^\top\Delta\),
\(c_v\), \(\tr(A_w)\), and \(c_w\) by their quadrature versions above.

Thus the numerical evaluation of the regularized spectral deterministic
equivalent consists only of solving \eq{omega-zeta-hd} at the quadrature nodes
and forming the displayed sums and \(\Nquad\times\Nquad\) matrix expressions;
no optimization remains.  Gauss--Legendre quadrature is used in the experiments.
Increasing \(\Nquad\) refines the deterministic numerical approximation to the
continuum integral; when \(\alpha_p\) or \(\alpha_q\) is close to one, using more
nodes near the corresponding endpoint may improve numerical accuracy.

\section{Nuclear penalty}
\label{app:nuclear}

Section~\ref{sec:est2}, in particular Eq.~(\ref{eq:beta-ridge-spectral}), regularizes each least-squares task indexed by
$\rho\in[0,1]$ separately.  We consider here a coupled alternative that
retains the separable ridge penalty and adds a squared nuclear penalty
across the entire coefficient field.  The latter promotes a common
low-dimensional feature subspace for the continuum of tasks, as in convex
multi-task feature learning \cite{ArgyriouEvgeniouPontil2008}.

Let $\nu$ be a finite positive measure on $[0,1]$.  The KL construction in
the main text corresponds to
\(
    \dd\nu(\rho)=2(1-\rho)\,\dd\rho.
\)
For $\beta\in L_2([0,1],\nu;\R^m)$, define the finite-rank operator
\[
    B:L_2([0,1],\nu)\to\R^m,
    \qquad
    B f=\int_0^1\beta(\rho)f(\rho)\,\dd\nu(\rho).
\]
Its adjoint is
\[
    (B^*x)(\rho)=\beta(\rho)^\top x,
    \qquad x\in\R^m.
\]
 
Using $\hDelta$ and the empirical quantities $\hmu(\rho)$ and $\hM(\rho)$
defined in Eq.~(\ref{eq:murho}) of Section~\ref{sec:est2}, the coupled training criterion is
\[
\begin{aligned}
    \hat\Phi (\beta)
    ={}&\int_0^1\left[
      \hDelta^\top\beta(\rho)
      -\frac12\beta(\rho)^\top\hM(\rho)\beta(\rho)
      -\frac\zeta2\norm{\beta(\rho)}^2
      \right]\dd\nu(\rho) -\frac\lambda2\norm{B}_*^2,
\end{aligned}
\]
where $\zeta\geqslant0$ and $\lambda>0$.  As in
Eq.~(\ref{eq:vhat-spec-ridge}) and Section~\ref{sec:est2}, the fitted affine task and the
two spectral potentials remain
$
    \hat u(\rho,x)
    =\hat\beta(\rho)^\top(x-\hmu(\rho))
$, with
\begin{align*}
    \hat v(x)
    &=\int_0^1\left(
      \hat u(\rho,x)-\frac\rho2\hat u(\rho,x)^2
      \right)\dd\nu(\rho),\\
    \hat w(y)
    &=\int_0^1\left(
      -\hat u(\rho,y)-\frac{1-\rho}{2}\hat u(\rho,y)^2
      \right)\dd\nu(\rho).
\end{align*}
Thus only the manner in which the coefficient curve is fitted changes; the
population criteria $\mathcal J$, $\mathcal K$, and~$\mathcal L$ remain
those introduced in Section~\ref{sec:finite}, with the scoring identities of
Section~\ref{sec:scoring}. We still work in the proportional regime of Assumption~\ref{ass:hd},
$
     {m}/{n_p}\to\alpha_p$,
 $
     {m}/{n_q}\to\alpha_q$.
    
    The next subsection gives an exact continuum formulation for finite
samples.  We then derive the operator-valued matrix-Dyson equation, state
the additional local-law and compactness hypotheses required for a genuine
continuum deterministic equivalent, obtain the linearized equation needed
for scoring, and finish with population checks and numerical algorithms.
 Note that our deterministic equivalent is only proved for
finite discrete measures, and we conjecture that the same result holds for generic
probability distributions $\nu$.

\subsection{Operator formulation and matrix-Dyson deterministic equivalent}
\label{appH:sec:operator-central}

The main text already gives the taskwise ridge estimator, its scalar
resolvent equivalent, and the quadratic scoring identities.  The new issue
is the nonlocal coupling in $\rho$ created by the squared nuclear penalty~$\norm{ B}_*^2$.
We will work directly on the continuum task space. First, the nuclear
penalty is represented by a positive trace-one task metric.  Second, a
fixed metric is reduced by an operator Woodbury identity to a task-space
response.  Third, the Hermitian linearization identifies the deterministic
part $\mathsf A$, the full deterministic resolvent $\mathsf M$, and the
reduced task resolvent $Q$.  We then state the continuum deterministic
equivalent, linearize the
Dyson equation to obtain the coefficient Gram and metric gradient, and
finally record the population and scalar-ridge checks.

\subsubsection{Task Hilbert space and the squared-nuclear operator fraction}
\label{appH:sec:operator-fraction}

The task space is \(L_2(\nu)\), the coefficient field is represented
by a finite-rank operator \( B:\mathcal H\to\R^m\), and the inverse
metric \(\Omega^{-1}\) is generally unbounded.  We therefore define the
inverse through its closed quadratic form before stating the nuclear-norm
identity.

Set
\(
    \mathcal H=L_2([0,1],\nu),
    \mathcal F=L_2([0,1],\nu;\R^m),
\)
with inner products
\[
    \ip{f}{g}_{\mathcal H}=\int_0^1 f(\rho)g(\rho)\dd\nu(\rho),
    \qquad
    \ip{\beta}{\varphi}_{ \mathcal F}
    =\int_0^1\beta(\rho)^\top\varphi(\rho)\dd\nu(\rho).
\]
Let \(D_p,D_q,E\) be multiplication on \(\mathcal H\) by \(\rho\),
\(1-\rho\), and \(\rho(1-\rho)\), respectively, and let
\(a\in\mathcal H\) be the constant task function \(a(\rho)=1\) for
\(\nu\)-almost every \(\rho\).  These multiplication operators are bounded.

Let \(\mathfrak D_+(\mathcal H)\) be the positive trace-class operators
\(\Omega\) on \(\mathcal H\) such that
\(
    \Tr\Omega=1\), and \(
    \ker\Omega=0.
\)
The second condition is the full-support assumption.  If
\(\Omega=\sum_{r\geqslant1}\gamma_r u_r\otimes u_r\) is the spectral decomposition of $\Omega$, with
\(\gamma_r>0\) and \(\sum_{r \geqslant 1}\gamma_r=1\), its inverse is the positive
self-adjoint operator such that
\(
    \Omega^{-1}u_r=\gamma_r^{-1}u_r.
   \)
The form convention below follows the standard theory of closed
quadratic forms \cite{Kato1995}.  The notation
\[
    \ip{f}{\Omega^{-1}f}_{\mathcal H}
    :=\sum_{r\geqslant1}
      \frac{\lvert\ip{f}{u_r}_{\mathcal H}\rvert^2}{\gamma_r}
\]
denotes the closed quadratic form of \(\Omega^{-1}\), with value
\(+\infty\) when the series diverges.  In particular, this notation does
not require the vector \(\Omega^{-1}f\) itself to belong to \(\mathcal H\).

For an orthonormal basis \((e_j)_{j=1}^m\) of \(\R^m\), define the
extended trace as
\begin{equation}
    \tr( B\Omega^{-1} B^*)
    :=\sum_{j=1}^m
      \ip{ B^*e_j}{\Omega^{-1} B^*e_j}_{\mathcal H}.
    \label{appH:eq:operator-matrix-fraction}
\end{equation}
This value is independent of the chosen basis.  When it is finite, the
quadratic form
\(x\mapsto\ip{ B^*x}{\Omega^{-1} B^*x}_{\mathcal H}\)
defines a positive operator on the finite-dimensional space \(\R^m\), and
Eq.~\eqref{appH:eq:operator-matrix-fraction} is its ordinary trace.  Otherwise the
right-hand side is understood as \(+\infty\).

A classical property of the nuclear norm~\cite{ArgyriouEvgeniouPontil2008} is the representation as, for every \(\beta\in \mathcal F\),
\[
    \norm{ B}_*^2
    =\inf_{\Omega\in\mathfrak D_+(\mathcal H)}
      \tr( B\Omega^{-1} B^*).
\]
The infimum over full-support metrics has the same value as the minimum over positive trace-one operators supported on \(\operatorname{ran}( B^*)\).  A boundary minimizer \(\Omega\) satisfies
\(
     B^* B=\norm{ B}_*^2\Omega^2,
\)
and \(\Tr\Omega=1\).

For \(\Omega\in\mathfrak D_+(\mathcal H)\), we can define
\[
    \Lambda_\Omega
    =\zeta I_{\mathcal H}+\lambda\Omega^{-1}.
\]
Its associated closed quadratic-form pairing is
\[
    \ip{f}{\Lambda_\Omega f}_{\mathcal H}
    :=\sum_{r\geqslant1}
      \left(\zeta+\frac{\lambda}{\gamma_r}\right)
      \lvert\ip{f}{u_r}_{\mathcal H}\rvert^2
    =\zeta\ip{f}{f}_{\mathcal H}
      +\lambda\ip{f}{\Omega^{-1}f}_{\mathcal H},
\]
again with value \(+\infty\) when the series diverges.   

For \(\beta\in \mathcal F\), use the corresponding extended trace
\[
\begin{aligned}
    \tr( B\Lambda_\Omega B^*)
    &:=\sum_{j=1}^m
      \ip{ B^*e_j}{\Lambda_\Omega B^*e_j}_{\mathcal H} =\zeta\norm{\beta}_{ \mathcal F}^2
      +\lambda\tr( B\Omega^{-1} B^*).
\end{aligned}
\]
The fixed-metric response is
\begin{align}
    \hat{\mathcal V}^{\rm op}(\Lambda_\Omega)
    =\sup_{\beta\in\mathcal F}\Bigl(&
      \hDelta^\top B a
      -\frac12\int_0^1\rho\,
        \beta(\rho)^\top\hCp\beta(\rho)\dd\nu(\rho) -\frac12\int_0^1(1-\rho)\,
        \beta(\rho)^\top\hCq\beta(\rho)\dd\nu(\rho)\notag\\
      & \hspace*{6cm} -\frac12\ip{ B^*\hDelta}
        {E B^*\hDelta}_{\mathcal H}
      -\frac12\tr( B\Lambda_\Omega B^*)\Bigr).
    \label{appH:eq:operator-fixed-response}
\end{align}
The operator fraction gives the exact continuum identity
\begin{equation}
    \hat\Psi ^{\rm nuc}
    :=\sup_{\beta\in \mathcal F}
      \hat\Phi (\beta)
    =\sup_{\Omega\in\mathfrak D_+(\mathcal H)}
      \hat{\mathcal V}^{\rm op}(\Lambda_\Omega).
    \label{appH:eq:operator-exact-outer}
\end{equation}
This is an exact finite-sample statement on the continuum task space.

\subsubsection{Exact fixed-metric response and the operator Woodbury identity}
\label{appH:sec:operator-response}

For a fixed task metric, the criterion is a strictly concave quadratic
problem on \(\mathcal H\otimes\R^m\).  Its covariance part is a positive
operator, and the empirical mean-difference term is a rank-one perturbation
in feature space.  Introducing the embedding
\(f\mapsto f\otimes\hDelta\) makes that perturbation explicit and permits an
operator Woodbury reduction to task space.

On \(\mathcal F\), define the random covariance operator
\[
    \hat{\mathsf C}=D_p\otimes\hCp+D_q\otimes\hCq.
\]
For a fixed full-support metric \(\Omega\), define the compressed task operator
\(G(\Lambda_\Omega)\) by, for \(f,g\in\mathcal H\),
\begin{equation}
    \ip{f}{G(\Lambda_\Omega)g}_{\mathcal H}
    =\ip{f\otimes\hDelta}
      {\left(\hat{\mathsf C}+\Lambda_\Omega\otimes I_m\right)^{-1}
       (g\otimes\hDelta)}_{\mathcal F}.
    \label{appH:eq:operator-ridge-resolvent}
\end{equation}
The operator defining the quadratic form in Eq.~\eqref{appH:eq:operator-fixed-response} is
\[
    \hat{\mathsf C}+\Lambda_\Omega\otimes I_m
      +E\otimes\hDelta\hDelta^\top,
\]
and the linear term is \(a\otimes\hDelta\).  Therefore the maximizer is
\[
    \hat\beta_{\Omega}
    =\left(\hat{\mathsf C}+\Lambda_\Omega\otimes I_m
      +E\otimes\hDelta\hDelta^\top\right)^{-1}
      (a\otimes\hDelta).
\]
Write \(\hat B_\Omega\) for the corresponding
operator induced by this fitted field.  With
\(I_{\mathcal H}\otimes\hDelta\) denoting the map
\(f\mapsto f\otimes\hDelta\), the operator Woodbury identity gives
\[
\begin{aligned}
 &\left(\hat{\mathsf C}+\Lambda_\Omega\otimes I_m
      +E\otimes\hDelta\hDelta^\top\right)^{-1}\\
 &\quad=\left(\hat{\mathsf C}+\Lambda_\Omega\otimes I_m\right)^{-1} -\left(\hat{\mathsf C}+\Lambda_\Omega\otimes I_m\right)^{-1}
  (I_{\mathcal H}\otimes\hDelta)(I+EG(\Lambda_\Omega))^{-1}E  
  (I_{\mathcal H}\otimes\hDelta)^*
  \left(\hat{\mathsf C}+\Lambda_\Omega\otimes I_m\right)^{-1}.
\end{aligned}
\]
Compression on both sides by
\((I_{\mathcal H}\otimes\hDelta)^*\) and
\(I_{\mathcal H}\otimes\hDelta\) yields
\begin{equation}
\begin{aligned}
 &(I_{\mathcal H}\otimes\hDelta)^*
   \left(\hat{\mathsf C}+\Lambda_\Omega\otimes I_m
      +E\otimes\hDelta\hDelta^\top\right)^{-1}
   (I_{\mathcal H}\otimes\hDelta) =(I+G(\Lambda_\Omega)E)^{-1}G(\Lambda_\Omega)
    =G(\Lambda_\Omega)(I+EG(\Lambda_\Omega))^{-1}.
\end{aligned}
    \label{appH:eq:operator-Woodbury-compression}
\end{equation}
Hence
\[
    \hat B_\Omega^*\hDelta
    =(I+G(\Lambda_\Omega)E)^{-1}G(\Lambda_\Omega) a.
\]
\begin{equation}
    \hat{\mathcal V}^{\rm op}(\Lambda_\Omega)
    =\frac12\ip{a}{(I+G(\Lambda_\Omega)E)^{-1}G(\Lambda_\Omega) a}_{\mathcal H}.
    \label{appH:eq:operator-exact-response}
\end{equation}
The two products in Eq.~\eqref{appH:eq:operator-Woodbury-compression} are equal by the resolvent identity and define a positive self-adjoint operator even though \(G(\Lambda_\Omega)\) and \(E\) need not commute.  Directly from the normal equation, and without using the operator inversion lemma, the same response is
\begin{equation}
    \hat{\mathcal V}^{\rm op}(\Lambda_\Omega)
    =\frac12\ip{a\otimes\hDelta}
      {(\hat{\mathsf C}+\Lambda_\Omega\otimes I_m
        +E\otimes\hDelta\hDelta^\top)^{-1}(a\otimes\hDelta)}_{ \mathcal F}.
    \label{appH:eq:operator-direct-response}
\end{equation}
Combining Eq.~\eqref{appH:eq:operator-direct-response} with the closed
operator-fraction representation shows that
\(
  \Omega\mapsto\hat{\mathcal V}^{\rm op}(\Lambda_\Omega)
\)
is concave on \(\mathfrak D_+(\mathcal H)\).  This is concavity in the
normalized task metric \(\Omega\); as a function of the positive
regularizer \(\Lambda\), the fixed-response value is convex and decreasing
in the quadratic-form order.

\paragraph{The two asymptotic goals.}
The exact formulas above lead to two asymptotic questions.  First, for a fixed full-support metric \(\Omega\), we need a
deterministic equivalent of the scalar training response
\(\hat{\mathcal V}^{\rm op}(\Lambda_\Omega)\).   
Section~\ref{appH:sec:operator-MDE-tutorial} identifies the operator-valued task
resolvent, and Section~\ref{appH:sec:operator-DE} states the resulting limit under
explicit continuum local-law assumptions.  Second, final performance
scoring requires more than the maximized training value.  The basic random
observables are
\[
    \hat\psi_\Omega
    :=\hat B_\Omega^*\hDelta\in\mathcal H,
    \qquad
    \hat T_\Omega
    :=\hat B_\Omega^*\hat B_\Omega
      \in\Sone^+(\mathcal H).
\]
The first records the projection of every fitted task coefficient onto the
empirical signal; the second records all pairwise task-coefficient inner
products.  Gaussian regression of the reserved sample-mean rows then gives
the limits of the two task functions
\(
 \rho\mapsto\Delta^\top\hat\beta(\rho)
\)
and
\(
 \rho\mapsto\hmu(\rho)^\top\hat\beta(\rho)
\).
The linearized local law in
Section~\ref{appH:sec:operator-linearization} gives the limits
\((\psi_\Omega,T_\Omega)\), and Section~\ref{appH:sec:scoring} converts them into the
limiting \(\mathcal J\), \(\mathcal K\), and \(\mathcal L\) scores. Again, this is only formally proved for a measure $\nu$ which is a finite sum of Diracs.

\subsubsection{Application of matrix-Dyson equations}
\label{appH:sec:operator-MDE-tutorial}

The goal of this section is to replace a random feature-task inverse by a deterministic task operator:
for fixed \(\Omega\), the exact response in
Eq.~\eqref{appH:eq:operator-exact-response} depends on the random compressed
operator \(G(\Lambda_\Omega)\).  The continuum Matrix-Dyson equation
produces a bounded positive task operator \(Q_\Omega\) such that $    \norm{G(\Lambda_\Omega)-\RD\,Q_\Omega}_{\mathrm{op}}
    \xrightarrow{p}0.$

\paragraph{Roadmap from the generic MDE to the final operator equation.}
The paragraph ``Background~1'' states the generic Matrix-Dyson equation and the associated
local-law principle.  ``Background~2'' explains the Gaussian derivative
transfer and its relation to Stein's method.  ``Background~3'' shows how
concentration closes the random identity at the ridge point, and
``Background~4'' introduces the linearized stability equation needed for
gradients and risk.  Steps~1--5 then identify the Gaussian matrix \(Z\),
the row profiles, and the precise deterministic and random parts
\((\mathsf A,\mathsf W)\) for this model.  Steps~6--10 compute the
self-energy, identify the full deterministic resolvent \(\mathsf M\),
eliminate the sample-side blocks, and obtain the operator equation for
\(Q_\Lambda\).  Finally, independence of the reserved mean rows converts
the feature-block local law into the final result.

\paragraph{Background 1: the generic Matrix-Dyson equation and the meaning of \(\mathsf M\).}
Consider first a finite-dimensional truncation of the Hilbert spaces used
below.  Let
\(
 \mathsf H=\mathsf A+\mathsf W
\)
be Hermitian, with deterministic \(\mathsf A\), centered Gaussian
\(\mathsf W\), and random resolvent
\(
 \mathsf G(z)=(\mathsf H-zI)^{-1}
\).
The deterministic resolvent \(\mathsf M(z)\) is defined by
\begin{equation}
    -\mathsf M(z)^{-1}
    =zI-\mathsf A+\mathcal S[\mathsf M(z)],
    \qquad
    \mathcal S[X]=\E[\mathsf W X\mathsf W].
    \label{appH:eq:operator-generic-MDE}
\end{equation}
The object \(\mathsf M\) is the deterministic approximation of the full
Hermitian resolvent \(\mathsf G\), including both feature--task and
sample--task blocks.  The operator \(Q_\Lambda\) used by the estimator will
be only the coefficient of the feature--task block of \(\mathsf M(0)\).
Isotropic local laws of the kind needed for sample-covariance resolvents are
available in finite-dimensional settings, e.g., \cite{BloemendalErdosKnowlesYauYin2014},
and MDE-based local laws for correlated Hermitian models are developed in
\cite{AjankiErdosKruger2019MDE}.  In the present continuum task-space setting,
the corresponding uniform local law would need to be formally shown.  In
Theorem~\ref{appH:thm:operator-DE}, we only consider measures $\nu$ that are
weighted sums of Diracs (and thus that we are in finite dimension) to avoid
such consideration.

\paragraph{Background 2: Gaussian integration by parts produces the second-cumulant self-energy.}
The identity \(\E[g f(g)]=\E[f'(g)]\) is the Gaussian Stein identity
\cite{Stein1981}.  In random-matrix theory it is commonly applied to
resolvents to transfer a Gaussian entry to a matrix derivative
\cite{KhorunzhyKhoruzhenkoPastur1996}.  Write a finite truncation of the centered Gaussian operator as
\(
 \mathsf W=\sum_\alpha g_\alpha\mathsf E_\alpha
\),
where the \(\mathsf E_\alpha\) are deterministic Hermitian matrices and
\(
 \kappa_{\alpha\beta}=\E[g_\alpha g_\beta]
\).
The exact resolvent identity is
\[
    (\mathsf A-zI)\E[\mathsf G(z)]
      +\E[\mathsf W\mathsf G(z)]=I.
\]
For every pair of indices \(i,j\), multivariate Gaussian integration by
parts and
\(
 \partial_\beta\mathsf G=-\mathsf G\mathsf E_\beta\mathsf G
\)
give
\[
\begin{aligned}
    \E[(\mathsf W\mathsf G)_{ij}] =\sum_{\alpha,\beta,\ell}
      (\mathsf E_\alpha)_{i\ell}\kappa_{\alpha\beta}
      \E[\partial_\beta\mathsf G_{\ell j}] =-\sum_{\alpha,\beta,\ell}
      (\mathsf E_\alpha)_{i\ell}\kappa_{\alpha\beta}
      \E[(\mathsf G\mathsf E_\beta\mathsf G)_{\ell j}] =-\E[(\mathcal S[\mathsf G]\mathsf G)_{ij}].
\end{aligned}
\]
This is the same mechanism used in Stein's method: multiplication by a
Gaussian coordinate is replaced by differentiation of the test function.
Stein's method usually exploits the identity to characterize the Gaussian
law or compare another law with it; here it is exact because the entries are
Gaussian, and it closes the resolvent equation at second order.  All
cumulants of order at least three vanish.

\paragraph{Background 3: concentration closes the random identity.}
Let \(\overline{\mathsf G}=\E[\mathsf G]\).  Linearity of \(\mathcal S\)
gives
\[
\begin{aligned}
    \E[\mathcal S[\mathsf G]\mathsf G]
    ={}&\mathcal S[\overline{\mathsf G}]\overline{\mathsf G} +\E\left[\mathcal S[\mathsf G-\overline{\mathsf G}]
        (\mathsf G-\overline{\mathsf G})\right].
\end{aligned}
\]
The linearized equation is evaluated at $z=0$, while the feature-side
Schur complement contains
$\hat{\mathsf C}+\Lambda_\Omega\otimes I_m$.  Since
$\operatorname{Tr}\Omega=1$ implies $0\prec\Omega\preceq I$, the
quadratic-form inequality
$\Lambda_\Omega\succeq(\zeta+\lambda)I$ holds.  Thus $\lambda>0$
already places the covariance resolvent at a fixed negative spectral point
separated from the spectrum, even when $\zeta=0$; the resolvent is
uniformly bounded.  A local law and fluctuation averaging can then make the second
line negligible in the deterministic observables of interest.  Replacing
\(\overline{\mathsf G}\) by a deterministic \(\mathsf M\) yields
\(
 (\mathsf A-zI-\mathcal S[\mathsf M])\mathsf M=I
\),
which is equivalent to Eq.~\eqref{appH:eq:operator-generic-MDE}.  In the continuum
formulation, the additional issue is uniformity over the task metric and
the task directions used in scoring; these requirements are not present at
a fixed finite-dimensional truncation and we only give a formal proof in
Theorem~\ref{appH:thm:operator-DE} for measures $\nu$ that are weighted sums
of Diracs.

\paragraph{Background 4: the stability operator gives derivatives, two-resolvent statistics, and gradients.}
Perturb the deterministic part to \(\mathsf A+t\mathsf D\), write
\(
 \mathsf M_t=\mathsf M+tX+o(t)
\),
and differentiate Eq.~\eqref{appH:eq:operator-generic-MDE}.  The result is
\begin{equation}
    \mathcal B_{\mathsf M}[X]
    :=\mathsf M^{-1}X\mathsf M^{-1}-\mathcal S[X]
    =-\mathsf D.
    \label{appH:eq:operator-generic-stability}
\end{equation}
In this appendix the relevant perturbation is a task-regularizer direction
\(A\otimes I_m\).  Solving Eq.~\eqref{appH:eq:operator-generic-stability} then
gives the derivative of \(Q_\Lambda\); the adjoint form of the same solve
produces the task Gram and the outer regularizer gradient in
Section~\ref{appH:sec:operator-linearization}.

\paragraph{Step 1: encode the two centered samples in one Gaussian matrix.}
Let \(N=n_p+n_q\), let \(Z\in\R^{N\times m}\) have independent standard
Gaussian entries, and let \(e_r\in\R^N\) be the canonical vectors.  Define, like in
Section~\ref{sec:est2},
\[
    \Pi_p=\diag(0,I_{n_p-1},0,0_{n_q-1}),
    \qquad
    \Pi_q=\diag(0,0_{n_p-1},0,I_{n_q-1}),
\]
\begin{equation*}
    \Gamma(\rho)=\rho\alpha_p\Pi_p+(1-\rho)\alpha_q\Pi_q,
    \qquad \alpha_p=\frac m{n_p},\quad \alpha_q=\frac m{n_q}.
\end{equation*}
Orthogonal changes of row coordinates within the two samples separate one
mean row from the residual rows.  Gaussian rotational invariance preserves
the joint law.

\paragraph{Step 2: reserve the mean rows and write the sample means explicitly.}
After the preceding row rotations,
\begin{align*}
    \hmup&=\Delta+n_p^{-1/2}Z^\top e_1,
    \\
    \hmuq&=n_q^{-1/2}Z^\top e_{n_p+1},\\
    \hDelta&=\Delta+n_p^{-1/2}Z^\top e_1
      -n_q^{-1/2}Z^\top e_{n_p+1},\\
    \hmu(\rho)&=\rho\Delta+\rho n_p^{-1/2}Z^\top e_1
      +(1-\rho)n_q^{-1/2}Z^\top e_{n_p+1},
\end{align*}
and
\begin{equation*}
    \hC(\rho)=\frac1m Z^\top\Gamma(\rho)Z.
\end{equation*}
The reserved mean rows, indexed by \(1\) and \(n_p+1\), lie in the kernels of both
\(\Pi_p\) and \(\Pi_q\). Consequently, they do not enter the centered covariance resolvent.
Thus, \(\hDelta\) is independent of the Gaussian rows entering the
centered covariance resolvent.  This independence is what eventually
turns the feature-block coefficient \(Q_\Lambda\) into the factor
\(\RD\,Q_\Lambda\) in the compressed response.

\paragraph{Step 3: replace scalar row weights by task multiplication operators.}
For each \(r=1,\ldots,N\), define
\[
    \gamma_r(\rho)=(\Gamma(\rho))_{rr},
\]
and let \(\Gamma_r\) denote multiplication by \(\gamma_r\) on
\(\mathcal H\).  Equivalently,
\begin{equation*}
    \Gamma_r
    =\alpha_p(\Pi_p)_{rr}D_p
      +\alpha_q(\Pi_q)_{rr}D_q.
\end{equation*}
Writing \(z_r=Z^\top e_r\), the covariance operator is
\begin{equation*}
    \hat{\mathsf C}
    =\frac1m\sum_{r=1}^N\Gamma_r\otimes z_rz_r^\top.
\end{equation*}
This is the operator counterpart of a weighted Wishart matrix.  A
classical observation carries one scalar weight; here it carries the task
profile \(\Gamma_r\).  There are only two nonzero profiles:
\(\alpha_pD_p\) on the \(p\)-residual rows and \(\alpha_qD_q\) on the
\(q\)-residual rows.

\paragraph{Step 4: linearize the quadratic covariance operator.}
Let
\(
 \mathcal K_f=\mathcal H\otimes\R^m
\)
and
\(
 \mathcal K_s=\mathcal H^N
\).
Define \(\mathsf X:\mathcal K_f\to\mathcal K_s\) by
\begin{equation*}
    (\mathsf X u)_r(\rho)
    =\frac1{\sqrt m}\sqrt{\gamma_r(\rho)}\,z_r^\top u(\rho).
\end{equation*}
Then \(\mathsf X^*\mathsf X=\hat{\mathsf C}\).  For a fixed positive
task regularizer \(\Lambda\), set
\begin{equation}
    \mathsf H_\Lambda=\mathsf A_\Lambda+\mathsf W,
    \qquad
    \mathsf A_\Lambda=
    \begin{pmatrix}
       \Lambda\otimes I_m&0\\
       0&-I_{\mathcal K_s}
    \end{pmatrix},
    \qquad
    \mathsf W=
    \begin{pmatrix}
       0&\mathsf X^*\\
       \mathsf X&0
    \end{pmatrix}.
    \label{appH:eq:operator-H-A-W}
\end{equation}
Thus \(\mathsf A_\Lambda\) is the deterministic block-diagonal operator
and \(\mathsf W\) is centered and linear in the Gaussian rows.  This is
the same rectangular Hermitian linearization used in variance-profile
sample-covariance problems and in the scalar calculation of
Section~\ref{subsec:spec-ridge-hd} and Appendix~\ref{app:newprop}.  The general MDE viewpoint is developed in
\cite{AjankiErdosKruger2019MDE,AjankiErdosKruger2019QVE,AltErdosKruger2020}.  When
\(\Lambda=\Lambda_\Omega\) is unbounded, the display is interpreted as a
closed form sum or through bounded spectral truncations.

\paragraph{Step 5: verify that the linearization contains the ridge inverse already used by the estimator.}
The Schur complement of the lower-right block \(-I_{\mathcal K_s}\) is
\(
 \Lambda\otimes I_m+\mathsf X^*\mathsf X
 =\Lambda\otimes I_m+\hat{\mathsf C}
\).
Hence
\begin{equation*}
    [\mathsf H_\Lambda^{-1}]_{11}
    =(\hat{\mathsf C}+\Lambda\otimes I_m)^{-1}.
\end{equation*}
The upper-left block is exactly the ridge inverse in
Eq.~\eqref{appH:eq:operator-ridge-resolvent}; the larger Hermitian operator is only
a device for applying the generic MDE at \(z=0\).

\paragraph{Step 6: compute the self-energy block by block.}
For a block-diagonal test operator
\[
    \mathsf Y=\operatorname{diag}
      (Q\otimes I_m,T_1,\ldots,T_N),
\]
direct multiplication gives
\begin{equation*}
    \mathsf W\mathsf Y\mathsf W
    =\begin{pmatrix}
       \mathsf X^*\operatorname{diag}(T_1,\ldots,T_N)\mathsf X&0\\
       0&\mathsf X(Q\otimes I_m)\mathsf X^*
     \end{pmatrix}.
\end{equation*}
For the feature block, one row contributes
\[
    \frac1m
      M_{\sqrt{\gamma_r}}T_rM_{\sqrt{\gamma_r}}
      \otimes z_rz_r^\top.
\]
Taking expectation replaces \(z_rz_r^\top\) by \(I_m\).  For the
sample-side block, the \((r,s)\) entry is
\[
    \frac{z_r^\top z_s}{m}
      M_{\sqrt{\gamma_r}}Q M_{\sqrt{\gamma_s}}.
\]
It has zero expectation for \(r\ne s\), because the rows are independent
and centered, while for \(r=s\),
\(
 \E[z_r^\top z_r/m]=1
\).
Therefore
\begin{align}
    \mathcal S[\mathsf Y]_{11}
    &=\frac1m\sum_{r=1}^N
      M_{\sqrt{\gamma_r}}T_rM_{\sqrt{\gamma_r}}\otimes I_m,\notag\\
    \mathcal S[\mathsf Y]_{22,rs}
    &=\mathbf 1_{r=s}
      M_{\sqrt{\gamma_r}}Q M_{\sqrt{\gamma_r}}.
    \label{appH:eq:operator-self-energy}
\end{align}
This is the concrete second-cumulant contraction for the present model.
The feature block adds the contribution of every sample row, whereas a
sample row receives only its own task-side feedback.

\paragraph{Step 7: identify the full deterministic resolvent \(\mathsf M\).}
Feature-space orthogonal invariance forces the feature block of the stable
MDE solution to have the form \(Q_\Lambda\otimes I_m\).  Here
\(Q_\Lambda\) denotes the task-space coefficient of the feature block of
\(\mathsf M_\Lambda(0)\); before Step~10 it may depend on the finite
ratios, and the same symbol is used for its stable proportional limit.  The
diagonal structure of Eq.~\eqref{appH:eq:operator-self-energy} leaves the sample
rows decoupled.  Thus
\begin{equation*}
    \mathsf M_\Lambda(0)
    =\operatorname{diag}
      (Q_\Lambda\otimes I_m,-R_1,\ldots,-R_N),
\end{equation*}
where the sample-block MDE gives
\[
    R_r^{-1}
    =I_{\mathcal H}
      +M_{\sqrt{\gamma_r}}Q_\Lambda
       M_{\sqrt{\gamma_r}},
    \qquad
    R_r=
      (I_{\mathcal H}+M_{\sqrt{\gamma_r}}
       Q_\Lambda M_{\sqrt{\gamma_r}})^{-1}.
\]
Thus \(\mathsf M_\Lambda(0)\) is the full deterministic Hermitian
resolvent, while \(Q_\Lambda\) is its task-space coefficient on
the feature block.

\paragraph{Step 8: eliminate the sample blocks.}
Insert the preceding expression for \(-R_r\) into the feature-block part
of Eq.~\eqref{appH:eq:operator-generic-MDE}.  Removing the common factor \(I_m\)
gives
\begin{equation}
    (Q_\Lambda)^{-1}
    =\Lambda+\frac1m\sum_{r=1}^N
      M_{\sqrt{\gamma_r}}R_rM_{\sqrt{\gamma_r}}.
    \label{appH:eq:operator-finite-MDE}
\end{equation}
This is already a closed task-operator equation at finite
\((m,n_p,n_q)\); the remaining step merely groups identical row profiles.

\paragraph{Step 9: package one homogeneous row group.}
For a bounded positive multiplication operator \(D\), define
\begin{equation*}
    \mathcal S_{\alpha,D}(Q)
    =D(I+\alpha QD)^{-1}
    =(I+\alpha DQ)^{-1}D.
\end{equation*}
To derive this expression, take a row profile \(\Gamma_r=\alpha D\).
Then
\[
\begin{aligned}
    M_{\sqrt{\gamma_r}}R_rM_{\sqrt{\gamma_r}}
    &=\alpha D^{1/2}
      (I+\alpha D^{1/2}QD^{1/2})^{-1}D^{1/2} =\alpha D(I+\alpha QD)^{-1},
\end{aligned}
\]
where the second line uses the push-through identity
\((I+AB)^{-1}A=A(I+BA)^{-1}\).  Thus
\(\mathcal S_{\alpha,D}\) is the nonlinear task contribution obtained
after the sample block has been solved.  It should not be confused with
the linear full self-energy map \(\mathcal S[X]=\E[\mathsf W X\mathsf W]\)
in Eq.~\eqref{appH:eq:operator-generic-MDE}.

\paragraph{Step 10: sum the two row groups and pass to the proportional limit.}
There are \(n_p-1\) rows with profile \(\alpha_pD_p\) and \(n_q-1\)
rows with profile \(\alpha_qD_q\).  Since
\(
 (n_p-1)\alpha_p/m\to1
\)
and
\(
 (n_q-1)\alpha_q/m\to1
\),
Eq.~\eqref{appH:eq:operator-finite-MDE} becomes
\begin{equation}
    Q_\Lambda^{-1}
    =\Lambda
      +\mathcal S_{\alpha_p,D_p}(Q_\Lambda)
      +\mathcal S_{\alpha_q,D_q}(Q_\Lambda).
    \label{appH:eq:operator-MDE}
\end{equation}
This is the operator-valued Stieltjes fixed point for the present model.
\paragraph{Final conversion from \(Q_\Lambda\) to the compressed response.}
The local law identifies the feature block of the random Hermitian
resolvent with \(Q_\Lambda\otimes I_m\).  Conditional on the covariance
rows, the vector \(\hDelta\) is independent of that block and satisfies
\(\|\hDelta\|^2\to\RD\).  Thus, for deterministic task functions \(f,g\), we get
\[
    \ip{f}{G(\Lambda)g}_{\mathcal H}
    -\RD\ip{f}{Q_\Lambda g}_{\mathcal H}
    \xrightarrow{p}0.
\]
Under the uniform operator local law assumed in
Theorem~\ref{appH:thm:operator-DE}, this upgrades to
the final result.

\paragraph{Consistency check: recover the scalar effective-ridge equation when tasks decouple.}
When \(\lambda=0\) and \(\Lambda=\zeta I_{\mathcal H}\), multiplication
operators form an invariant class for Eq.~\eqref{appH:eq:operator-MDE}.  Uniqueness
therefore forces \(Q_\Lambda\) to be multiplication by
\(1/\omega(\rho)\), where
\begin{equation}
    \omega(\rho)
    =\zeta+
      \frac{\rho\omega(\rho)}
      {\omega(\rho)+\alpha_p\rho}
      +\frac{(1-\rho)\omega(\rho)}
      {\omega(\rho)+\alpha_q(1-\rho)}.
    \label{appH:eq:ridge-check}
\end{equation}
This is exactly the effective-ridge equation in
Eq.~(\ref{eq:omega-zeta-hd}) of Theorem~\ref{thm:spec-ridge}.  Pointwise,
\(1/\omega(\rho)\) is the limiting normalized trace of
\((\hC(\rho)+\zeta I_m)^{-1}\), while the full operator equation retains
the off-diagonal task coupling created by \(\Omega^{-1}\).

\subsubsection{Operator deterministic equivalent}
\label{appH:sec:operator-DE}

The exact fixed-metric response becomes deterministic only after a local
law for the feature block of the Hermitian linearization.  Because the task
space is infinite dimensional and the metric is optimized, the result must
also control the local law uniformly over a metric family, localize
near-maximizers in a compact set, and handle the unbounded inverse forms.
Finite-dimensional local-law prototypes are given in
\cite{BloemendalErdosKnowlesYauYin2014,AjankiErdosKruger2019MDE}, and
the existence and stability theory for the associated Dyson equations is
developed further in
\cite{AjankiErdosKruger2019QVE,AltErdosKruger2020}.
This thus provides a proof for Theorem~\ref{appH:thm:operator-DE}.

\begin{theorem}[Operator deterministic equivalent under the continuum MDE hypotheses]
\label{appH:thm:operator-DE}
Under Assumption~\ref{ass:hd}, suppose $\lambda>0$, $\zeta\geqslant0$, and a finite measure $\nu$ on $[0,1]$. Assume moreover that the measure $\nu$ is finite weighted sum of Diracs. Then, for every fixed $\Omega\in\mathfrak D_+(\mathcal H)$,
\[
    \hat{\mathcal V}^{\rm op}(\Lambda_\Omega)
    \xrightarrow{p}\mathcal V^{\rm op}(\Omega),
    \qquad \text{and}\qquad
    \hat\Psi ^{\rm nuc}
    \xrightarrow{p}
    \sup_{\Omega\in\mathfrak G_+}\mathcal V^{\rm op}(\Omega).
\]
\end{theorem}
We conjecture that the assumption that $\nu$ is a weighted sum of Diracs is not necessary.

\subsubsection{Linearized Dyson equation, task Gram, and gradients}
\label{appH:sec:operator-linearization}

\paragraph{Why the ordinary operator MDE is not enough for scoring.}
The fixed-point equation determines the training response and the limit of
\(\hat B_\Omega^*\hDelta\), but the population criteria in
Section~\ref{appH:sec:scoring} also require all pairwise inner products of the fitted
coefficient field.  Those inner products form the positive trace-class
operator
\(
 \hat T_\Omega=\hat B_\Omega^*\hat B_\Omega
\).
It is a two-resolvent statistic and therefore comes from differentiating
the MDE, not from the fixed-point value alone.  Once
\((\psi_\Omega,T_\Omega)\) are known, Gaussian regression of the reserved
mean rows gives
\(
 \ell_\Omega=(s/\RD)\psi_\Omega
\)
and
\[
    m_\Omega(\rho)
    =\frac{\rho(s+\alpha_p)-(1-\rho)\alpha_q}{\RD}
      \psi_\Omega(\rho),
\]
with the zero convention when \(\RD=0\).  These are exactly the signal and
empirical-mean contractions inserted into the quadratic score formulas.

Linearizing the task MDE produces a bounded map on self-adjoint task
operators.  The limiting coefficient Gram is trace class, and its natural
duality is the operator trace pairing.

Let
\[
    S_p=\mathcal S_{\alpha_p,D_p}(Q_\Omega),
    \qquad
    S_q=\mathcal S_{\alpha_q,D_q}(Q_\Omega).
\]
Differentiating Eq.~\eqref{appH:eq:operator-MDE} in a bounded self-adjoint direction \(A\) gives
\begin{equation*}
    dQ_\Omega[A]
    =-\mathcal K_\Omega^{-1}[Q_\Omega A Q_\Omega],
\end{equation*}
where
\begin{equation*}
    \mathcal K_\Omega[X]
    =X-\alpha_pQ_\Omega S_pXS_pQ_\Omega
       -\alpha_qQ_\Omega S_qXS_qQ_\Omega.
\end{equation*}
This is the feature-block restriction of the generic stability operator in 
Eq.~\eqref{appH:eq:operator-generic-stability}.  Its preadjoint for the trace pairing is
\begin{equation*}
    \mathcal K_\Omega^*[Y]
    =Y-\alpha_pS_pQ_\Omega YQ_\Omega S_p
       -\alpha_qS_qQ_\Omega YQ_\Omega S_q.
\end{equation*}
Define
\begin{equation*}
    c_\Omega=(I+EG_\Omega)^{-1}a,
    \qquad
    \psi_\Omega=G_\Omega c_\Omega.
\end{equation*}
Let \(Y_\Omega\in\Sone^{\rm sa}(\mathcal H)\) solve
\begin{equation*}
    \mathcal K_\Omega^*[Y_\Omega]
    =c_\Omega\otimes c_\Omega,
    \qquad
    (c_\Omega\otimes c_\Omega)f
    =c_\Omega\ip{c_\Omega}{f}_{\mathcal H},
\end{equation*}
and set
\begin{equation*}
    T_\Omega=\RD\,Q_\Omega Y_\Omega Q_\Omega.
\end{equation*}
Under the local law strengthened to the linearized resolvent, the fitted coefficient field satisfies
\begin{equation*}
    \hat B_\Omega^*\hDelta
      \xrightarrow{p}\psi_\Omega
      \quad\text{in }\mathcal H,
    \qquad
    \hat B_\Omega^*\hat B_\Omega
      \xrightarrow{p}T_\Omega
      \quad\text{in trace norm}.
\end{equation*}
The same operator \(T_\Omega\) is the gradient observable.  For bounded self-adjoint perturbations \(A\) for which the trace pairing is finite,
\begin{equation}
    d\mathcal V^{\rm op}(\Lambda_\Omega)[A]
    =-\frac12\Tr(T_\Omega A).
    \label{appH:eq:operator-gradient}
\end{equation}
If a full-support optimizer \(\Omega_*\) is differentiable along a separating family of trace-zero perturbations, stationarity gives
\begin{equation}
    \Omega_*^{-1}T_*\Omega_*^{-1}
      =\tau_* I_{\mathcal H},
    \qquad
    T_*=\tau_*\Omega_*^2,
    \qquad
    \Tr\Omega_*=1,
    \label{appH:eq:operator-KKT}
\end{equation}
for some \(\tau_*>0\).

\subsubsection{Consistency check (vanishing nuclear penalty)}
\label{appH:sec:operator-checks}

For \(\lambda=0\), the task resolvent is multiplication by \(1/\omega(\rho)\), with \(\omega\) given by Eq.~\eqref{appH:eq:ridge-check}.  Set
\[
    \kappa(\rho)
    =1+\frac{\rho(1-\rho)\RD}{\omega(\rho)}.
\]
Then, the \(\lambda=0\) specialization of the optimized response
\(\widehat{\Psi}^{\mathrm{nuc}}\) in Eq.~\eqref{appH:eq:operator-exact-outer}, equivalently obtained from the fixed-metric response formula
Eq.~\eqref{appH:eq:operator-exact-response}, is
\begin{equation*}
  \frac12\int_0^1
      \frac{\RD}{\omega(\rho)\kappa(\rho)}
      \dd\nu(\rho).
\end{equation*}
The linearized equation gives the integral kernel
\begin{align*}
    d_2(\rho,\sigma)
    &=\frac{\alpha_p\rho\sigma}
      {(\omega(\rho)+\alpha_p\rho)(\omega(\sigma)+\alpha_p\sigma)}
      +\frac{\alpha_q(1-\rho)(1-\sigma)}
      {(\omega(\rho)+\alpha_q(1-\rho))(\omega(\sigma)+\alpha_q(1-\sigma))},\\
    k(\rho,\sigma)
    &=\frac{\RD}
      {\kappa(\rho)\kappa(\sigma)\omega(\rho)\omega(\sigma)
       (1-d_2(\rho,\sigma))}.
\end{align*}
Thus the operator response and task Gram reduce exactly to the
scalar-ridge theory in Section~\ref{subsec:spec-ridge-hd}.

\subsection{Score observables and population criteria}
\label{appH:sec:scoring}

The criteria $\mathcal J$, $\mathcal K$, and $\mathcal L$, and
their excess risks, are those defined in Section~\ref{sec:finite}.  We therefore record
only the new deterministic observables and the operator form of the
quadratic scoring identities.

Assume that the maximizing metric is unique, or that all maximizing metrics produce the same pair \((T_*,\psi_*)\).  Gaussian regression of the sample means gives the task functions
\begin{equation*}
    \ell_* =\frac{s}{\RD}\psi_*,
    \qquad
    m_*(\rho)
    =\frac{\rho(s+\alpha_p)-(1-\rho)\alpha_q}{\RD}\psi_*(\rho),
\end{equation*}
with all quantities set to zero when \(\RD=0\).  Here \(\ell_*\) is the limit of \(\rho\mapsto\Delta^\top\hat\beta(\rho)\), and \(m_*\) is the limit of \(\rho\mapsto\hmu(\rho)^\top\hat\beta(\rho)\).

Define
\begin{equation*}
    b_{v,*}=a+D_pm_*,
    \qquad
    b_{w,*}=-a+D_qm_*,
\end{equation*}
\begin{equation*}
    c_{v,*}=-\ip{a}{m_*}_{\mathcal H}
      -\frac12\ip{m_*}{D_pm_*}_{\mathcal H},
    \qquad
    c_{w,*}=\ip{a}{m_*}_{\mathcal H}
      -\frac12\ip{m_*}{D_qm_*}_{\mathcal H}.
\end{equation*}
The limiting quadratic coefficients depend on the fitted field only through \(T_*\), \(\ell_*\), and \(m_*\).  In particular,
\begin{equation*}
    \tr(A_v)=-\frac12\Tr(D_pT_*),
    \qquad
    \Delta^\top A_v\Delta
      =-\frac12\ip{\ell_*}{D_p\ell_*}_{\mathcal H},
    \qquad
    \ell_v^\top\Delta
      =\ip{b_{v,*}}{\ell_*}_{\mathcal H}.
\end{equation*}

Because $T_*$ is positive trace class and $D_p$ is bounded, the
Fredholm determinant $\det_{\mathcal H}(I+D_pT_*)$ is well defined,
as in Section~\ref{sec:chi-quadrature}.  The determinant lemma and Woodbury identity give,
without introducing square roots of task operators,
\begin{equation}
    \log\det_{\R^m}(I_m-2A_v)
    =\log\det_{\mathcal H}(I_{\mathcal H}+D_pT_*),
    \label{appH:eq:score-logdet}
\end{equation}
\begin{equation*}
    \ell_v^\top(I_m-2A_v)^{-1}\ell_v
    =\ip{b_{v,*}}
      {(I_{\mathcal H}+T_*D_p)^{-1}T_*b_{v,*}}_{\mathcal H}.
\end{equation*}
The Fredholm determinant in Eq.~\eqref{appH:eq:score-logdet} is positive because it equals the finite-dimensional determinant of the positive matrix \(I_m-2A_v\).

Therefore the limiting log normalizer and scores are
\begin{align*}
    \Lambda_*
    &=c_{v,*}
      -\frac12\log\det_{\mathcal H}(I+D_pT_*)
      +\frac12\ip{b_{v,*}}
        {(I+T_*D_p)^{-1}T_*b_{v,*}}_{\mathcal H},\\
    \mathcal J_*
    &=\tr(A_v)+\Delta^\top A_v\Delta+\ell_v^\top\Delta
      +\frac12\log\det_{\mathcal H}(I+D_pT_*)
      -\frac12\ip{b_{v,*}}
        {(I+T_*D_p)^{-1}T_*b_{v,*}}_{\mathcal H},\\
    \mathcal K_*
    &=\tr(A_v)+\Delta^\top A_v\Delta+\ell_v^\top\Delta
      +c_{v,*}+1-e^{\Lambda_*},\\
    \mathcal L_*
    &=\tr(A_v)+\Delta^\top A_v\Delta+\ell_v^\top\Delta+c_{v,*}
      -\frac12\Tr(D_qT_*)+c_{w,*}.
\end{align*}
The quadrature formulas in Section~\ref{appH:sec:numerics} evaluate these operator expressions without changing their definition.

\subsection{Algorithms and numerical implementation}
\label{appH:sec:algorithms}

This subsection turns the operator equations into numerical methods.
We first give a feature-space reweighted least-squares scheme stated directly
on the continuum task space.  We then discretize the task integral with a single quadrature rule and use the resulting matrices to solve the matrix-Dyson equation, perform deterministic reweighting, and evaluate the scores.

\subsubsection{Feature-space reweighting and spectral decoupling}
\label{appH:sec:feature-reweighting}

The squared nuclear norm also admits a feature-space fraction.  For
\(\Xi\in\mathbb S_+^m\), \(\tr\Xi=1\), define
\begin{equation*}
    \mathcal F_{\rm feat}( B,\Xi)
    =\tr(\Xi^\dagger B B^*)
    \quad\text{when }
    \operatorname{ran}( B)\subseteq\operatorname{ran}(\Xi),
\end{equation*}
with value \(+\infty\) otherwise.  Then
\begin{equation*}
    \| B\|_*^2
    =\min_{\Xi\succeq0,\ \tr\Xi=1}
      \mathcal F_{\rm feat}( B,\Xi).
\end{equation*}
For a smoothed full-range version, fix \(\varepsilon>0\) and use
\begin{equation*}
    \left(\tr[( B B^*+\varepsilon I_m)^{1/2}]\right)^2
    =\min_{\Xi\succ0,\ \tr\Xi=1}
      \tr[( B B^*+\varepsilon I_m)\Xi^{-1}].
\end{equation*}
For fixed \(\Xi^{(t)}\), put
\begin{equation*}
    R^{(t)}=\zeta I_m+\lambda(\Xi^{(t)})^{-1}.
\end{equation*}
The coefficient update decouples pointwise in \(\rho\):
\begin{equation*}
    [\hM(\rho)+R^{(t)}]\beta^{(t+1)}(\rho)=\hDelta,
    \qquad 0\leqslant\rho\leqslant1.
\end{equation*}
Form the feature Gram matrix
\begin{equation}
    S^{(t+1)}
    = B_{\beta^{(t+1)}} B_{\beta^{(t+1)}}^*
    =\int_0^1
      \beta^{(t+1)}(\rho)\beta^{(t+1)}(\rho)^\top\dd\nu(\rho),
    \label{appH:eq:feature-Gram}
\end{equation}
and update
\begin{equation*}
    \Xi^{(t+1)}
    =\frac{(S^{(t+1)}+\varepsilon I_m)^{1/2}}
      {\tr[(S^{(t+1)}+\varepsilon I_m)^{1/2}]}.
\end{equation*}
This is block-coordinate ascent for the smoothed jointly concave
feature-metric representation.  Damping and backtracking on the smoothed
objective make the iteration monotone.

The decoupled step is evaluated by the generalized-eigenvalue spectral
algorithm used in Section~\ref{sec:est2}.  Define the empirical
uncentered second moments
\[
    \hat\Sigma_p=\hCp+\hmup\hmup^\top,
    \qquad
    \hat\Sigma_q=\hCq+\hmuq\hmuq^\top,
\]
and the affine augmentations
\begin{equation*}
    \overline\Sigma_p^{(t)}=
    \begin{pmatrix}
      \hat\Sigma_p+R^{(t)}&\hmup\\
      \hmup^\top&1
    \end{pmatrix},
    \qquad
    \overline\Sigma_q^{(t)}=
    \begin{pmatrix}
      \hat\Sigma_q+R^{(t)}&\hmuq\\
      \hmuq^\top&1
    \end{pmatrix},
    \qquad
    \overline\delta=
    \begin{pmatrix}\hDelta\\0\end{pmatrix}.
\end{equation*}
A Schur-complement calculation gives
\begin{equation*}
    [\rho\overline\Sigma_p^{(t)}+(1-\rho)\overline\Sigma_q^{(t)}]^{-1}
      \overline\delta
    =\begin{pmatrix}
       \beta^{(t+1)}(\rho)\\
       -\hmu(\rho)^\top\beta^{(t+1)}(\rho)
     \end{pmatrix}.
\end{equation*}
Consequently,
\begin{equation*}
    \hat{\mathcal V}^{\rm feat}(R^{(t)})
    =\frac12\int_0^1
      \overline\delta^\top
      [\rho\overline\Sigma_p^{(t)}+(1-\rho)\overline\Sigma_q^{(t)}]^{-1}
      \overline\delta\dd\nu(\rho).
\end{equation*}
A single generalized eigendecomposition of
\((\overline\Sigma_p^{(t)},\overline\Sigma_q^{(t)})\) evaluates the continuum
response and its moment derivatives through the same divided-difference formulas.  Thus one outer feature-space update
requires one generalized eigendecomposition rather than a coupled system
on feature-task coordinates.

\subsubsection{Task-space reweighted least squares and quadrature}
\label{appH:sec:numerics}

A deterministic quadrature rule (such as Gauss--Legendre, as used in experiments)
\begin{equation*}
    \nu_k=\sum_{\ell=1}^k w_\ell\delta_{\rho_\ell},
    \qquad
    w_\ell>0,
    \qquad
    0<\rho_\ell<1,
\end{equation*}
turns the task operators into \(k\times k\) matrices.

In weighted coordinates, set
\[
    D_p=\diag(\rho_1,\ldots,\rho_k),
    \qquad D_q=I_k-D_p,
    \qquad E=\diag(\rho_\ell(1-\rho_\ell))_{\ell=1}^k,
\]
\[
    a=(\sqrt{w_1},\ldots,\sqrt{w_k})^\top,
    \qquad
    \Lambda_\Omega=\zeta I_k+\lambda\Omega^{-1},
    \qquad
    \Omega\in\mathbb S_{++}^k,
    \quad \tr\Omega=1.
\]
Define the weighted coefficient matrix
\[
     B=(\sqrt{w_1} B(\rho_1)\ \cdots\
      \sqrt{w_k} B(\rho_k))\in\R^{m\times k}.
\]
Thus the \(\ell\)-th column of \( B\) is
\(\sqrt{w_\ell} B(\rho_\ell)\).  The quadrature score is
\begin{align*}
   \!\!\! \hat\Phi_{k,\lambda,\zeta}( B)
    ={}&\hDelta^\top B a
      -\frac12\tr( B^\top\hCp B D_p)
      -\frac12\tr( B^\top\hCq B D_q)-\frac12\tr( B^\top\hDelta\hDelta^\top B E)
      -\frac\zeta2\tr( B^\top B)
      -\frac\lambda2\norm{ B}_*^2.
\end{align*}
The same task metric is then used in two distinct numerical procedures.  The first is an empirical reweighted
least-squares algorithm.  The second replaces the empirical coefficient
Gram by the Gram obtained from the Matrix-Dyson stability equation.

\paragraph{Deterministic task-space reweighting}

The deterministic counterpart uses the same metric update but replaces the
empirical Gram \(( B^{(t)})^\top B^{(t)}\) by the Matrix-Dyson prediction
\(T_{\Omega^{(t)}}\), obtained from the quadrature stability equation in 
Eq.~\eqref{appH:eq:matrix-Gram}.  More precisely, define
\begin{equation*}
    \Phi_{\varepsilon}^{(k)}(\Omega)
    =\mathcal V_k(\Omega)
      -\frac{\lambda\varepsilon}{2}\tr(\Omega^{-1}).
\end{equation*}
Using the operator envelope formula in Eq.~\eqref{appH:eq:operator-gradient} and
\(d\Lambda_\Omega=-\lambda\Omega^{-1}(d\Omega)\Omega^{-1}\) gives
\begin{equation*}
    \nabla_\Omega\Phi_{\varepsilon}^{(k)}(\Omega)
    =\frac\lambda2\Omega^{-1}
      (T_\Omega+\varepsilon I_k)\Omega^{-1}.
\end{equation*}
Hence the trace-one KKT equation is the fixed-point update
\begin{equation}
    \widetilde\Omega^{(t+1)}
    =\frac{(T_{\Omega^{(t)}}+\varepsilon I_k)^{1/2}}
      {\tr[(T_{\Omega^{(t)}}+\varepsilon I_k)^{1/2}]}.
    \label{appH:eq:task-deterministic-update}
\end{equation}
A practical iteration therefore solves the MDE for \(Q_{\Omega^{(t)}}\), solves
the stability equation for \(T_{\Omega^{(t)}}\), and then applies
Eq.~\eqref{appH:eq:task-deterministic-update}.  Damping and backtracking on
\(\Phi_{\varepsilon}^{(k)}\) provide a monotone version of this fixed-point
scheme.  All outer matrices have dimension \(k\), independently of the
feature dimension \(m\).

\paragraph{Matrix-Dyson solve and numerical evaluation}

The quadrature approximation of the operator MDE is
\begin{equation*}
    Q^{-1}=\Lambda_\Omega
      +D_p(I_k+\alpha_pQD_p)^{-1}
      +D_q(I_k+\alpha_qQD_q)^{-1}.
\end{equation*}
Its deterministic response is
\begin{equation*}
    \mathcal V_k(\Omega)
    =\frac12a^\top(I_k+\RD\,QE)^{-1}\RD\,Q a.
\end{equation*}
With
\[
    S_p=D_p(I_k+\alpha_pQD_p)^{-1},
    \qquad
    S_q=D_q(I_k+\alpha_qQD_q)^{-1},
\]
the stability map is
\begin{equation*}
    \mathcal L_Q[X]
    =Q^{-1}XQ^{-1}
      -\alpha_pS_pXS_p
      -\alpha_qS_qXS_q.
\end{equation*}
The task Gram is obtained from
\begin{equation}
    \mathcal L_Q[T]=\RD cc^\top,
    \qquad
    c=(I_k+E\RD\,Q)^{-1}a.
    \label{appH:eq:matrix-Gram}
\end{equation}

For fixed $\Omega$, the monotone Picard iteration
\[
  Q^{(j+1)}=\left[\Lambda_\Omega
      +D_p(I_k+\alpha_pQ^{(j)}D_p)^{-1}
      +D_q(I_k+\alpha_qQ^{(j)}D_q)^{-1}\right]^{-1},
  \qquad Q^{(0)}=0,
\]
converges to the minimal positive fixed point by the standard monotone iteration
argument in ordered cones~\cite{Amann1976}, using the order-reversing property
of matrix inversion in the Loewner order~\cite{Bhatia2007}. Under the
uniqueness condition for the Dyson fixed point, this is the positive solution.

Near the solution, Newton's acceleration solves
\begin{equation*}
    \mathcal L_Q[\Delta Q]
    =Q^{-1}-\Lambda_\Omega
      -D_p(I_k+\alpha_pQD_p)^{-1}
      -D_q(I_k+\alpha_qQD_q)^{-1},
\end{equation*}
with backtracking chosen to keep \(Q+t\Delta Q\succ0\) and decrease the
residual.  After solving Eq.~\eqref{appH:eq:matrix-Gram}, projected gradient ascent
on \(\Omega\succeq\delta I_k\), \(\tr\Omega=1\), uses
\[
    \nabla_\Omega\mathcal V_k
    =\frac\lambda2\Omega^{-1}T\Omega^{-1}.
\]
For moderate \(k\), a dense \(k^2\times k^2\) solve is sufficient for the
Gram equation; a matrix-free Krylov method applies \(\mathcal L_Q\) at
\(O(k^3)\) cost per iteration.

The same quadrature rule evaluates the feature Gram in
Eq.~\eqref{appH:eq:feature-Gram} as
\[
    S^{(t+1)}\simeq\sum_{\ell=1}^k w_\ell
      \beta^{(t+1)}(\rho_\ell)\beta^{(t+1)}(\rho_\ell)^\top.
\]
It also replaces the Fredholm determinant, trace, and inverse in
Section~\ref{appH:sec:scoring} by their matrix counterparts.  Increasing \(k\) until
the response, scores, and outer metric stabilize provides a direct
numerical refinement check.  Warm starts are effective across nearby
parameters and successive quadrature rules.

\subsection{Nuclear penalty experiments}
We provide in \myfig{nuclearasympt} and \myfig{nuclearsweep} simple experiments highlighting the results in this appendix. \myfig{nuclearasympt} performs a sweep in $\alpha_q$ and compares empirical estimates with their deterministic asymptotic limit (for $m=100$), while \myfig{nuclearsweep} compares deterministic equivalents for ridge and nuclear penalties, showing the benefits of the nuclear penalty for a wide range of values of $\alpha_p$ and $\alpha_q$.

\begin{figure}
    \centering
    \includegraphics[width=0.4\linewidth]{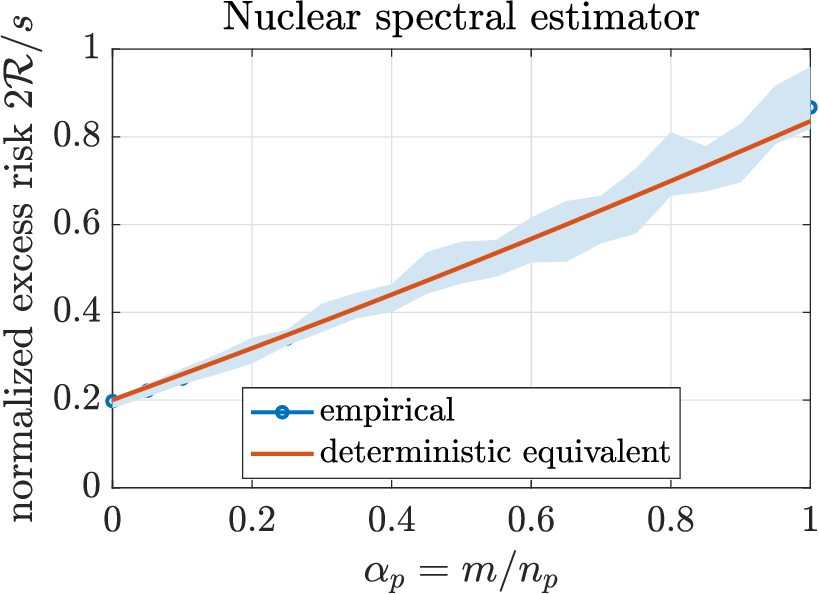}

\caption{Performance of the estimator with nuclear penalty, when varying $\alpha_p$ ($\alpha_q=0.1$ fixed, with $\alpha_p$ swept over $[0,1]$). \label{fig:nuclearasympt}}
\end{figure}

\begin{figure}
   \centering
    \includegraphics[width=0.4\linewidth]{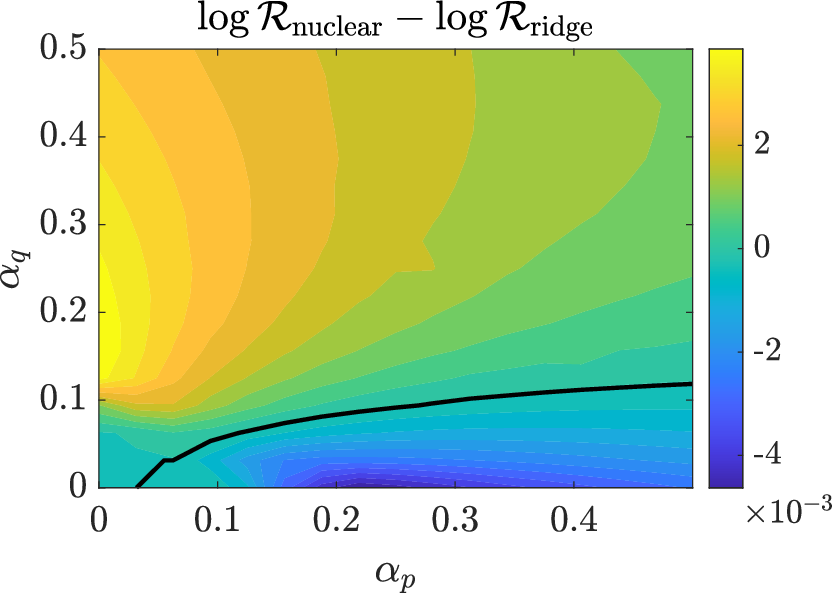}

\caption{Comparison of two spectral estimators for a given $ s= 1$ when regularization
parameters are optimized. The black line is the zero-level line.\label{fig:nuclearsweep}}
\end{figure}

\bibliographystyle{unsrt}

\bibliography{references}

\begin{thebibliography}{10}

\bibitem{Bach2026}
Francis Bach.
\newblock A spectral framework for closed-form relative density estimation,
  2026.
\newblock arXiv:2605.10668.

\bibitem{SugiyamaNakanoKawanabe2008}
Masashi Sugiyama, Taiji Suzuki, Shinichi Nakajima, Hisashi Kashima, Paul von
  B{\"u}nau, and Motoaki Kawanabe.
\newblock Direct importance estimation for covariate shift adaptation.
\newblock {\em Annals of the Institute of Statistical Mathematics},
  60(4):699--746, 2008.

\bibitem{KanamoriHidoSugiyama2009}
Takafumi Kanamori, Shohei Hido, and Masashi Sugiyama.
\newblock A least-squares approach to direct importance estimation.
\newblock {\em Journal of Machine Learning Research}, 10:1391--1445, 2009.

\bibitem{SugiyamaSuzukiKanamori2012}
Masashi Sugiyama, Taiji Suzuki, and Takafumi Kanamori.
\newblock {\em Density Ratio Estimation in Machine Learning}.
\newblock Cambridge University Press, 2012.

\bibitem{NguyenWainwrightJordan2010}
XuanLong Nguyen, Martin~J. Wainwright, and Michael~I. Jordan.
\newblock Estimating divergence functionals and the likelihood ratio by convex
  risk minimization.
\newblock {\em IEEE Transactions on Information Theory}, 56(11):5847--5861,
  2010.

\bibitem{DonskerVaradhan1975}
Monroe~D. Donsker and S.~R.~Srinivasa Varadhan.
\newblock Asymptotic evaluation of certain {M}arkov process expectations for
  large time, {I}.
\newblock {\em Communications on Pure and Applied Mathematics}, 28(1):1--47,
  1975.

\bibitem{Qin1998}
Jing Qin.
\newblock Inferences for case-control and semiparametric two-sample density
  ratio models.
\newblock {\em Biometrika}, 85(3):619--630, 1998.

\bibitem{CandesSur2020}
Emmanuel~J. Cand{\`e}s and Pragya Sur.
\newblock The phase transition for the existence of the maximum likelihood
  estimate in high-dimensional logistic regression.
\newblock {\em The Annals of Statistics}, 48(1):27--42, 2020.

\bibitem{ChardonLerasleMourtada2026}
Hugo Chardon, Matthieu Lerasle, and Jaouad Mourtada.
\newblock Finite-sample performance of the maximum likelihood estimator in
  logistic regression, 2024.
\newblock arXiv:2411.02137.

\bibitem{HoerlKennard1970}
Arthur~E. Hoerl and Robert~W. Kennard.
\newblock Ridge regression: Biased estimation for nonorthogonal problems.
\newblock {\em Technometrics}, 12(1):55--67, 1970.

\bibitem{HuangSmolaGrettonBorgwardtScholkopf2007}
Jiayuan Huang, Alexander~J. Smola, Arthur Gretton, Karsten~M. Borgwardt, and
  Bernhard Sch{\"o}lkopf.
\newblock Correcting sample selection bias by unlabeled data.
\newblock In {\em Advances in Neural Information Processing Systems}, 2007.

\bibitem{MenonOng2016}
Aditya~Krishna Menon and Cheng~Soon Ong.
\newblock Linking losses for density ratio and class-probability estimation.
\newblock In {\em International Conference on Machine Learning}, 2016.

\bibitem{SurCandes2019}
Pragya Sur and Emmanuel~J. Cand{\`e}s.
\newblock A modern maximum-likelihood theory for high-dimensional logistic
  regression.
\newblock {\em Proceedings of the National Academy of Sciences},
  116(29):14516--14525, 2019.

\bibitem{salehi2019impact}
Fariborz Salehi, Ehsan Abbasi, and Babak Hassibi.
\newblock The impact of regularization on high-dimensional logistic regression.
\newblock In {\em Advances in Neural Information Processing Systems}, 2019.

\bibitem{BaiSilverstein2010}
Zhidong Bai and Jack~W. Silverstein.
\newblock {\em Spectral Analysis of Large Dimensional Random Matrices}.
\newblock Springer, 2nd edition, 2010.

\bibitem{HachemLoubatonNajim2007}
Walid Hachem, Philippe Loubaton, and Jamal Najim.
\newblock Deterministic equivalents for certain functionals of large random
  matrices.
\newblock {\em The Annals of Applied Probability}, 17(3):875--930, 2007.

\bibitem{dobriban2018prediction}
Edgar Dobriban and Stefan Wager.
\newblock High-dimensional asymptotics of prediction: Ridge regression and
  classification.
\newblock {\em The Annals of Statistics}, 46(1):247--279, 2018.

\bibitem{montanari2022interpolation}
Andrea Montanari and Yiqiao Zhong.
\newblock The interpolation phase transition in neural networks: Memorization
  and generalization under lazy training.
\newblock {\em The Annals of Statistics}, 50(5):2816--2847, 2022.

\bibitem{Bach2024RandomProjections}
Francis Bach.
\newblock High-dimensional analysis of double descent for linear regression
  with random projections.
\newblock {\em SIAM Journal on Mathematics of Data Science}, 6(1):26--50, 2024.

\bibitem{mccullagh2019generalized}
Peter McCullagh and John~A. Nelder.
\newblock {\em Generalized Linear Models}.
\newblock Chapman \& Hall, 1989.

\bibitem{BoydVandenberghe2004}
Stephen Boyd and Lieven Vandenberghe.
\newblock {\em Convex Optimization}.
\newblock Cambridge University Press, 2004.

\bibitem{NocedalWright2006}
Jorge Nocedal and Stephen~J. Wright.
\newblock {\em Numerical Optimization}.
\newblock Springer, 2nd edition, 2006.

\bibitem{broniatowski2006minimization}
Michel Broniatowski and Amor Keziou.
\newblock Minimization of $\varphi$-divergences on sets of signed measures.
\newblock {\em Studia Scientiarum Mathematicarum Hungarica}, 43(4):403--442,
  2006.

\bibitem{Gautschi2004}
Walter Gautschi.
\newblock {\em Orthogonal Polynomials: Computation and Approximation}.
\newblock Oxford University Press, 2004.

\bibitem{banzato2025existence}
Erika Banzato, Mathias Drton, Kian Saraf-Poor, and Hongjian Shi.
\newblock Existence of direct density ratio estimators, 2025.
\newblock arXiv:2502.12738.

\bibitem{SoudryHofferNacsonGunasekarSrebro2018}
Daniel Soudry, Elad Hoffer, Mor~Shpigel Nacson, Suriya Gunasekar, and Nathan
  Srebro.
\newblock The implicit bias of gradient descent on separable data.
\newblock {\em Journal of Machine Learning Research}, 19(70):1--57, 2018.

\bibitem{GunasekarLeeSoudrySrebro2018}
Suriya Gunasekar, Jason~D. Lee, Daniel Soudry, and Nathan Srebro.
\newblock Implicit bias of gradient descent on linear convolutional networks.
\newblock In {\em Advances in Neural Information Processing Systems}, 2018.

\bibitem{van2000asymptotic}
Aad~W. van~der Vaart.
\newblock {\em Asymptotic Statistics}.
\newblock Cambridge University Press, 1998.

\bibitem{Vershynin2018}
Roman Vershynin.
\newblock {\em High-Dimensional Probability: An Introduction with Applications
  in Data Science}.
\newblock Cambridge University Press, 2018.

\bibitem{thrampoulidis2014gaussian}
Christos Thrampoulidis, Samet Oymak, and Babak Hassibi.
\newblock The {Gaussian} min-max theorem in the presence of convexity, 2014.
\newblock arXiv:1408.4837.

\bibitem{TOH2015}
Christos Thrampoulidis, Samet Oymak, and Babak Hassibi.
\newblock Regularized linear regression: A precise analysis of the estimation
  error.
\newblock In {\em Conference on Learning Theory}, 2015.

\bibitem{CorlessGonnetHareJeffreyKnuth1996}
Robert~M. Corless, Gaston~H. Gonnet, David E.~G. Hare, David~J. Jeffrey, and
  Donald~E. Knuth.
\newblock On the {L}ambert {W} function.
\newblock {\em Advances in Computational Mathematics}, 5:329--359, 1996.

\bibitem{SimonTraceIdeals2005}
Barry Simon.
\newblock {\em Trace Ideals and Their Applications}.
\newblock American Mathematical Society, 2005.

\bibitem{brent2013algorithms}
Richard~P. Brent.
\newblock {\em Algorithms for Minimization Without Derivatives}.
\newblock Prentice Hall, 1973.

\bibitem{ArgyriouEvgeniouPontil2008}
Andreas Argyriou, Theodoros Evgeniou, and Massimiliano Pontil.
\newblock Convex multi-task feature learning.
\newblock {\em Machine Learning}, 73(3):243--272, 2008.

\bibitem{daubechies2010iteratively}
Ingrid Daubechies, Ronald DeVore, Massimo Fornasier, and C.~Sinan
  G{\"u}nt{\"u}rk.
\newblock Iteratively reweighted least squares minimization for sparse
  recovery.
\newblock {\em Communications on Pure and Applied Mathematics}, 63(1):1--38,
  2010.

\bibitem{BachJenattonMairalObozinski2012}
Francis Bach, Rodolphe Jenatton, Julien Mairal, and Guillaume Obozinski.
\newblock Optimization with sparsity-inducing penalties.
\newblock {\em Foundations and Trends in Machine Learning}, 4(1):1--106, 2012.

\bibitem{AjankiErdosKruger2019MDE}
Oskari~H. Ajanki, L{\'a}szl{\'o} Erd{\H{o}}s, and Torben~H. Kr{\"u}ger.
\newblock Stability of the matrix {Dyson} equation and random matrices with
  correlations.
\newblock {\em Probability Theory and Related Fields}, 173(1--2):293--373,
  2019.

\bibitem{BelghaziEtAl2018}
Mohamed~Ishmael Belghazi, Aristide Baratin, Sai Rajeshwar, Sherjil Ozair,
  Yoshua Bengio, Aaron Courville, and Devon Hjelm.
\newblock Mutual information neural estimation.
\newblock In {\em International Conference on Machine Learning}, 2018.

\bibitem{PooleEtAl2019}
Ben Poole, Sherjil Ozair, Aaron van~den Oord, Alexander~A. Alemi, and George
  Tucker.
\newblock On variational bounds of mutual information.
\newblock In {\em International Conference on Machine Learning}, 2019.

\bibitem{RahimiRecht2007}
Ali Rahimi and Benjamin Recht.
\newblock Random features for large-scale kernel machines.
\newblock In {\em Advances in Neural Information Processing Systems}, 2007.

\bibitem{WilliamsSeeger2001}
Christopher K.~I. Williams and Matthias Seeger.
\newblock Using the {N}ystr{\"o}m method to speed up kernel machines.
\newblock In {\em Advances in Neural Information Processing Systems}, 2001.

\bibitem{LouartLiaoCouillet2018}
Cosme Louart, Zhenyu Liao, and Romain Couillet.
\newblock A random matrix approach to neural networks.
\newblock {\em The Annals of Applied Probability}, 28(2):1190--1248, 2018.

\bibitem{MeiMontanari2022}
Song Mei and Andrea Montanari.
\newblock The generalization error of random features regression: Precise
  asymptotics and the double descent curve.
\newblock {\em Communications on Pure and Applied Mathematics}, 75(4):667--766,
  2022.

\bibitem{HuLu2023}
Hong Hu and Yue~M. Lu.
\newblock Universality laws for high-dimensional learning with random features.
\newblock {\em IEEE Transactions on Information Theory}, 69(3):1932--1964,
  2023.

\bibitem{GeraceEtAl2020}
Federica Gerace, Bruno Loureiro, Florent Krzakala, Marc M{\'e}zard, and Lenka
  Zdeborov{\'a}.
\newblock Generalisation error in learning with random features and the hidden
  manifold model.
\newblock In {\em International Conference on Machine Learning}, 2020.

\bibitem{LoureiroEtAl2021GenericFeatures}
Bruno Loureiro, C{\'e}dric Gerbelot, Hugo Cui, Sebastian Goldt, Florent
  Krzakala, Marc M{\'e}zard, and Lenka Zdeborov{\'a}.
\newblock Learning curves of generic features maps for realistic datasets with
  a teacher-student model.
\newblock In {\em Advances in Neural Information Processing Systems}, 2021.

\bibitem{LoureiroEtAl2022ConvexLosses}
Bruno Loureiro, Cedric Gerbelot, Maria Refinetti, Gabriele Sicuro, and Florent
  Krzakala.
\newblock Fluctuations, bias, variance \& ensemble of learners: Exact
  asymptotics for convex losses in high-dimension.
\newblock In {\em International Conference on Machine Learning}, 2022.

\bibitem{Johnstone2001}
Iain~M. Johnstone.
\newblock On the distribution of the largest eigenvalue in principal components
  analysis.
\newblock {\em The Annals of Statistics}, 29(2):295--327, 2001.

\bibitem{BaikBenArousPeche2005}
Jinho Baik, G{\'e}rard Ben~Arous, and Sandrine P{\'e}ch{\'e}.
\newblock Phase transition of the largest eigenvalue for nonnull complex sample
  covariance matrices.
\newblock {\em The Annals of Probability}, 33(5):1643--1697, 2005.

\bibitem{ThrampoulidisAbbasiHassibi2018}
Christos Thrampoulidis, Ehsan Abbasi, and Babak Hassibi.
\newblock Precise error analysis of regularized {M}-estimators in high
  dimensions.
\newblock {\em IEEE Transactions on Information Theory}, 64(8):5592--5628,
  2018.

\bibitem{Bornemann2010Fredholm}
Folkmar Bornemann.
\newblock On the numerical evaluation of {F}redholm determinants.
\newblock {\em Mathematics of Computation}, 79(270):871--915, 2010.

\bibitem{Kato1995}
Tosio Kato.
\newblock {\em Perturbation Theory for Linear Operators}.
\newblock Springer, 1995.

\bibitem{BloemendalErdosKnowlesYauYin2014}
Alex Bloemendal, L{\'a}szl{\'o} Erd{\H{o}}s, Antti Knowles, Horng-Tzer Yau, and
  Jun Yin.
\newblock Isotropic local laws for sample covariance and generalized {Wigner}
  matrices.
\newblock {\em Electronic Journal of Probability}, 19(33):1--53, 2014.

\bibitem{Stein1981}
Charles~M. Stein.
\newblock Estimation of the mean of a multivariate normal distribution.
\newblock {\em The Annals of Statistics}, 9(6):1135--1151, 1981.

\bibitem{KhorunzhyKhoruzhenkoPastur1996}
Alexei~M. Khorunzhy, Boris~A. Khoruzhenko, and Leonid~A. Pastur.
\newblock Asymptotic properties of large random matrices with independent
  entries.
\newblock {\em Journal of Mathematical Physics}, 37(10):5033--5060, 1996.

\bibitem{AjankiErdosKruger2019QVE}
Oskari~H. Ajanki, L{\'a}szl{\'o} Erd{\H{o}}s, and Torben~H. Kr{\"u}ger.
\newblock {\em Quadratic Vector Equations on Complex Upper Half-Plane}.
\newblock Memoirs of the American Mathematical Society. American Mathematical
  Society, 2019.

\bibitem{AltErdosKruger2020}
Johannes Alt, L{\'a}szl{\'o} Erd{\H{o}}s, and Torben~H. Kr{\"u}ger.
\newblock The {Dyson} equation with linear self-energy: Spectral bands, edges
  and cusps.
\newblock {\em Documenta Mathematica}, 25:1421--1539, 2020.

\bibitem{Amann1976}
Herbert Amann.
\newblock Fixed point equations and nonlinear eigenvalue problems in ordered
  {Banach} spaces.
\newblock {\em SIAM Review}, 18(4):620--709, 1976.

\bibitem{Bhatia2007}
Rajendra Bhatia.
\newblock {\em Positive Definite Matrices}.
\newblock Princeton University Press, 2007.

\end{thebibliography}

\end{document}